\documentclass{article}

\usepackage{microtype}
\usepackage{graphicx}
\usepackage{subfigure}
\usepackage{booktabs} 
\usepackage{amsfonts}
\usepackage{mathrsfs}
\usepackage{multirow}
\usepackage[normalem]{ulem}
\usepackage{pgfplots} 
\useunder{\uline}{\ul}{}
\usepackage{wrapfig}
\usepackage{color}
\usepackage{hyperref}

\pgfplotsset{compat=1.17}

\usepackage[accepted]{icml2023}

\usepackage{amsmath}
\usepackage{amssymb}
\usepackage{mathtools}
\usepackage{amsthm}

\usepackage[capitalize,noabbrev]{cleveref}

\theoremstyle{plain}

\theoremstyle{definition}

\theoremstyle{remark}

\usepackage[textsize=tiny]{todonotes}

\icmltitlerunning{Moderately Distributional Exploration for Domain Generalization}
\begin{document}

\twocolumn[
\icmltitle{Moderately Distributional Exploration for Domain Generalization}



\icmlsetsymbol{equal}{*}
\icmlsetsymbol{dagger}{$\dagger$}

\begin{icmlauthorlist}
\icmlauthor{Rui Dai}{yyy}
\icmlauthor{Yonggang Zhang}{bbb,dagger}
\icmlauthor{Zhen Fang}{aaa}
\icmlauthor{Bo Han}{bbb}
\icmlauthor{Xinmei Tian}{yyy,ccc,dagger}
\end{icmlauthorlist}

\icmlaffiliation{yyy}{University of Science and Technology of China, Hefei, China}
\icmlaffiliation{aaa}{Australian Artificial Intelligence Institute, University of Technology Sydney, Sydney, Australia}
\icmlaffiliation{ccc}{Institute of Artificial Intelligence, Hefei Comprehensive National Science Center, Hefei, China}
\icmlaffiliation{bbb}{Department of Computer Science, Hong Kong Baptist University, HongKong, China}

\icmlcorrespondingauthor{Xinmei Tian}{xinmei@ustc.edu.cn}
\icmlcorrespondingauthor{Yonggang Zhang}{csygzhang@comp.hkbu.edu.hk}

\icmlkeywords{Domain Generalization, Machine Learning, ICML}

\vskip 0.3in
]



\printAffiliationsAndNotice{}  

\begin{abstract}
Domain generalization (DG) aims to tackle the distribution shift between training domains and unknown target domains. 
Generating new domains is one of the most effective approaches, yet its performance gain depends on the distribution discrepancy between the generated and target domains.
Distributionally robust optimization is promising to tackle distribution discrepancy by exploring domains in an uncertainty set. 
However, the uncertainty set may be overwhelmingly large, leading to low-confidence prediction in DG.
It is because a large uncertainty set could introduce domains containing semantically different factors from training domains.
To address this issue, we propose to perform a \emph{mo}derately \emph{d}istributional \emph{e}xploration (MODE) for domain generalization. Specifically, MODE performs distribution exploration in an uncertainty \textit{subset} that shares the same semantic factors with the training domains. 
We show that MODE can endow models with provable generalization performance on unknown target domains.
The experimental results show that MODE achieves competitive performance compared to state-of-the-art baselines.
\end{abstract}

\section{Introduction}
\label{Introduction}

Deep neural networks (DNNs) have achieved exciting performance on various tasks. The successes of DNNs heavily depend on an underlying assumption that the training domains and target domain share the same distribution. 
However, this assumption may not hold in some practical scenarios, which leads to the failure of DNNs. To release this assumption, researchers have studied a more practical learning setting called \textit{Domain Generalization} (DG) \cite{muandet2013domain,ye2021towards,shen2021towards}. The goal of DG is to train models using training domains such that these models can generalize well in the unknown target domain which shares the same semantics with the training domains.

To generalize well on the unknown target domains, previous works introduce a domain generation strategy, enhancing the performance of DNNs by generating new domains \cite{zhou2020learning,zhou2020deep,wang2021learning,xu2021fourier}. The underlying intuition of this approach is that learning with many generated domains could make DNNs robust against domain shifts. However, it remains challenging how to construct new domains to achieve a provable generalization performance on target domains. Namely, it is challenging to guarantee a mitigated distribution discrepancy between the generated domains and target domains. Accordingly, the generated domains may fail to promote generalizability or even cause performance degradation of DNNs. The reason lies in the fact that target domains are unknown in the training process, leading to an uncontrollable distribution discrepancy between the generated and the target domains.

Distributionally Robust Optimization (DRO) is a possible strategy to tackle the distribution discrepancy between training and target domains \cite{csiszar1967information,namkoong2016stochastic,staib2019distributionally}. The intuition of DRO is to extend one distribution to a distribution space, i.e., uncertainty set, and uses the worst-case distribution in the uncertainty set for model training \cite{sinha2017certifying, michel2021modeling, mehra2022certifying}. By ensuring uniformly well performance inside the uncertainty set around the training domains, DRO can enlarge the influence of the training domains and thus shrink the distribution discrepancy between training and test domains. Unfortunately, directly employing DRO to DG has shown limited performance improvement in practice \cite{shen2021towards}. The failure of DRO may be related to the overwhelmingly large property of the employed uncertainty set. Such a large uncertainty set may introduce some unrelated domains containing semantics inconsistently with training domains. Consequently, models trained over the uncertainty set make decisions with fairly low confidence, known as the low confidence issue \cite{hu2018does, frogner2019incorporating,shen2021towards}.

\tikzstyle{arrow1} = [thick,->,>=stealth]
\tikzstyle{arrow2} = [<->,>=stealth]
\begin{figure}[htbp]\label{fig:causal}\label{fig:g}
\centering
\vspace{5mm}
\begin{tikzpicture}[scale=1, line width=0.2pt]
\draw (0, 0) node(s) [circle, draw]  {{\footnotesize\,$S$\,}};
\draw (-2, 0) node(c)[circle, draw]  {{\footnotesize\,$C$\,}};
\draw (-3, -1.2) node(y) [circle, draw, fill=black!25]  {{\footnotesize\,$Y$\,}};
\draw (-1, -1.2) node(x)[circle, draw, fill=black!25]  {{\footnotesize\,$X$\,}};
\draw [arrow1] (s) -- (x); 
\draw [arrow1] (c) -- (x);
\draw [arrow1] (c) -- (y);
\draw (1.5, 0) node(s2) [circle, draw]  {{\footnotesize\,$S$\,}};
\draw (1.5, -1.2) node(c2)[circle, draw]  {{\footnotesize\,$C$\,}};
\draw (2.75, -0.6) node(G)[rectangle, draw]  {{\footnotesize\,$\mathbf{G}$\,}};
\draw (4, -0.6) node(x2)[circle, draw, fill=black!25]  {{\footnotesize\,$X$\,}};
\draw [arrow1] (s2) -- (G); 
\draw [arrow1] (c2) -- (G);
\draw [arrow1] (G) -- (x2);
\end{tikzpicture}

\caption{The left shows the causal relationship of data $X$, label $Y$, semantic factor $C$ and non-semantic factor $S$. The right shows that data $X$ are generated by an causal mechanism $\mathbf{G}$ with two causes: semantic factor $C$ and non-semantic factor $S$.}
\vspace{-0.3cm}
\end{figure}
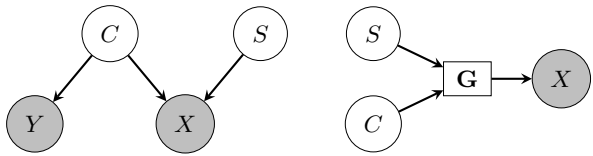
To fully unleash the potential of DRO in DG, we propose to perform distribution exploration in an uncertainty \emph{subset}, which shares the same semantic factors with the training domains, avoiding the exploration of semantically unrelated domains. The insight lies in that merely exploring the semantically related subset could shrink the space of the uncertainty set, mitigating the low confidence issue above.

Specifically, following prior works \cite{suter2019robustly,zhang2020causal,mitrovic2020representation,mahajan2021domain,zhang2021causaladv,veitch2021counterfactual,lv2022causality}, we assume that observed data $X$ are generated by an causal mechanism $\mathbf{G}$ with two causes: semantic factor $C$ and non-semantic factor $S$, i.e., $X=\mathbf{G}(S,C)$, where the label $Y$ is the effect of the semantic factor $C$. Built upon this assumption, we can perform DRO on the subset of non-semantic factor $S$, rather than on the original uncertainty set containing both semantic and non-semantic factors. Motivated by this insight, we propose a novel approach \textit{moderately distributional exploration} (MODE) for domain generalization.

To support our approach, we develop a theoretical framework that provides the generalization estimation of our learning principle and gives the risk estimation for the unknown target domain. Empirically, we conduct extensive experiments to verify the effectiveness of our approach. The experimental results show that {MODE} achieves competitive performance compared with the state-of-the-art baselines.

\section{Related Work}
\label{Related work}
\subsection{Domain Generalization}
{Domain generalization} aims to learn more generalized knowledge from existing multiple source domains and finally test on the unknown target domain. Over the years, great efforts have been made in many directions, such as Invariant Representation \cite{chuang2020estimating,nguyen2021domain,xiao2021bit,shi2022gradient}, Causal \cite{mahajan2021domain,mouli2021asymmetry, lv2022causality}, and Optimization \cite{krueger2021out,zhang2021can,lei2021near,gulrajani2021in}. 
To generalize well on the unknown target domains, previous works introduce a domain generation strategy, enhancing the performance of DNNs by generating new domains.
\citet{shankar2018generalizing} perturbs the input samples along the direction of the most significant domain change while maintaining semantics.
\citet{zhou2020deep} trains a domain transformation model to transform images to unseen domains by fooling a domain classifier.
\citet{somavarapu2020frustratingly,borlino2021rethinking} simply use a style transfer like AdaIN \cite{huang2017arbitrary} to argument data in style aspects to optimize the model. 
\citet{zhou2020learning} train a data generator to generate new domains using optimal transport to measure the distribution divergence.
\citet{zhou2021domain,zhou2021mixstyle} achieves style augmentation in the feature level by mixing the CNN feature map's mean and std between instances of different domains.
\citet{li2022uncertainty} focuses on addressing the uncertain nature of domain shifts by modeling feature statistics as uncertain distributions, which is also achieved through the use of AdaIN, where non-semantic factors are replaced with randomly chosen values from the modeled distributions.
\citet{tang2021crossnorm} address the problem of domain shift by developing two simple and efficient normalization methods that can reduce the non-semantic domain shift between different distributions, while \citet{zhang2022towards} jointly learns semantic and variation encoders to disentangle the semantic and non-semantic factors.
Our approach explores the non-semantic factor to create augmented samples, which to some extent, is similar to approaches of data augmentation.

\subsection{Distributionally Robust Optimization}
Distributionally robust optimization is a promising approach to tackle distribution discrepancy by exploring unknown domains in a fixed uncertainty set \cite{sagawa2019distributionally}. DRO has developed plenty of approaches with different methods to measuring distribution discrepancy, such as Wasserstein distance \cite{sinha2017certifying, mehra2022certifying}, $f$-divergence \cite{csiszar1967information,ben2013robust,namkoong2016stochastic,michel2021modeling} and maximum mean discrepancy (MMD) \cite{staib2019distributionally}.
Unfortunately, employing DRO to DG has shown limited performance improvement in practice \cite{shen2021towards}.
\cite{hu2018does,frogner2019incorporating,shen2021towards} have pointed out that in order to capture the unknown target domain, the uncertainty set is often overwhelmingly large, leading the learned model to make decisions with fairly low confidence in DRO.
\citet{liu2021stable,liu2022distributionally} focuses on the low confidence problem, and use a Wasserstein distance is employed to determine the uncertainty set. 
\citet{liu2022distributionally2} uses data geometry to construct more reasonable and effective uncertainty sets, while \citet{qiao2023topology} constructs the uncertainty using the data topology.
Our approach MODE tackles the low confidence problem by performing distribution exploration in a specific uncertainty subset (non-semantic factor) and uses Wasserstein distance \cite{sinha2017certifying, mehra2022certifying} to measure the distribution discrepancy in DG.

\section{Learning Setups}\label{ls}

Let $\mathcal{X}$ denote the feature space and  $\mathcal{Y}=\{1, \ldots, \mathrm{K}_{\mathcal{Y}}\}$  denote the label space. We consider the training domains $D_{X_{l} Y_{l}}$ ($l=1,...,N$), $N$ joint distributions defined over $\mathcal{X} \times \mathcal{Y} $, where $ X_{l}$  and  $Y_{l}$  are random variables whose outputs are from  $\mathcal{X} $ and  $\mathcal{Y}$, respectively. We also have a target domain $D_{X_{\mathrm{t}} Y_{\mathrm{t}}}$, a joint distribution defined over $\mathcal{X}\times \mathcal{Y}$, shares the same semantics with training domains $D_{X_{l} Y_{l}}$.

In this paper, we focus on domain generalization. The formal definition of domain generalization is given as follows.
\theoremstyle{definition}
\newtheorem{problem}{Problem}
\begin{problem}
(Domain Generalization). Let  $D_{X_{l} Y_{l}}$ ($l=1,...,N$) and $D_{X_{\mathrm{t}} Y_{\mathrm{t}}}$ be training domains and unseen target domain, respectively. Given sets of samples called the training data: for any $l=1,...,N$,
\begin{align*}
{TR}_{l}=\left\{\left(\mathbf{x}_{l}^{1}, y_{l}^{1}\right), \ldots,\left(\mathbf{x}_{l}^{n_{l}}, y_{l}^{n_{l}}\right)\right\} \sim D_{X_{l} Y_{l}}^{n_l},~ \text {i.i.d.},
\end{align*}
the aim of domain generalization is to train a classifier  $f$  by using the training data ${TR}_{l},l=1,...,N$ such that, for any test data  $\mathbf{x} \sim D_{X_{\mathrm{t}}}$, $f$ can classify  $\mathbf{x}$  into the correct class.
\end{problem}

\textbf{Causal Assumption.} Following prior works \cite{mitrovic2020representation,zhang2020causal,suter2019robustly,mahajan2021domain,zhang2021causaladv,lv2022causality,nguyen2022front,chen2022learning}, we \textit{assume} that the feature random variables are generated by the following causal mechanism. Let $\mathcal{S}$ and $\mathcal{C}$ be the non-semantic factor space and semantic factor space, respectively. There exists an causal mechanism $\mathbf{G}: \mathcal{S} \times \mathcal{C} \rightarrow \mathcal{X}$ such that,
\begin{equation}\label{Eq::CauMec}
X_{t} = \mathbf{G}(S_{t}, C)~ \text{and} ~X_{l} = \mathbf{G}(S_{l}, C),~ \forall{ \ l=1,...,N}
\end{equation}
where $S_{l}$ and $S_{t}$ are random variables defined over the non-semantic factor space $\mathcal{S}$, and $C$ is the random variable defined over the semantic factor space $\mathcal{C}$. In summary, Eq. \eqref{Eq::CauMec} means that the feature randoms $X$ share the same semantic $C$, but don't share the non-semantic factors $S$.

Generally, one hopes that the non-semantic random variable cannot affect the label random variable $Y$, which can be determined only by the semantic $C$. Therefore, following \citet{mitrovic2020representation,zhang2020causal,suter2019robustly,mahajan2021domain,lv2022causality,nguyen2022front}, we further assume that for any $l=1,...,N$,
\begin{equation}\label{Ass2}
\begin{split}
Y \leftarrow C~\text{and}~Y_{l} = Y_{t} = Y
.
\end{split}
\end{equation} 

\textbf{Uncertainty Set and Non-semantic Space.} 
To enhance the diversity of training domains and preserve the semantics among domains, the uncertainty set where we perform DRO is defined in the following:
\begin{equation} \label{omega}
{\Omega} = \{ S_{\boldsymbol{\alpha}} = \Psi(\boldsymbol{\alpha}, S_{1},\ldots,S_{N}) : \boldsymbol{\alpha}\in \mathcal{A}\},
\end{equation}
where $\mathcal{A}$ is a parametric space, $\Psi$ is a function that could generate random variables, $\Omega$ is a set of random variables defined over $\mathcal{S}$.

The non-semantic space w.r.t. $\Omega$ is defined in following:
\begin{equation}\label{non-semanticspace}
\mathcal{S}_{\Omega} = \bigcup_{S_{\boldsymbol{\alpha}}\in \Omega} {\rm supp} D_{S_{\boldsymbol{\alpha}}},
\end{equation}
where $D_{S_{\boldsymbol{\alpha}}}$ is the distribution w.r.t. random variable $S_{\boldsymbol{\alpha}}$.

\textbf{Model and Risks.}
Here we introduce some necessary concepts about models and risks. Denote $\mathrm{f}_{\mathbf{w}}: \mathcal{X} \rightarrow \mathbb{R}^{K}$  by the model depending on the parameters $\mathbf{w} \in \mathcal{W}$, where $\mathcal{W}$ is the parameter space. 
Given a loss $\ell$  w.r.t. training domain $D_{X_{l} Y_{l}}$, the training domain risk w.r.t. the model $\mathbf{f}_{\mathbf{w}}$ is
\begin{align}
R_{l}(\mathbf{w})=\mathbb{E} \ \ell\left(\mathbf{f}_{\mathbf{w}} ; X_l, Y_l \right)= \mathbb{E} \ \ell\left(\mathbf{f}_{\mathbf{w}}\circ \mathbf{G} ; S_l,C, Y_l\right),
\end{align}
and the corresponding empirical risk w.r.t.  $\mathbf{f}_{\mathbf{w}}$  is
\begin{align}
\widehat{R}_{l}(\mathbf{w})=\frac{1}{n_{l}} \sum_{i=1}^{n_{l}} \ell\left(\mathbf{f}_{\mathbf{w}} ; \mathbf{x}_{l}^{i}, y_{l}^{i}\right).
\end{align}
Lastly, the target domain risk w.r.t.  $\mathbf{f}_{\mathbf{w}}$  is defined as follows:
\begin{align}
R_{t}(\mathbf{w})=\mathbb{E} \ell\left(\mathbf{f}_{\mathbf{w}} ; X_t, Y_t \right) = \mathbb{E} \ \ell\left(\mathbf{f}_{\mathbf{w}}\circ \mathbf{G} ; S_t,C, Y_t\right).
\end{align}

\section{Learning Strategy}
In this section, we introduce our main motivation and develop a theoretical framework to support our insight and guide the algorithm design.

\subsection{Motivation}

To generalize well on the unknown target domain, it is one of the most effective strategies to generate new domains to enhance the performance of DNNs \cite{zhou2020learning,zhou2020deep,wang2021learning}. There is an underlying intuition that learning with many generated domains could make DNNs robust against domain shifts.
However, it remains challenging how to mitigate the distribution discrepancy between the generated domains and target domains. Accordingly, the generated domains may fail to promote generalizability or even cause performance degradation of DNNs. The reason is the invisibility of the target domain, which leads to an unmeasurable distribution discrepancy between the generated and the target domain \cite{liang2021pareto}.

Distributionally Robust Optimization (DRO) is a possible strategy to tackle distribution discrepancy \cite{sinha2017certifying, mehra2022certifying}. This is because DRO extends one distribution to a distribution space, i.e., uncertainty set, and trains models with the worst-case distribution in the uncertainty set. By ensuring uniformly well performance inside the uncertainty set around the training domains, DRO can enlarge the influence of the training domains and thus shrink the distribution discrepancy between training and test domains. However, employing DRO to DG has shown limited performance improvement in practice \cite{shen2021towards}. The failure of DRO may be related to the overwhelmingly large property of the employed uncertainty set. Such a large uncertainty set may introduce unrelated domains containing semantics inconsistently with training domains. Consequently, models trained over the set can make decisions with fairly low confidence, known as the low confidence issue \cite{hu2018does,frogner2019incorporating,shen2021towards}.

\subsection{Moderately Distributional Exploration}

To fully unleash the potential of DRO in DG, we propose moderately distributional exploration to perform distribution exploration MODE in an uncertainty \emph{subset}, which shares the same semantic factors with the training domains, avoiding the exploration in the direction of semantics. The insight lies in that merely exploring the semantically related subset can shrink the space of the uncertainty set, mitigating the low confidence issue.

The considered uncertainty subset is used for exploiting the worst-case distribution, i.e., performing DRO on the subset of non-semantic factor $S$, which can be captured as follows:
\begin{align}
\min_{\mathbf{w}\in \mathcal{W}} R_{\Omega}(\mathbf{w})=\min_{\mathbf{w}\in \mathcal{W}}\max_{{ S_{\boldsymbol{\alpha}} \in \Omega} }\mathbb{E} \ \ell\left(\mathbf{f}_{\mathbf{w}}\circ \mathbf{G} ; S_{\boldsymbol{\alpha}},C, Y\right).
\end{align}
In practical scenarios, it is challenging to exactly estimate a distribution under DG scenarios, resulting in a restricted searching capacity (more discussions are shown in Appendix \ref{Discussion}). Therefore, we propose to explore the non-semantic factor for each sample rather than for each domain:
\begin{equation}
\min_{\mathbf{w}\in \mathcal{W}} R_{\mathcal{S}_{\Omega}}(\mathbf{w}) = \min_{\mathbf{w}\in \mathcal{W}} \mathbb{E} \max_{\mathbf{s}\in \mathcal{S}_{\Omega}} \ell\left(\mathbf{f}_{\mathbf{w}}\circ \mathbf{G} ; \mathbf{s},C, Y\right),
\end{equation}
where $\mathcal{S}_{\Omega}$ stands for the non-semantic space introduced in Eq. \eqref{non-semanticspace} and $C$ represents the semantic random variable. The corresponding empirical risk $\widehat{R}_{\mathcal{S}_{\Omega}}(\mathbf{w})$ w.r.t. $R_{\mathcal{S}_{\Omega}}(\mathbf{w})$ is:
\begin{align}\label{empiricaltarget}
 \frac{1}{\sum_{l=1}^{N} n_{l}} \sum_{l=1}^{N} \sum_{i=1}^{n_{l}} \max_{\mathbf{s}\in \mathcal{S}_{\Omega} } \ell\left(\mathbf{f}_{\mathbf{w}}\circ \mathbf{G} ;\mathbf{s},\mathbf{c}_{l}^{i}, y_{l}^{i}\right),
\end{align}
where $\mathbf{c}_l^i$ is the element of the semantic part of $\mathbf{G}^{-1}(\mathbf{x}_l^i)$.
Besides the worst-case optimization, following previous works \cite{zhou2020deep,zhou2020learning,xu2021fourier}, we further introduce the empirical risk in our optimization. Namely, both the exploited and original data are used for model training with a parameter $\beta$ used for trading off the risks:
\begin{align}\label{cls}
\min_{\mathbf{w} \in \mathcal{W}} \widehat{R}^{\beta}_{\boldsymbol{\lambda}}(\mathbf{w}) = (1-\beta ) \sum_{l=1}^N \lambda_l \widehat{R}_{l}(\mathbf{w})+\beta \widehat{R}_{\mathcal{S}_{\Omega}}(\mathbf{w}),
\end{align}
where $\boldsymbol{\lambda}=[\lambda_1,...,\lambda_N]\in \Delta_N$ are fixed weights.

\subsection{Theoretical Insights of MODE}
Here, we give a learning theory to provide theoretical support for our proposed learning strategy. The main conclusions are summarized as follows:

$\bullet$ Theorem \ref{thm1} shows that the empirical model given by Eq. \eqref{cls} can achieve consistent learning performance.

$\bullet$ Theorem \ref{thm4} shows the risk estimation for the unknown target domain w.r.t. the empirical model given by Eq. \eqref{cls}.

Before giving detailed theoretical results, we introduce several necessary concepts.  Specifically, we use notation ${R}^{\beta}_{\boldsymbol{\lambda}}(\mathbf{w})$ to represent the ideal form of $\widehat{R}^{\beta}_{\boldsymbol{\lambda}}(\mathbf{w}) $ in Eq. \eqref{cls}:
\begin{align}\label{clsif}
R^{\beta}_{\boldsymbol{\lambda}}(\mathbf{w}) =  (1-\beta ) \sum_{l=1}^N \lambda_l R_{l}(\mathbf{w})+\beta R_{\mathcal{S}_{\Omega}}(\mathbf{w}).
\end{align}

To measure the distribution discrepancy between the two domains, we use Optimal Transport Cost \cite{sinha2017certifying, mehra2022certifying} defined as follows: 
\newtheorem{defin}{Definition}
\begin{defin}\label{def1}
(Optimal Transport Cost and Wasserstein-1 Distance \cite{villani2009optimal,villani2021topics}). Given a cost function  $c$  :  $\mathcal{Z} \times \mathcal{Z} \rightarrow \mathbb{R}_{+}$ , the \textit{Optimal Transport Cost} w.r.t.  $c$  between two probability distances  $D$  and  $D^{\prime}$  is defined as:
\begin{align*} 
\mathrm{W}_{c}\left(D, D^{\prime}\right)=\inf _{\pi \in \Pi\left(D, D^{\prime}\right)} \mathbb{E}_{\left(\mathbf{x}, \mathbf{x}^{\prime}\right) \sim \pi} c\left(\mathbf{x}, \mathbf{x}^{\prime}\right),
\end{align*}
where $\Pi\left(D, D^{\prime}\right)$ is the space of all couplings for $D$ and $D^{\prime}$. Furthermore, if the cost $c$ is a \textit{metric}, then the \textit{Optimal Transport Cost} is also called the \textit{Wasserstein-1} distance.
\end{defin}

Similar to \citet{sinha2017certifying}, our results rely on the usual covering number \cite{vershynin2018high} for the model classes  $\mathcal{F}=\left\{\ell\left(\mathbf{f}_{\mathbf{w}} ; \cdot\right): \mathbf{w} \in \mathcal{W}\right\} $ to represent the complexity. Intuitively the covering numbers  $\mathcal{N}\left(\mathcal{F}, \epsilon, L^{\infty}\right)$ is the minimal number of  $L^{\infty}$  balls of radius  $\epsilon>0$  needed to cover the model classes $ \mathcal{F}$, respectively. The rigorous definition on covering number is given in the Appendix \ref{cn}.

We first show that our approach can achieve consistent learning performance under mild assumptions.
\theoremstyle{plain}
\newtheorem{lemma}{Lemma}
\newtheorem{thm}{\bf Theorem}
\newtheorem{coro}{\bf Corollary}
\begin{thm}\label{thm1}
(Excess Generalization Bound). Assume that
$\bullet$ $0 \leq \ell (\mathbf{f_{w}} ;\mathbf{x}, y) \leq M_{\ell}<+\infty$,

$\bullet$  $S_{1}, S_{2},...,S_{l}$ are mutually independent,

$\bullet$  $S_{l} \perp \!\!\! \perp C$ and $Y_{l} = Y_{t} = Y$,$\forall l=1,....,N$.

Let $\widehat{\mathbf{w}}$ be the solution of Eq. \eqref{cls}, i.e.,
\begin{align*}
\centering
\widehat{\mathbf{w}} \in \underset{\mathbf{w} \in \mathcal{W}}{\arg \min }~ \widehat{R}^{\beta }_{\boldsymbol{\lambda}}(\mathbf{w}).
\end{align*}
With the probability at least  $1-4 e^{-t}>0$ ,
\begin{align}\label{thm1bound}
\begin{split}
 \space  R^{\beta }_{\boldsymbol{\lambda}}(\widehat{\mathbf{w}})-\min _{\mathbf{w} \in \mathcal{W}} R^{\beta }_{\boldsymbol{\lambda}}(\mathbf{w} ) \leq   \epsilon_{\boldsymbol{\lambda}}^{\beta}(n_1,...,n_N; t),
\end{split}
\end{align}
where $\epsilon_{\boldsymbol{\lambda}}^{\beta}(n_1,...,n_N; t)$ is equal to
\begin{align*}
\begin{split}
&(1-\beta) \sum_{l=1}^N\frac{b_{0} M_{\ell}\lambda_l }{\sqrt{n_l}} \int_{0}^{1} \sqrt{\log \mathcal{N}\left(\mathcal{F}, M_{\ell} \epsilon, L^{\infty}\right)} d \epsilon \\
+&\beta \frac{b_{1} M_{\ell}}{\sqrt{\sum_{l=1}^N n_l}} \int_{0}^{1} \sqrt{\log \mathcal{N}\left(\mathcal{F}, M_{\ell} \epsilon, L^{\infty}\right)} \mathrm{d} \epsilon \\
+&2(1-\beta)\sum_{l=1}^N \lambda_{l} M_{\ell} \sqrt{\frac{2 t}{n_l}}+2\beta M_{\ell} \sqrt{\frac{2 t}{\sum_{l=1}^N n_{l}}},
\end{split}
\end{align*}
here  $b_{0}$ and $b_{1}$  are uniform constants.
\end{thm}

Under proper conditions, one can show that the bound (Eq. \eqref{thm1bound}) can attained $\tilde{\mathcal{O}}( \sum_{l=1}^N \frac{\lambda_l}{\sqrt{n_l}}) +\tilde{\mathcal{O}}(  \frac{1}{\sqrt{\sum_{l=1}^N n_l}})$, i.e.,
\begin{align*}
\begin{split}
R^{\beta }_{\boldsymbol{\lambda}}(\widehat{\mathbf{w}})-\min _{\mathbf{w} \in \mathcal{W}} R^{\beta }_{\boldsymbol{\lambda}}(\mathbf{w} ) \leq \tilde{\mathcal{O}}( \sum_{l=1}^N \frac{\lambda_l}{\sqrt{n_l}}) +\tilde{\mathcal{O}}(  \frac{1}{\sqrt{\sum_{l=1}^N n_l}}).
\end{split}
\end{align*}
Corollary \ref{thm2} in Appendix \ref{coro1} gives an example supporting this claim. Next, the following theorem gives a learning bound to estimate the unknown target domain risk. 

\begin{thm}\label{thm4}
(Risk Estimation).
Given the same conditions in Theorem \ref{thm1} and let $\widehat{\mathbf{w}}$ be the solution of Eq. \eqref{cls}, i.e.,
\begin{align*}
\centering
\widehat{\mathbf{w}} \in \underset{\mathbf{w} \in \mathcal{W}}{\arg \min }~ \widehat{R}^{\beta }_{\boldsymbol{\lambda}}(\mathbf{w}).
\end{align*}
If the cost function $c(\cdot,\cdot): \mathcal{S}\times \mathcal{S}\rightarrow \mathbb{R}_{+}$ is a continuous metric and
 $\ell(\mathbf{f}_{\mathbf{w}}\circ \mathbf{G}; \mathbf{s},\mathbf{c},y)$ is  $L_{c}$-Lipschitz w.r.t. $c$, i.e., 
\begin{equation*}
|\ell(\mathbf{f}_{\mathbf{w}}\circ \mathbf{G}; \mathbf{s},\mathbf{c},y)-\ell(\mathbf{f}_{\mathbf{w}}\circ \mathbf{G}; \mathbf{s}',\mathbf{c},y)|\leq L_{c}c(\mathbf{s},\mathbf{s}'),
\end{equation*} 
then with the probability at least  $1-4 e^{-t}>0$ ,
\begin{align}\label{thm4bound}
\begin{split}
 \space & R_{t}(\widehat{\mathbf{w}})-\min _{\mathbf{w} \in \mathcal{W}} R^{\beta }_{\boldsymbol{\lambda}}(\mathbf{w} ) \\
\leq &  (1-\beta) L_{c} \sum_{l=1}^{N} \lambda_l \mathrm{W}_{c}\left(D_{S_{\mathrm{t}}}, D_{S_{l}}\right)\\
 +&\beta L_{c} \min_{S_{\boldsymbol{\alpha}} \in \Omega} \mathrm{W}_{c}\left(D_{S_{\mathrm{t}}}, D_{S_{\boldsymbol{\alpha}}}\right)+ \epsilon_{\boldsymbol{\lambda}}^{\beta}(n_1,...,n_N; t),
\end{split}
\end{align}
where $\epsilon_{\boldsymbol{\lambda}}^{\beta}(n_1,...,n_N; t)$ is introduced in Theorem \ref{thm1}.
\end{thm}

\textbf{The Trade-off of $\Omega$.}
From Theorem \ref{thm4}, we can see that the distribution discrepancy between the target distribution and distributions in distribution searching space can hurt the network's generalization ability due to the term $\beta L_{c} \min_{S_{\boldsymbol{\alpha}} \in \Omega} \mathrm{W}_{c}\left(D_{S_{\mathrm{t}}}, D_{S_{\boldsymbol{\alpha}}}\right)$.
When $\Omega$ is large enough to include $D_{S_{\mathrm{t}}}$, this term becomes $0$.
Although a larger $\Omega$ will decrease this term, the approximate risk $\min _{\mathbf{w} \in \mathcal{W}}R^{\beta}_{\boldsymbol{\lambda}}(\mathbf{w})$ will be increased, which means that there is a trade-off between these two terms about the choice of $\Omega$.

\textbf{The Trade-off of $\beta$.}
It can be observed that the increase of $\beta$ leads to the decrease of $(1-\beta) L_{c} \sum_{l=1}^{N} \lambda_l \mathrm{W}_{c}\left(D_{S_{\mathrm{t}}}, D_{S_{l}}\right)+\beta L_{c} \min_{S_{\boldsymbol{\alpha}} \in \Omega} \mathrm{W}_{c}\left(D_{S_{\mathrm{t}}}, D_{S_{\boldsymbol{\alpha}}}\right)$. 
But $\beta$ also determines the value of $\min _{\mathbf{w} \in \mathcal{W}}R^{\beta}_{\boldsymbol{\lambda}}(\mathbf{w})$ and $\epsilon_{\boldsymbol{\lambda}}^{\beta}(n_1,...,n_N; t)$, leading to the trade-off of the choice of $\beta$ in practice.

\section{Realization of MODE}
\label{F}
Motivated by our theoretical insights, we propose a realization of MODE by using some existing style transfer approaches, which will be introduced below \footnote{Code: \url{github.com/Rxsw/MODE}.}.

$\bullet$ \textbf{Loss Functions.}
Following \citet{li2017deeper,xu2021fourier}, we use the cross entropy loss as $\ell$.

$\bullet$ \textbf{Algorithm Design.}
The key in algorithm design is the implementation of the causal mechanism $\mathbf{G}$.
In practice, we use Fourier-based transfer and AdaIN-based transfer to construct $\mathbf{G}$ in our algorithm. 

Other style transfer methods not introduced in this paper can also be applied to our approach in the same way.

\textbf{MODE-F: Fourier-based MODE.}
The Fourier-based transfer \cite{xu2021fourier} is considered able to separate the stylistic information from the semantic information by using the discrete Fourier transform to decompose the data $\mathbf{x}$ into its amplitude $\mathcal{A}\left(\mathbf{x}\right)$ and phase $\mathcal{P}\left(\mathbf{x}\right)$.
It is believed that the phase information contains more high-level semantics and is not easily
affected by domain shifts \cite{oppenheim1979phase,oppenheim1981importance}, which makes it possible to create samples of different styles by mixing amplitudes. 

In practice, we treat the amplitude $\mathcal{A}\left(\mathbf{x}\right)$ as the non-semantic factor $S$ and treat the phase $\mathcal{P}\left(\mathbf{x}\right)$ as the semantic factor $C$. Following our approach, we explore $\mathcal{A}\left(\mathbf{x}\right)$ corresponding to the worst-case generated data fixing $\mathcal{P}\left(\mathbf{x}\right)$.

Since Fourier-based transfer creates a new sample by mixing amplitudes and maintaining the original phase, we have:
\begin{align}
\begin{split}
\hat{\mathcal{A}}_{\gamma}\left(\boldsymbol{\alpha},\mathbf{x}\right) & = \gamma[\alpha_{0}\mathcal{A}\left(\mathbf{x}\right) +\sum_{l=1}^{M} \alpha_{l} \mathcal{A}\left(\mathbf{x}_{l}\right)]\\ &+(1-\gamma)\mathcal{A}\left(\mathbf{x}\right), \end{split}\\
\mathbf{G}_{\gamma}( \boldsymbol{\alpha},\mathbf{x}) & =  \operatorname{iFFT} \left[\hat{\mathcal{A}}_{\gamma}\left(\boldsymbol{\alpha},\mathbf{x}\right) * e^{-j * \mathcal{P}\left(\mathbf{x}\right)}\right], 
\end{align}
where $\mathcal{A}\left(\mathbf{x}_{l}\right),l=1,\cdots,M$ are the amplitudes of $M$ other images and $\boldsymbol{\alpha}\in \Delta_{M+1}$.
So that we could control the direction of stylization by changing $\boldsymbol{\alpha}=[\alpha_{0},\alpha_{1},\cdots,\alpha_{M}]$ and control the degree of stylization by changing $\gamma$. More details of Fourier-based MODE are shown in Appendix \ref{fb}.

\textbf{MODE-A: AdaIN-based MODE.}
AdaIN \cite{huang2017arbitrary} is one of the representative methods of neural style transfer. It uses the mean $\boldsymbol{\mu}$ and std $\boldsymbol{\sigma}$ of feature map output by the fixed encoder $\operatorname{E}$ to control style information and trains a decoder $\operatorname{D}$ to restore the stylized image from the feature map whose mean and std had been changed. 

In practice, we treat these mean $\boldsymbol{\mu}$ and std $\boldsymbol{\sigma}$ as the non-semantic factor $S$ and treat the normalized feature map as the semantic factor $C$. Following our approach, by fixing the normalized feature map, we explore $\boldsymbol{\mu}$ and $\boldsymbol{\sigma}$ which corresponds to the worst-case generated data.

Since AdaIN-based transfer creates a new sample by mixing mean-std and maintaining the original normalized feature map, we first calculate mixed mean and mixed std:
\begin{align}
\tilde{\boldsymbol{\mu}}(\boldsymbol{\alpha},\mathbf{x}) & = \alpha_{0} \boldsymbol{\mu}(\operatorname{E}(\mathbf{x})) +\sum_{l=1}^{M}\alpha_{l} \boldsymbol{\mu}(\operatorname{E}(\mathbf{x}_{l})),\\
\tilde{\boldsymbol{\sigma}}(\boldsymbol{\alpha},\mathbf{x}) & = \alpha_{0} \boldsymbol{\sigma}(\operatorname{E}(\mathbf{x})) +\sum_{l=1}^{M}\alpha_{l} \boldsymbol{\sigma}(\operatorname{E}(\mathbf{x}_{l})),
\end{align}
and then apply mixed mean and mixed std to the normalized feature map and restore image by decoder $\operatorname{D}$: 
\begin{align}
\tilde{\mathbf{Z}}(\boldsymbol{\alpha},\mathbf{x})  &= \tilde{\boldsymbol{\mu}}(\boldsymbol{\alpha},\mathbf{x})  +\tilde{\boldsymbol{\sigma}}(\boldsymbol{\alpha},\mathbf{x}) \frac{\operatorname{E}(\mathbf{x})-\boldsymbol{\mu}(\operatorname{E}(\mathbf{x}))}{\boldsymbol{\sigma}(\operatorname{E}(\mathbf{x}))},\\
\mathbf{G}_{\gamma}( \boldsymbol{\alpha},\mathbf{x}) & = \operatorname{D}(\gamma\tilde{\mathbf{Z}}(\boldsymbol{\alpha},\mathbf{x})+(1-\gamma)\operatorname{E}(\mathbf{x})),
\end{align}
where $\operatorname{E}(\mathbf{x}_{l}),l=1,\cdots,M$ are the feature map of $M$ other images and $\boldsymbol{\alpha}\in \Delta_{M+1}$.
So that we could control the direction of stylization by changing $\boldsymbol{\alpha}=[\alpha_{0},\alpha_{1},\cdots,\alpha_{M}]$ and control the degree of stylization by changing $\gamma$. More details of AdaIN-based MODE are shown in Appendix \ref{ab}.

\textbf{The Non-semantic Space $\mathcal{S}_{\Omega}$.}
In each iteration, the non-semantic space $\mathcal{S}_{\Omega}$ in Eq. \eqref{empiricaltarget} is defined as:
\begin{equation}
{\mathcal{S}_{\Omega}} = \{ s_{\boldsymbol{\alpha}} = \sum_{l=0}^M \alpha_l s_{l} : \boldsymbol{\alpha}=[\alpha_0,...,\alpha_M]\in \Delta_{M+1}\},
\end{equation}
where $s_{l}$ is $\mathcal{A}\left(\mathbf{x}_l\right)$ in MODE-F and $(\boldsymbol{\mu},\boldsymbol{\sigma})_l$ in MODE-A.

\textbf{Update $\boldsymbol{\alpha}$.} In each iteration, exploring the optimal $\boldsymbol{\alpha}$ requires multiple inner steps, leading the training time to increase exponentially. To speed up the exploring process, inspired by methods of adversarial attacks \cite{szegedy2013intriguing,goodfellow2014explaining,madry2017towards,zhang2020principal}, we use the gradient's direction and fixed step size $\mu$ to update $\boldsymbol{\alpha}$ every time after generating augmented data $\hat{\mathbf{x}}$ using $\mathbf{G}_{\gamma}$:
\begin{align}
\label{updatea22}
\boldsymbol{\alpha}^{k}  = \operatorname{Normalize}(\boldsymbol{\alpha}^{k-1}+\mu \operatorname{sign}(\nabla_{\boldsymbol{\alpha}} \ell(\mathbf{f_{w}};\hat{\mathbf{x}}^{k-1},y)) ),
\end{align}  
where $\hat{\mathbf{x}}^{k-1} =\mathbf{G}_{\gamma}( \boldsymbol{\alpha}^{k-1},\mathbf{x})$ and $\boldsymbol{\alpha}\in \Delta_{M+1}$.

\begin{algorithm}[!t]
   \caption{\textbf{MODE}}
   \label{alg:DR}
\begin{algorithmic}
   \STATE {\bfseries Input:} training set, batch size $n$, number of inner steps $K$, number of style provider $M$, step size $\mu$, model architecture parametrized by $\mathbf{w}$, hyperparameter $\beta$ and $\gamma$, the causal mechanism $\mathbf{G}_{\gamma}$
   \STATE {\bfseries Output:} Robust model $\mathbf{f_{w}}$
   \STATE Randomly initialize model $\mathbf{f_{w}}$, or initialize model with pre-trained configuration
   \REPEAT
   \STATE Read a mini-batch $\boldsymbol{x} = [\mathbf{x}_{1}, ..., \mathbf{x}_{n}]$, $\boldsymbol{y} = [y_{1}, ..., y_{n}]$ from the training set
   \STATE \textit{\textcolor{blue}{\# Maximization: Exploration}}
   \FOR{$i=1$ {\bfseries to} $n$ (\textit{in parallel})} 
   \STATE Initialize $\alpha_{0},\alpha_{1},\cdots,\alpha_{M}$ as $\boldsymbol{\alpha}^{0}_{i}$
   \STATE Initialize $\hat{\mathbf{x}}_{i}^{0} \leftarrow \mathbf{G}_{\gamma}( \boldsymbol{\alpha}^{0}_{i},\mathbf{x}_{i})$
   \FOR{$k=1$ {\bfseries to} $K$}
   \STATE \textit{\textcolor{blue}{\# Inner Step}}
   \STATE $\tilde{\boldsymbol{\alpha}}^{k}_{i} \leftarrow $ Eq.\eqref{updatea22}
   \STATE $\hat{\mathbf{x}}_{i}^{k} \leftarrow \mathbf{G}_{\gamma}( \boldsymbol{\alpha}^{k}_{i},\mathbf{x}_{i})$
   \ENDFOR
   \ENDFOR
   \STATE \textit{\textcolor{blue}{\# Minimization: Update Model}}
   \STATE $\widehat{R}_{l}(\mathbf{w})=\frac{1}{n} \sum_{i=1}^{n}  \ell\left(\mathbf{f}_{\mathbf{w}} ; \mathbf{x}_{i}, y_{i}\right)$
   \STATE $\widehat{R}_{s}(\mathbf{w})=\frac{1}{n} \sum_{i=1}^{n}  \ell(\mathbf{f}_{\mathbf{w}} ; \hat{\mathbf{x}}_{i}^{K}, y_{i})$
   \STATE $\mathbf{w} \leftarrow \mathbf{w}-\operatorname{lr} \nabla_{\mathbf{w}}\left[(1-\beta)\widehat{R}_{\mathrm{l}}(\mathbf{w})+\beta \widehat{R}_{\mathrm{s}}(\mathbf{w})\right]$
   \UNTIL{convergence}
\end{algorithmic}
\end{algorithm}

$\bullet$ \textbf{Stochastic Realization\footnote{In practice, we set $\boldsymbol{\lambda}:\{ \lambda_l = n_l / \sum_{i=1}^{N}n_i,l=1,\cdots,N \}$ aforementioned in Eq. \eqref{cls}, which means that each sample in different domains is given the same weight.}.}
Algorithm \ref{alg:DR} gives a stochastic realization for MODE, where minibatch is randomly sampled in each iteration.
We first explore the worst-case generated data by \textit{Maximization}. Specifically, the value of $\boldsymbol{\alpha}$ is uniformly initialized and updated by Eq.\eqref{updatea22} for $K$ steps. 
After getting the final augmented data, we update the model parameters $\mathbf{w}$ by \textit{Minimization} which is one step of minibatch gradient descent with $\widehat{R}_{l}(\mathbf{w})$, $\widehat{R}_{s}(\mathbf{w})$ and $\beta$.

\section{Experiments}
\label{Experiments}
In this section, we demonstrate the superiority of our approach on several DG benchmarks. 

\begin{table*}[h]
\centering
\caption{Leave-one-domain-out classification accuracies (in $\%$) on PACS and Office-Home in ResNet18. The best and second-best results are highlighted in bold and underlined, respectively. DRO$\dagger$ is the result of directly applying Group DRO \cite{sagawa2019distributionally} to DG. CIRL$\dagger$ \cite{lv2022causality} is the result of reproducing using the authors’ official codes and following the same settings in the original papers.}
\scalebox{0.88}{
\begin{tabular}{c|ccccc|ccccc} \hline
\label{tab:po}
Dataset       & \multicolumn{5}{c|}{PACS}                                                                             & \multicolumn{5}{c}{Office-Home}                                                                     \\ \hline
Methods       & A             & C              & P             & \multicolumn{1}{c|}{S}              & Avg.           & A             & C              & P             & \multicolumn{1}{c|}{R}             & Avg.          \\ \hline
DeepAll \cite{zhou2020deep}       & 77.6          & 76.8           & 95.9          & \multicolumn{1}{c|}{69.5}           & 79.9           & 57.9          & 52.7           & 73.5          & \multicolumn{1}{c|}{74.8}          & 64.7          \\
Jigen \cite{carlucci2019domain}         & 79.4          & 75.3           & 96.0          & \multicolumn{1}{c|}{71.4}           & 80.5           & 53.0            & 47.5           & 71.5          & \multicolumn{1}{c|}{72.8}          & 61.2          \\
MMD-AAE \cite{li2018domain}       & 75.2          & 72.7           & 96.0          & \multicolumn{1}{c|}{64.2}           & 77.0           & 56.5          & 47.3           & 72.1          & \multicolumn{1}{c|}{74.8}          & 62.7          \\
CrossGrad \cite{shankar2018generalizing}     & 79.8          & 76.8           & 96.0          & \multicolumn{1}{c|}{70.2}           & 80.7           & 58.4          & 49.4           & 73.9          & \multicolumn{1}{c|}{75.8}          & 64.4          \\
DDAIG \cite{zhou2020deep}         & 84.2          & 78.1           & 95.3          & \multicolumn{1}{c|}{74.7}           & 83.1           & 59.2          & 52.3           & 74.6          & \multicolumn{1}{c|}{76.0}          & 65.5          \\
L2A-OT \cite{zhou2020learning}        & 83.3          & 78.2           & 96.2          & \multicolumn{1}{c|}{73.6}           & 82.8           & \textbf{60.6}          & 50.1           & \textbf{74.8}          & \multicolumn{1}{c|}{77.0}          & 65.6          \\
MixStyle \cite{zhou2021domain}      & 83.0          & 78.6           & 96.3    & \multicolumn{1}{c|}{71.2}           & 82.3           & 58.7          & 53.4           & 74.2          & \multicolumn{1}{c|}{75.9}          & 65.5          \\
MatchDG \cite{mahajan2021domain}        & 81.3          & {\ul80.7}           & {\ul 96.5} & \multicolumn{1}{c|}{79.7}           & 84.6           & -             & -              & -             & \multicolumn{1}{c|}{-}             & -             \\
CICF \cite{li2021confounder}          & 80.7          & 76.9           & 95.6          & \multicolumn{1}{c|}{74.5}           & 81.9           & 57.1          & 52.0           & 74.1          & \multicolumn{1}{c|}{75.6}          & 64.7          \\
RSC \cite{huang2020self}            & 83.4          & 80.3           & 96.0          & \multicolumn{1}{c|}{80.9}           & 85.2           & 58.4          & 47.9           & 71.6          & \multicolumn{1}{c|}{74.5}          & 63.1          \\
FACT \cite{xu2021fourier}          & \textbf{85.9}    & 79.4           & \textbf{96.6}          & \multicolumn{1}{c|}{80.8}           & 85.7           & {\ul 60.3}    & 54.9           & {\ul 74.5}    & \multicolumn{1}{c|}{ \textbf{76.6}}    & 66.6          \\
CIRL$\dagger$ \cite{lv2022causality} & {\ul 85.5{\tiny$\pm$0.2}}          & 79.6{\tiny$\pm$0.3}           & 96.1{\tiny$\pm$0.5}          & \multicolumn{1}{c|}{{\ul 82.7{\tiny$\pm$0.3}}}           & {\ul 86.0}           & 58.6{\tiny$\pm$0.2}          & {\ul 55.4{\tiny$\pm$0.1}}           & 73.8{\tiny$\pm$0.3}          & \multicolumn{1}{c|}{75.1{\tiny$\pm$0.1}}          &{\ul 65.7}           \\ \hline
DRO$\dagger$ \cite{sagawa2019distributionally}          & 82.5{\tiny$\pm$1.3}          & 79.1{\tiny$\pm$1.0}           & 95.1{\tiny$\pm$0.2}          & \multicolumn{1}{c|}{78.5{\tiny$\pm$1.2}}           & 83.2           & 52.8{\tiny$\pm$0.3}          & 49.2{\tiny$\pm$0.6}          & 67.6{\tiny$\pm$0.4}          & \multicolumn{1}{c|}{70.8{\tiny$\pm$0.5}}          & 60.1          \\ \hline
MODE-F (ours)  & 84.5{\tiny$\pm$0.6}         & 80.4{\tiny$\pm$0.8}          & 95.5{\tiny$\pm$0.2}         & \multicolumn{1}{c|}{82.2{\tiny$\pm$0.7}}          & 85.7          & 57.7{\tiny$\pm$0.1}         & 54.0{\tiny$\pm$0.4}          & 73.9{\tiny$\pm$0.2}         & \multicolumn{1}{c|}{{\ul 76.1{\tiny$\pm$0.3}} }       & 65.4         \\
MODE-A (ours)   & 84.4{\tiny$\pm$0.9}         & \textbf{81.9{\tiny$\pm$0.9}} & 95.2{\tiny$\pm$0.3}         & \multicolumn{1}{c|}{\textbf{85.8{\tiny$\pm$0.3}}} & \textbf{86.9} & 60.1{\tiny$\pm$0.8}         & \textbf{57.3{\tiny$\pm$0.6}} & 74.2{\tiny$\pm$0.5}         & \multicolumn{1}{c|}{76.0{\tiny$\pm$0.2}}         &   \textbf{66.9}  \\ \hline
\end{tabular}}
\end{table*}

\subsection{Datasets}
Following previous works \cite{zhou2020deep,huang2020self,xu2021fourier,lv2022causality}, we evaluate our approach on three standard DG benchmark datasets described below. More results include VLCS \cite{torralba2011unbiased}, DomainNet \cite{peng2019moment} and Mini-DomainNet \cite{zhou2021domain} are given in the Appendix \ref{moreresult}.

\textbf{Digits-DG} \cite{zhou2020deep} consists of 4 digit datasets: MNIST (M) \cite{lecun1998gradient}, MNIST-M (M-M) \cite{ganin2015unsupervised}, SVHN (SV) \cite{netzer2011reading} and SYN (SY) \cite{ganin2015unsupervised} which differ in font style, color, and background. Following \cite{zhou2020deep}, we randomly select 600 images of each class from each domain, where $80\%$ of the selected images are used for training, and $20\%$ are used for validation.

\textbf{PACS} \cite{li2017deeper} is an object recognition benchmark designed for DG which consists of 9,991 images from 4 domains namely Photo (P), Art-painting (A), Cartoon (C), Sketch (S) with large style discrepancy and has seven categories in each domain. For a fair comparison, we use the training-validation-test split provided by \cite{li2017deeper}.

\textbf{Office-Home} \cite{venkateswara2017deep} consists of 15,500 images of 65 classes from four domains: Art (A), Clipart (C), Product (P), Real-World (R), which differ in image style and viewpoint. For a fair comparison, we use the training-validation-test split same as \citet{xiao2021bit}.

\subsection{Implementation Details}
Following the commonly used leave-one-domain-out strategy \cite{li2017deeper,xu2021fourier}, DG models are evaluated using one domain, after training on the other domains.

\begin{table}[!t]
\centering
\vspace{-3mm}
\caption{Leave-one-domain-out classification accuracies (in $\%$) on Digit-DG. 
The best and second-best results are highlighted in bold and underlined, respectively.}
\scalebox{0.65}{
\begin{tabular}{c|cccc|c}\hline
\label{tab:1}
Methods       & M              & M-M            & SV          & SY            & Avg.         \\ \hline
DeepAll \cite{zhou2020deep}       & 95.8           & 58.8           & 61.7          & 78.6           & 73.7           \\
Jigen \cite{carlucci2019domain}         & 96.5           & 61.4           & 63.7          & 74.0           & 73.9           \\
MMD-AAE \cite{li2018domain}       & 96.5           & 58.4           & 65.0          & 78.4           & 74.6           \\
CrossGrad \cite{shankar2018generalizing}     & 96.7           & 61.1           & 65.3          & 80.2           & 75.8           \\
DDAIG \cite{zhou2020deep}         & 96.6           & 64.1           & 68.6          & 81.0           & 77.6           \\
L2A-OT \cite{zhou2020learning}        & 96.7           & 63.9           & 68.6          & 83.2           & 78.1           \\
MixStyle \cite{zhou2021domain}      & 96.5           & 63.5           & 64.7          & 81.2           & 76.5           \\
CICF \cite{li2021confounder}          & 95.8           & 63.7           & 65.8          & 80.7           & 76.5           \\
FACT \cite{xu2021fourier}          & {\ul 97.9}           & 65.6           & 72.4          & {\ul90.3}           & 81.6           \\
CIRL \cite{lv2022causality}          & 96.1           & {\ul 69.8}           & \textbf{76.2} & 87.7           & {\ul82.5}           \\ \hline
MODE-F (ours)     & \textbf{98.5}{\tiny$\pm$0.1} & \textbf{72.7}{\tiny$\pm$0.1} & {\ul73.2 }{\tiny$\pm$0.6}        & \textbf{91.1}{\tiny$\pm$0.4} & \textbf{83.9} \\ \hline
\end{tabular}} \vspace{-1mm}
\end{table}

\textbf{Basic Details.} For Digits-DG, we use the backbone introduced by \cite{zhou2020learning,xu2021fourier}. All images are resized to 32×32. Following \citet{xu2021fourier}, we train the network using an SGD optimizer with a learning rate of 0.05, batch size of 128, a momentum of 0.9, and weight decay 5e-4 for 50 epochs. The learning rate is decayed by 0.1 every 20 epochs. We use random cropping in data augmentation.
For PACS and Office-Home, following \citet{li2017deeper,xu2021fourier}, we use a pre-trained ResNet-18 backbone \cite{he2016deep}, all images are resized to 224×224. We train the network using SGD optimizer with learning rate 5e-4, momentum 0.9, and weight decay 5e-4. We train the model for 80 epochs with batch size 16 and 50 epochs with batch size 32, respectively. The learning rate is decayed by 0.1 every 40 epochs. We use the standard augmentation protocol in \citet{li2017deeper,xu2021fourier}.

\textbf{Method-specific Details.}
The hyperparameters of our approach: The number of inner steps $K$, inner step size $\mu$, $\beta$, $\gamma$, and The number of style providers $M$. The settings of these hyperparameters are shown in Appendix \ref{md}.

\subsection{Experimental Results}

\textbf{Results on Digit-DG\footnote{We only use the Fourier-based method MODE-F for Digit-DG,
since AdaIN could not process low-resolution images originally.}.}
We show the Leave-one-domain-out classification accuracies on Digit-DG in Table \ref{tab:1}. It can be observed that our approach achieves the highest accuracy in most domains and the second-highest accuracy in the remaining domain.4
In particular, in MNIST-M (M-M), which has complex backgrounds and rich colors, our approach exceeds the previous approach of state-of-the-art by 2.7$\%$, which proves the capability of our approach. 

\textbf{Results on PACS.} We show the Leave-one-domain-out classification accuracies on PACS on Table \ref{tab:po}. It can be observed that our approach achieves the highest average accuracy. Our approach also exceeds Group DRO \cite{sagawa2019distributionally} by 3$\%$ on average. In particular, in the most challenging domain Sketch (S), where the image is only composed of simple lines without background, our approach exceeds the previous approach of state-of-the-art by 3$\%$.

\textbf{Results on Office-Home.}
We show the Leave-one-domain-out classification accuracies on Office-Home in Table \ref{tab:po}. 
It can be observed that our approach achieves the highest average accuracy.
Our approach also exceeds Group DRO \cite{sagawa2019distributionally} by 6$\%$ on average. In particular, in the Clipart (C), which is very similar to the domain Sketch in PACS, we still get 1.9\% better than the previous best result, which proves that our approach has consistent performance.

\subsection{Analytical Experiments}

\textbf{Number of Inner Steps $K$.}
Figure \ref{fig:2} is the average loss of samples in different inner steps to the current model changes with the number of epochs in a training process, and the effect of the number of inner steps $K$. It shows that our approach indeed finds samples with higher empirical risk and with the increase in the number of inner steps, the final accuracy also maintains the trend of generally increasing, which proves the effectiveness of our approach.

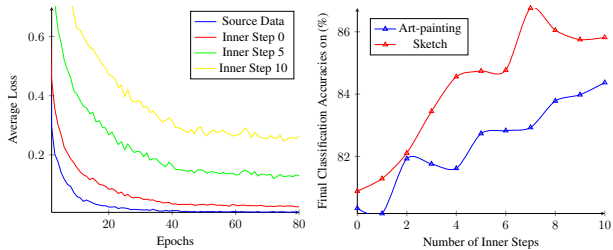
\begin{figure}[htbp] 
\centering 
\begin{minipage}{.24\textwidth}
\begin{tikzpicture}[scale=0.48] 
\begin{axis}[
    xlabel= Epochs, 
    ylabel= Average Loss, 
    tick align=outside,
    ymax=0.7,
    axis y line=left,
    axis x line=bottom
    ]
\addplot [mark = none, blue] table {00.txt};
\addlegendentry{Source Data}
\addplot [mark = none, red] table {0.txt};
\addlegendentry{Inner Step 0}
\addplot [mark = none, green] table {5.txt};
\addlegendentry{Inner Step 5}
\addplot [mark = none, yellow] table {10.txt};
\addlegendentry{Inner Step 10}
\end{axis}
\end{tikzpicture}
\end{minipage}%
\begin{minipage}{.24\textwidth}
\begin{tikzpicture}[scale=0.48] 
\begin{axis}[
    xlabel= Number of Inner Steps, 
    ylabel= Final Classification Accuracies on (\%),
    tick align=outside,
    axis y line=left,
    axis x line=bottom,
    legend style={at={(0.25,0.95)},anchor=north} 
    ]
\addplot[smooth,mark=triangle,blue] plot coordinates { 
    (0,80.34)
    (1,80.17)
    (2,81.93)
    (3,81.76)
    (4,81.62)
    (5,82.735)
    (6,82.83)
    (7,82.93)
    (8,83.78)
    (9,83.98)
    (10,84.37)
};
\addlegendentry{Art-painting}
\addplot[smooth,mark=triangle,red] plot coordinates { 
    (0,80.89)
    (1,81.29)
    (2,82.11)
    (3,83.45)
    (4,84.56)
    (5,84.74)
    (6,84.77)
    (7,86.76)
    (8,86.05)
    (9,85.75)
    (10,85.82)
};
\addlegendentry{Sketch}
\end{axis}
\end{tikzpicture}
\end{minipage}
\caption{The left shows the average loss of augmented samples in different inner steps changes with the number of epochs in a training process. The right shows the effect of the number of inner steps. All results are conducted on the PACS dataset with Sketch (S) or Art-painting (A) as the unknown target domain.}
\label{fig:2}
\end{figure}

\textbf{Hyperparameter $\beta$ and $\gamma$.}
Figure \ref{fig:3} show the final accuracies of different $\beta$ and $\gamma$. Since $\gamma$ controls the degree of stylization, the change of $\gamma$ can be viewed as a change of $\Omega$. It can be observed that there is a trade-off across the choices of $\beta$. The best choice for $\beta$ is between 0.2 and 0.4. We can also observe that choosing a larger $\Omega$ (corresponds to larger $\gamma$) would not always be better, which is consistent with our theory Eq.\eqref{thm4bound}. Some visualization results with too large $\Omega$ are shown in Figure \ref{Fig.bbb} and Figure \ref{fig:bad}. 

\begin{figure}[htbp] 
\centering 
\begin{minipage}{.24\textwidth}
\begin{tikzpicture}[scale=0.48] 
\begin{axis}[
    xlabel= $\beta$, 
    ylabel= $\gamma$, 
    zlabel= Classification Accuracies in $\%$,
    tick align=outside
    ]
\addplot3[surf,] coordinates {
 (0.000000,0.000000,81.585000) (0.000000,0.200000,81.585000) (0.000000,0.400000,81.585000) (0.000000,0.600000,81.585000) (0.000000,0.800000,81.585000) (0.000000,1.000000,81.585000)

 (0.200000,0.000000,81.660000) (0.200000,0.200000,81.050000) (0.200000,0.400000,80.855000) (0.200000,0.600000,82.175000) (0.200000,0.800000,83.125000) (0.200000,1.000000,84.855000)

 (0.400000,0.000000,82.490000) (0.400000,0.200000,80.145000) (0.400000,0.400000,80.975000) (0.400000,0.600000,81.515000) (0.400000,0.800000,83.000000) (0.400000,1.000000,83.635000)
 
 (0.600000,0.000000,80.465000) (0.600000,0.200000,80.610000) (0.600000,0.400000,79.585000) (0.600000,0.600000,82.490000) (0.600000,0.800000,82.395000) (0.600000,1.000000,82.340000)
 
 (0.800000,0.000000,80.340000) (0.800000,0.200000,79.955000) (0.800000,0.400000,79.070000) (0.800000,0.600000,80.830000) (0.800000,0.800000,82.200000) (0.800000,1.000000,82.125000)
 
 (1.000000,0.000000,79.195000) (1.000000,0.200000,77.685000) (1.000000,0.400000,78.220000) (1.000000,0.600000,77.780000) (1.000000,0.800000,77.440000) (1.000000,1.000000,77.240000)

};
\end{axis}
\end{tikzpicture}
\end{minipage}%
\begin{minipage}{.24\textwidth}
\begin{tikzpicture}[scale=0.48] 
\begin{axis}[
    xlabel= $\beta$, 
    ylabel= $\gamma$, 
    zlabel= Classification Accuracies in $\%$,
    tick align=outside
    ]
    ]
\addplot3[surf,] coordinates {
(0.000000,0.000000,78.180000)(0.000000,0.200000,78.180000)(0.000000,0.400000,78.180000)(0.000000,0.600000,78.180000)(0.000000,0.800000,78.180000)(0.000000,1.000000,78.180000)

(0.200000,0.000000,78.160000)(0.200000,0.200000,79.075000)(0.200000,0.400000,80.830000)(0.200000,0.600000,82.090000)(0.200000,0.800000,84.090000)(0.200000,1.000000,86.100000)

(0.400000,0.000000,78.510000)(0.400000,0.200000,79.590000)(0.400000,0.400000,82.190000)(0.400000,0.600000,82.840000)(0.400000,0.800000,84.190000)(0.400000,1.000000,85.820000)

(0.600000,0.000000,77.290000)(0.600000,0.200000,79.740000)(0.600000,0.400000,81.225000)(0.600000,0.600000,83.515000)(0.600000,0.800000,84.750000)(0.600000,1.000000,86.510000)

(0.800000,0.000000,77.220000)(0.800000,0.200000,79.225000)(0.800000,0.400000,80.465000)(0.800000,0.600000,83.145000)(0.800000,0.800000,85.220000)(0.800000,1.000000,85.870000)

(1.000000,0.000000,76.170000)(1.000000,0.200000,76.860000)(1.000000,0.400000,78.110000)(1.000000,0.600000,79.940000)(1.000000,0.800000,76.680000)(1.000000,1.000000,76.760000)

};
\end{axis}
\end{tikzpicture}
\end{minipage}
\caption{Classification accuracies of different $\beta$ and $\gamma$. The left results are conducted on the PACS dataset with Art-painting (A) as the unknown target domain; the right results are conducted with Sketch (S) as the unknown target domain.}
\label{fig:3}
\end{figure}
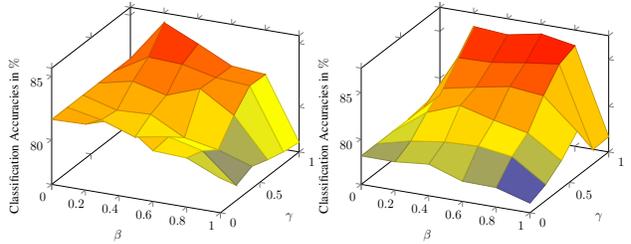

\textbf{Number of Style Provider $M$ and Inner Steps $K$.}
Figure \ref{fig:4} shows the final accuracies of different $\beta$, number of style provider $M$, and number of inner steps $K$. $M$ determines the number of styles that can be used in the exploration, which may affect $\mathcal{S}_{\Omega}$. It can be observed that there is also a trade-off across the choices of $M$. And in most cases, the larger $K$ results in better performance. More discussion is in Appendix \ref{Discussion}.  

\begin{figure}[tbp] 
\centering 
\begin{minipage}{.24\textwidth}
\begin{tikzpicture}[scale=0.48] 
\begin{axis}[
    xlabel= $\beta$, 
    ylabel= $M$, 
    zlabel= Classification Accuracies in $\%$,
    tick align=outside
    ]
\addplot3[surf,] coordinates {
 (0.000000,2.000000,81.585000)(0.000000,4.000000,81.585000)(0.000000,6.000000,81.585000)(0.000000,8.000000,81.585000)(0.000000,10.000000,81.585000)

(0.200000,2.000000,83.880000)(0.200000,4.000000,84.855000)(0.200000,6.000000,83.415000)(0.200000,8.000000,83.420000)(0.200000,10.000000,83.830000)

(0.400000,2.000000,83.270000)(0.400000,4.000000,83.635000)(0.400000,6.000000,82.250000)(0.400000,8.000000,82.735000)(0.400000,10.000000,82.565000)

(0.600000,2.000000,82.200000)(0.600000,4.000000,82.340000)(0.600000,6.000000,81.610000)(0.600000,8.000000,81.735000)(0.600000,10.000000,81.955000)

(0.800000,2.000000,82.395000)(0.800000,4.000000,82.125000)(0.800000,6.000000,81.440000)(0.800000,8.000000,80.830000)(0.800000,10.000000,80.535000)

(1.000000,2.000000,79.980000)(1.000000,4.000000,77.240000)(1.000000,6.000000,77.265000)(1.000000,8.000000,77.090000)(1.000000,10.000000,75.780000)

};
\end{axis}
\end{tikzpicture}
\end{minipage}%
\begin{minipage}{.24\textwidth}
\begin{tikzpicture}[scale=0.48] 
\begin{axis}[
    xlabel= $\beta$, 
    ylabel= $K$, 
    zlabel= Classification Accuracies in $\%$,
    tick align=outside,
    zmin= 70.0
    ]
    ]
\addplot3[surf,] coordinates {
(0.000000,0.000000,78.180000)(0.000000,2.000000,78.180000)(0.000000,4.000000,78.180000)(0.000000,6.000000,78.180000)(0.000000,8.000000,78.180000)(0.000000,10.000000,78.180000)

(0.200000,0.000000,78.770000)(0.200000,2.000000,80.930000)(0.200000,4.000000,83.830000)(0.200000,6.000000,84.550000)(0.200000,8.000000,85.640000)(0.200000,10.000000,83.810000)

(0.400000,0.000000,81.310000)(0.400000,2.000000,83.910000)(0.400000,4.000000,84.090000)(0.400000,6.000000,86.120000)(0.400000,8.000000,85.560000)(0.400000,10.000000,85.360000)

(0.600000,0.000000,82.660000)(0.600000,2.000000,82.510000)(0.600000,4.000000,85.230000)(0.600000,6.000000,85.410000)(0.600000,8.000000,85.790000)(0.600000,10.000000,85.670000)

(0.800000,0.000000,83.580000)(0.800000,2.000000,80.930000)(0.800000,4.000000,84.270000)(0.800000,6.000000,85.720000)(0.800000,8.000000,85.770000)(0.800000,10.000000,86.430000)

(1.000000,0.000000,54.490000)(1.000000,2.000000,68.660000)(1.000000,4.000000,76.730000)(1.000000,6.000000,77.750000)(1.000000,8.000000,78.030000)(1.000000,10.000000,85.820000)

};
\end{axis}
\end{tikzpicture}
\end{minipage}
\caption{Classification accuracies of different $\beta$, number of style provider $M$ and number of iterations $K$ in inner optimization. The left results are conducted on the PACS dataset with Art-painting (A) as the unknown target domain; the right results are conducted with Sketch (S) as the unknown target domain.}
\label{fig:4}
\end{figure}
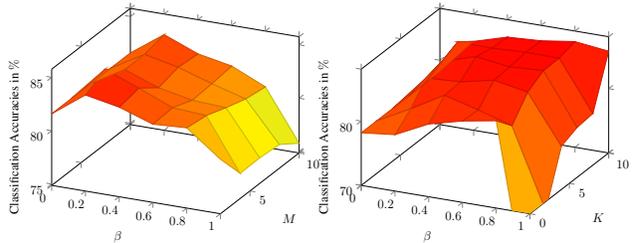

\begin{figure}[!tbp]
\centering
\label{Fig.sub.2}
\includegraphics[width=0.45\textwidth]{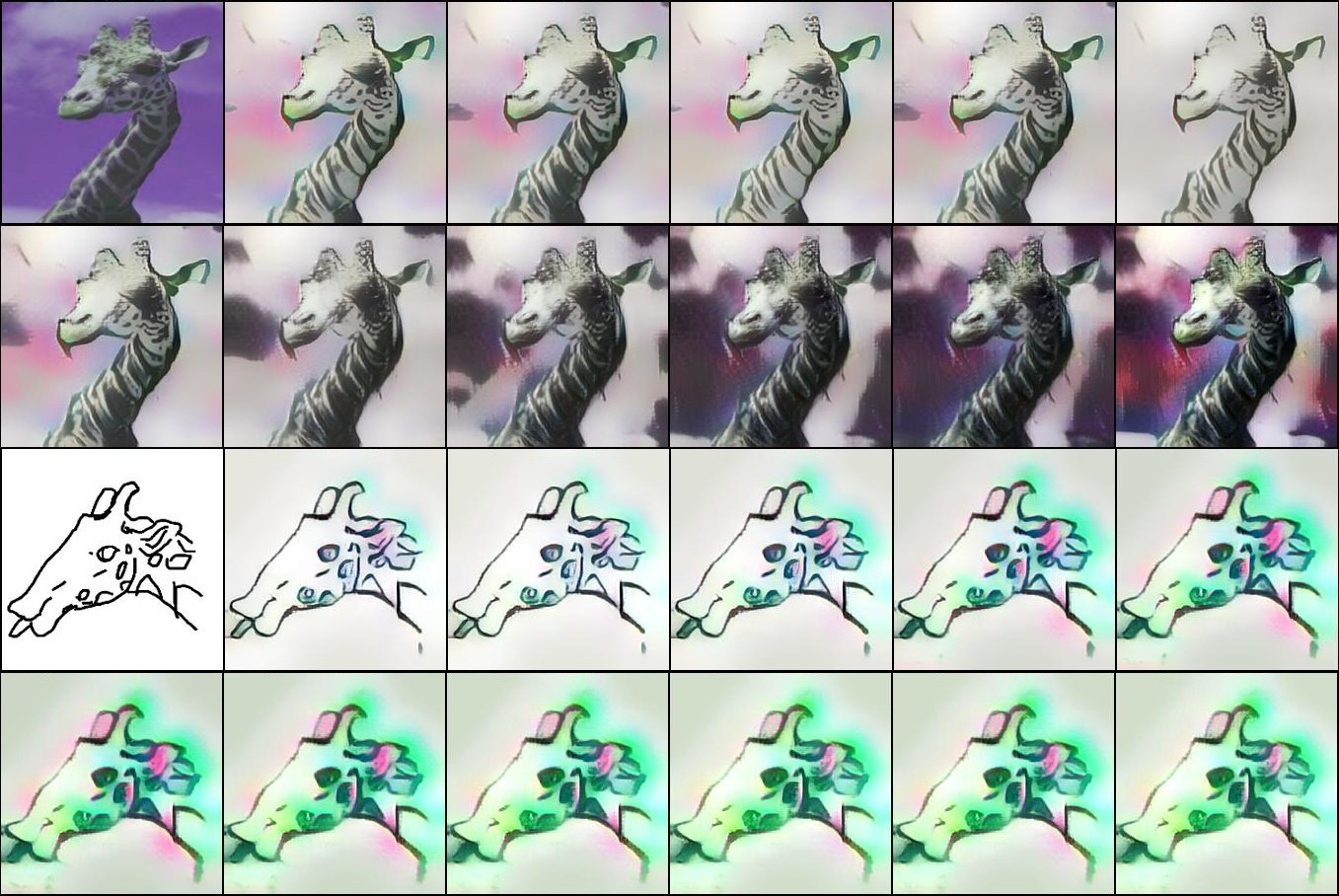}
\caption{Visualization results of the normal exploration process, where the semantic information is preserved when exploring non-semantic factors. More results are shown in Appendix \ref{vr}.}
\label{fig:ss}
\vspace{-3mm}
\end{figure}

\begin{figure}[!tbp]
\centering
\includegraphics[width=0.45\textwidth]{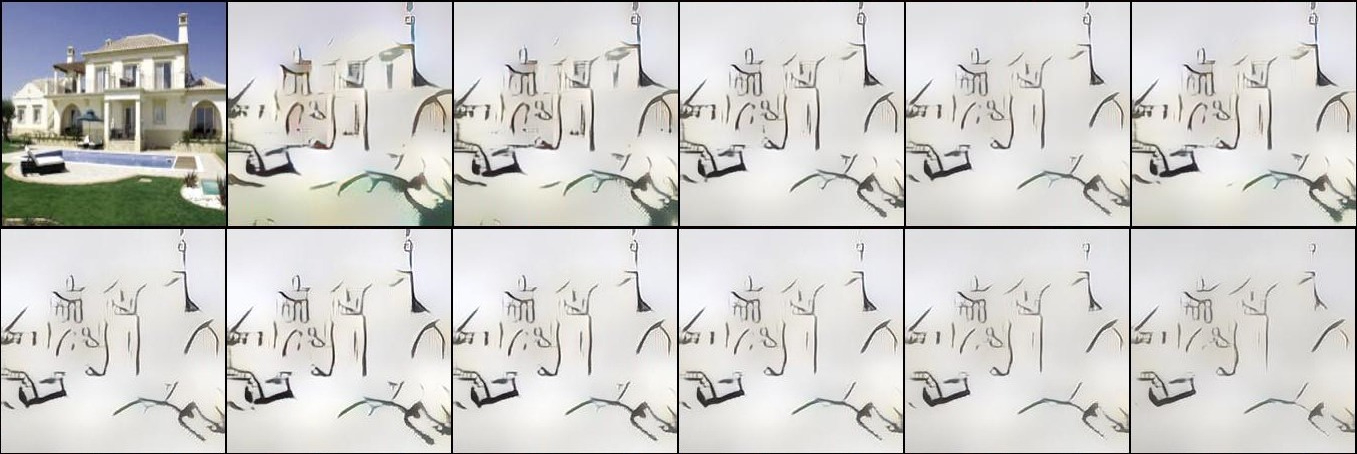}
\caption{Visualization results of the exploration process with too large $\Omega$. 
The semantic information is lost during the exploration.}
\label{Fig.bbb}
\vspace{-1mm}
\end{figure}

\textbf{Limitation.} In our work, different methods to construct the mechanism $\mathbf{G}$ have a great impact on the results. How to find a better way to construct $\mathbf{G}$ remains to be explored.

\section{Conclusion}
Generating new domains using these accessible training domains is one of the most effective approaches in domain generalization, yet their performance gain depends on the distribution discrepancy between the generated domain and the unknown target domain.
The low-confidence issue hinders the application of Distributionally robust optimization to DG.
To address this issue, we propose an approach called MODE, which performs distribution exploration in an uncertainty subset that shares the same semantic factors with the training domains, and theoretically shows the convergence guarantee toward the generalization performance on the unknown target domain. Empirically, we conduct extensive experiments to verify the effectiveness of our approach.
We hope that our work can inspire more ideas in the future.
\section{Acknowledgment}
This work was supported by NSFC No.62222117 and the Fundamental Research Funds for the Central Universities under contract WK3490000005.
YGZ and BH were supported by NSFC Young Scientists Fund No.62006202, Guangdong Basic and Applied Basic Research Foundation No.2022A1515011652, RGC Early Career Scheme No.22200720, and CAAI-Huawei MindSpore Open Fund.

\bibliography{example_paper}

\begin{thebibliography}{81}
\providecommand{\natexlab}[1]{#1}
\providecommand{\url}[1]{\texttt{#1}}
\expandafter\ifx\csname urlstyle\endcsname\relax
  \providecommand{\doi}[1]{doi: #1}\else
  \providecommand{\doi}{doi: \begingroup \urlstyle{rm}\Url}\fi

\bibitem[Ben-Tal et~al.(2013)Ben-Tal, Den~Hertog, De~Waegenaere, Melenberg, and
  Rennen]{ben2013robust}
Ben-Tal, A., Den~Hertog, D., De~Waegenaere, A., Melenberg, B., and Rennen, G.
\newblock Robust solutions of optimization problems affected by uncertain
  probabilities.
\newblock In \emph{Management Science}, volume~59, pp.\  341--357. INFORMS,
  2013.

\bibitem[Borlino et~al.(2020)Borlino, D'Innocente, and
  Tommasi]{borlino2021rethinking}
Borlino, F.~C., D'Innocente, A., and Tommasi, T.
\newblock Rethinking domain generalization baselines.
\newblock In \emph{ICPR}, 2020.

\bibitem[Carlucci et~al.(2019)Carlucci, D'Innocente, Bucci, Caputo, and
  Tommasi]{carlucci2019domain}
Carlucci, F.~M., D'Innocente, A., Bucci, S., Caputo, B., and Tommasi, T.
\newblock Domain generalization by solving jigsaw puzzles.
\newblock In \emph{CVPR}, 2019.

\bibitem[Chen et~al.(2022)Chen, Zhang, Bian, Yang, Kaili, Xie, Liu, Han, and
  Cheng]{chen2022learning}
Chen, Y., Zhang, Y., Bian, Y., Yang, H., Kaili, M., Xie, B., Liu, T., Han, B.,
  and Cheng, J.
\newblock Learning causally invariant representations for out-of-distribution
  generalization on graphs.
\newblock In \emph{NeurIPS}, 2022.

\bibitem[Chuang et~al.(2020)Chuang, Torralba, and
  Jegelka]{chuang2020estimating}
Chuang, C.-Y., Torralba, A., and Jegelka, S.
\newblock Estimating generalization under distribution shifts via
  domain-invariant representations.
\newblock In \emph{ICML}, 2020.

\bibitem[Csiszar(1967)]{csiszar1967information}
Csiszar, I.
\newblock Information-type measures of difference of probability distributions
  and indirect observations.
\newblock In \emph{Studia . Math. Hungar}, volume~2, 1967.

\bibitem[Dou et~al.(2019)Dou, Coelho~de Castro, Kamnitsas, and
  Glocker]{dou2019domain}
Dou, Q., Coelho~de Castro, D., Kamnitsas, K., and Glocker, B.
\newblock Domain generalization via model-agnostic learning of semantic
  features.
\newblock In \emph{NeurIPS}, 2019.

\bibitem[Fang et~al.(2020)Fang, Lu, Liu, Xuan, and Zhang]{fang2020open}
Fang, Z., Lu, J., Liu, F., Xuan, J., and Zhang, G.
\newblock Open set domain adaptation: Theoretical bound and algorithm.
\newblock In \emph{IEEE Transactions on Neural Networks and Learning Systems},
  volume~32, pp.\  4309--4322. IEEE, 2020.

\bibitem[Fang et~al.(2022)Fang, Li, Lu, Dong, Han, and Liu]{fang2022out}
Fang, Z., Li, Y., Lu, J., Dong, J., Han, B., and Liu, F.
\newblock Is out-of-distribution detection learnable?
\newblock In \emph{NeurIPS}, 2022.

\bibitem[Finn et~al.(2017)Finn, Abbeel, and Levine]{finn2017model}
Finn, C., Abbeel, P., and Levine, S.
\newblock Model-agnostic meta-learning for fast adaptation of deep networks.
\newblock In \emph{ICML}, 2017.

\bibitem[Frogner et~al.(2021)Frogner, Claici, Chien, and
  Solomon]{frogner2019incorporating}
Frogner, C., Claici, S., Chien, E., and Solomon, J.
\newblock Incorporating unlabeled data into distributionally robust learning.
\newblock In \emph{J. Mach. Learn. Res.}, volume~22, pp.\  56:1--56:46, 2021.

\bibitem[Ganin \& Lempitsky(2015)Ganin and Lempitsky]{ganin2015unsupervised}
Ganin, Y. and Lempitsky, V.
\newblock Unsupervised domain adaptation by backpropagation.
\newblock In \emph{ICML}, 2015.

\bibitem[Goodfellow et~al.(2015)Goodfellow, Shlens, and
  Szegedy]{goodfellow2014explaining}
Goodfellow, I.~J., Shlens, J., and Szegedy, C.
\newblock Explaining and harnessing adversarial examples.
\newblock In \emph{ICLR}, 2015.

\bibitem[Gulrajani \& Lopez-Paz(2021)Gulrajani and Lopez-Paz]{gulrajani2021in}
Gulrajani, I. and Lopez-Paz, D.
\newblock In search of lost domain generalization.
\newblock In \emph{ICLR}, 2021.

\bibitem[He et~al.(2016)He, Zhang, Ren, and Sun]{he2016deep}
He, K., Zhang, X., Ren, S., and Sun, J.
\newblock Deep residual learning for image recognition.
\newblock In \emph{CVPR}, 2016.

\bibitem[Hu et~al.(2018)Hu, Niu, Sato, and Sugiyama]{hu2018does}
Hu, W., Niu, G., Sato, I., and Sugiyama, M.
\newblock Does distributionally robust supervised learning give robust
  classifiers?
\newblock In \emph{ICML}, 2018.

\bibitem[Huang \& Belongie(2017)Huang and Belongie]{huang2017arbitrary}
Huang, X. and Belongie, S.
\newblock Arbitrary style transfer in real-time with adaptive instance
  normalization.
\newblock In \emph{ICCV}, 2017.

\bibitem[Huang et~al.(2020)Huang, Wang, Xing, and Huang]{huang2020self}
Huang, Z., Wang, H., Xing, E.~P., and Huang, D.
\newblock Self-challenging improves cross-domain generalization.
\newblock In \emph{ECCV}, 2020.

\bibitem[Krueger et~al.(2021)Krueger, Caballero, Jacobsen, Zhang, Binas, Zhang,
  Le~Priol, and Courville]{krueger2021out}
Krueger, D., Caballero, E., Jacobsen, J.-H., Zhang, A., Binas, J., Zhang, D.,
  Le~Priol, R., and Courville, A.
\newblock Out-of-distribution generalization via risk extrapolation (rex).
\newblock In \emph{ICML}, 2021.

\bibitem[LeCun et~al.(1998)LeCun, Bottou, Bengio, and
  Haffner]{lecun1998gradient}
LeCun, Y., Bottou, L., Bengio, Y., and Haffner, P.
\newblock Gradient-based learning applied to document recognition.
\newblock In \emph{Proceedings of the IEEE}, volume~86, pp.\  2278--2324. IEEE,
  1998.

\bibitem[Lei et~al.(2021)Lei, Hu, and Lee]{lei2021near}
Lei, Q., Hu, W., and Lee, J.
\newblock Near-optimal linear regression under distribution shift.
\newblock In \emph{ICML}, 2021.

\bibitem[Li et~al.(2017)Li, Yang, Song, and Hospedales]{li2017deeper}
Li, D., Yang, Y., Song, Y.-Z., and Hospedales, T.~M.
\newblock Deeper, broader and artier domain generalization.
\newblock In \emph{ICCV}, 2017.

\bibitem[Li et~al.(2018{\natexlab{a}})Li, Yang, Song, and
  Hospedales]{li2018learning}
Li, D., Yang, Y., Song, Y.-Z., and Hospedales, T.
\newblock Learning to generalize: Meta-learning for domain generalization.
\newblock In \emph{AAAI}, 2018{\natexlab{a}}.

\bibitem[Li et~al.(2019)Li, Zhang, Yang, Liu, Song, and
  Hospedales]{li2019episodic}
Li, D., Zhang, J., Yang, Y., Liu, C., Song, Y.-Z., and Hospedales, T.~M.
\newblock Episodic training for domain generalization.
\newblock In \emph{ICCV}, 2019.

\bibitem[Li et~al.(2018{\natexlab{b}})Li, Pan, Wang, and Kot]{li2018domain}
Li, H., Pan, S.~J., Wang, S., and Kot, A.~C.
\newblock Domain generalization with adversarial feature learning.
\newblock In \emph{CVPR}, 2018{\natexlab{b}}.

\bibitem[Li et~al.(2021)Li, Zhang, Wei, Lan, Zeng, Jin, and
  Chen]{li2021confounder}
Li, X., Zhang, Z., Wei, G., Lan, C., Zeng, W., Jin, X., and Chen, Z.
\newblock Confounder identification-free causal visual feature learning.
\newblock In \emph{arXiv preprint arXiv:2111.13420}, 2021.

\bibitem[Li et~al.(2022)Li, Dai, Ge, Liu, Shan, and Duan]{li2022uncertainty}
Li, X., Dai, Y., Ge, Y., Liu, J., Shan, Y., and Duan, L.-Y.
\newblock Uncertainty modeling for out-of-distribution generalization.
\newblock In \emph{ICLR}, 2022.

\bibitem[Liang et~al.(2021)Liang, Gong, Li, Liu, Li, Liu, Wang,
  et~al.]{liang2021pareto}
Liang, J., Gong, K., Li, S., Liu, C.~H., Li, H., Liu, D., Wang, G., et~al.
\newblock Pareto domain adaptation.
\newblock In \emph{NeurIPS}, 2021.

\bibitem[Liu et~al.(2021)Liu, Shen, Cui, Zhou, Kuang, Li, and
  Lin]{liu2021stable}
Liu, J., Shen, Z., Cui, P., Zhou, L., Kuang, K., Li, B., and Lin, Y.
\newblock Stable adversarial learning under distributional shifts.
\newblock In \emph{AAAI}, 2021.

\bibitem[Liu et~al.(2022{\natexlab{a}})Liu, Shen, Cui, Zhou, Kuang, and
  Li]{liu2022distributionally}
Liu, J., Shen, Z., Cui, P., Zhou, L., Kuang, K., and Li, B.
\newblock Distributionally robust learning with stable adversarial training.
\newblock In \emph{IEEE Transactions on Knowledge and Data Engineering},
  2022{\natexlab{a}}.

\bibitem[Liu et~al.(2022{\natexlab{b}})Liu, Wu, Li, and
  Cui]{liu2022distributionally2}
Liu, J., Wu, J., Li, B., and Cui, P.
\newblock Distributionally robust optimization with data geometry.
\newblock In \emph{NeurIPS}, 2022{\natexlab{b}}.

\bibitem[Lv et~al.(2022)Lv, Liang, Li, Zang, Liu, Wang, and
  Liu]{lv2022causality}
Lv, F., Liang, J., Li, S., Zang, B., Liu, C.~H., Wang, Z., and Liu, D.
\newblock Causality inspired representation learning for domain generalization.
\newblock In \emph{CVPR}, 2022.

\bibitem[Madry et~al.(2018)Madry, Makelov, Schmidt, Tsipras, and
  Vladu]{madry2017towards}
Madry, A., Makelov, A., Schmidt, L., Tsipras, D., and Vladu, A.
\newblock Towards deep learning models resistant to adversarial attacks.
\newblock In \emph{ICLR}, 2018.

\bibitem[Mahajan et~al.(2021)Mahajan, Tople, and Sharma]{mahajan2021domain}
Mahajan, D., Tople, S., and Sharma, A.
\newblock Domain generalization using causal matching.
\newblock In \emph{ICML}, 2021.

\bibitem[Matsuura \& Harada(2020)Matsuura and Harada]{matsuura2020domain}
Matsuura, T. and Harada, T.
\newblock Domain generalization using a mixture of multiple latent domains.
\newblock In \emph{AAAI}, 2020.

\bibitem[Mehra et~al.(2022)Mehra, Kailkhura, Chen, and
  Hamm]{mehra2022certifying}
Mehra, A., Kailkhura, B., Chen, P.-Y., and Hamm, J.
\newblock On certifying and improving generalization to unseen domains.
\newblock In \emph{arXiv preprint arXiv:2206.12364}, 2022.

\bibitem[Michel et~al.(2021)Michel, Hashimoto, and Neubig]{michel2021modeling}
Michel, P., Hashimoto, T., and Neubig, G.
\newblock Modeling the second player in distributionally robust optimization.
\newblock In \emph{ICLR}, 2021.

\bibitem[Mitrovic et~al.(2021)Mitrovic, McWilliams, Walker, Buesing, and
  Blundell]{mitrovic2020representation}
Mitrovic, J., McWilliams, B., Walker, J., Buesing, L., and Blundell, C.
\newblock Representation learning via invariant causal mechanisms.
\newblock In \emph{ICLR}, 2021.

\bibitem[Mouli \& Ribeiro(2021)Mouli and Ribeiro]{mouli2021asymmetry}
Mouli, S.~C. and Ribeiro, B.
\newblock Asymmetry learning for counterfactually-invariant classification in
  ood tasks.
\newblock In \emph{ICLR}, 2021.

\bibitem[Muandet et~al.(2013)Muandet, Balduzzi, and
  Sch{\"o}lkopf]{muandet2013domain}
Muandet, K., Balduzzi, D., and Sch{\"o}lkopf, B.
\newblock Domain generalization via invariant feature representation.
\newblock In \emph{ICML}, 2013.

\bibitem[Namkoong \& Duchi(2016)Namkoong and Duchi]{namkoong2016stochastic}
Namkoong, H. and Duchi, J.~C.
\newblock Stochastic gradient methods for distributionally robust optimization
  with f-divergences.
\newblock In \emph{NeurIPS}, 2016.

\bibitem[Netzer et~al.(2011)Netzer, Wang, Coates, Bissacco, Wu, and
  Ng]{netzer2011reading}
Netzer, Y., Wang, T., Coates, A., Bissacco, A., Wu, B., and Ng, A.~Y.
\newblock Reading digits in natural images with unsupervised feature learning.
\newblock 2011.

\bibitem[Nguyen et~al.(2021)Nguyen, Tran, Gal, and Baydin]{nguyen2021domain}
Nguyen, A.~T., Tran, T., Gal, Y., and Baydin, A.~G.
\newblock Domain invariant representation learning with domain density
  transformations.
\newblock In \emph{NeurIPS}, 2021.

\bibitem[Nguyen et~al.(2022)Nguyen, Do, Nguyen, Duong, and
  Nguyen]{nguyen2022front}
Nguyen, T., Do, K., Nguyen, D.~T., Duong, B., and Nguyen, T.
\newblock Front-door adjustment via style transfer for out-of-distribution
  generalisation.
\newblock In \emph{arXiv preprint arXiv:2212.03063}, 2022.

\bibitem[Oppenheim et~al.(1979)Oppenheim, Lim, Kopec, and
  Pohlig]{oppenheim1979phase}
Oppenheim, A., Lim, J., Kopec, G., and Pohlig, S.
\newblock Phase in speech and pictures.
\newblock In \emph{ICASSP}, 1979.

\bibitem[Oppenheim \& Lim(1981)Oppenheim and Lim]{oppenheim1981importance}
Oppenheim, A.~V. and Lim, J.~S.
\newblock The importance of phase in signals.
\newblock In \emph{Proceedings of the IEEE}, volume~69, pp.\  529--541. IEEE,
  1981.

\bibitem[Peng et~al.(2019)Peng, Bai, Xia, Huang, Saenko, and
  Wang]{peng2019moment}
Peng, X., Bai, Q., Xia, X., Huang, Z., Saenko, K., and Wang, B.
\newblock Moment matching for multi-source domain adaptation.
\newblock In \emph{ICCV}, 2019.

\bibitem[Qiao \& Peng(2023)Qiao and Peng]{qiao2023topology}
Qiao, F. and Peng, X.
\newblock Topology-aware robust optimization for out-of-distribution
  generalization.
\newblock In \emph{ICLR}, 2023.

\bibitem[Sagawa et~al.(2020)Sagawa, Koh, Hashimoto, and
  Liang]{sagawa2019distributionally}
Sagawa, S., Koh, P.~W., Hashimoto, T.~B., and Liang, P.
\newblock Distributionally robust neural networks.
\newblock In \emph{ICLR}, 2020.

\bibitem[Shafahi et~al.(2019)Shafahi, Najibi, Ghiasi, Xu, Dickerson, Studer,
  Davis, Taylor, and Goldstein]{shafahi2019adversarial}
Shafahi, A., Najibi, M., Ghiasi, M.~A., Xu, Z., Dickerson, J., Studer, C.,
  Davis, L.~S., Taylor, G., and Goldstein, T.
\newblock Adversarial training for free!
\newblock In \emph{NeurIPS}, 2019.

\bibitem[Shankar et~al.(2018)Shankar, Piratla, Chakrabarti, Chaudhuri, Jyothi,
  and Sarawagi]{shankar2018generalizing}
Shankar, S., Piratla, V., Chakrabarti, S., Chaudhuri, S., Jyothi, P., and
  Sarawagi, S.
\newblock Generalizing across domains via cross-gradient training.
\newblock In \emph{ICLR}, 2018.

\bibitem[Shen et~al.(2021)Shen, Liu, He, Zhang, Xu, Yu, and
  Cui]{shen2021towards}
Shen, Z., Liu, J., He, Y., Zhang, X., Xu, R., Yu, H., and Cui, P.
\newblock Towards out-of-distribution generalization: A survey.
\newblock In \emph{arXiv preprint arXiv:2108.13624}, 2021.

\bibitem[Shi et~al.(2022)Shi, Seely, Torr, N, Hannun, Usunier, and
  Synnaeve]{shi2022gradient}
Shi, Y., Seely, J., Torr, P., N, S., Hannun, A., Usunier, N., and Synnaeve, G.
\newblock Gradient matching for domain generalization.
\newblock In \emph{ICLR}, 2022.

\bibitem[Sinha et~al.(2018)Sinha, Namkoong, Volpi, and
  Duchi]{sinha2017certifying}
Sinha, A., Namkoong, H., Volpi, R., and Duchi, J.
\newblock Certifying some distributional robustness with principled adversarial
  training.
\newblock In \emph{ICLR}, 2018.

\bibitem[Somavarapu et~al.(2020)Somavarapu, Ma, and
  Kira]{somavarapu2020frustratingly}
Somavarapu, N., Ma, C.-Y., and Kira, Z.
\newblock Frustratingly simple domain generalization via image stylization.
\newblock In \emph{arXiv preprint arXiv:2006.11207}, 2020.

\bibitem[Staib \& Jegelka(2019)Staib and Jegelka]{staib2019distributionally}
Staib, M. and Jegelka, S.
\newblock Distributionally robust optimization and generalization in kernel
  methods.
\newblock In \emph{NeurIPS}, 2019.

\bibitem[Su et~al.(2019)Su, Vargas, and Sakurai]{su2019one}
Su, J., Vargas, D.~V., and Sakurai, K.
\newblock One pixel attack for fooling deep neural networks.
\newblock In \emph{IEEE Transactions on Evolutionary Computation}, volume~23,
  pp.\  828--841. IEEE, 2019.

\bibitem[Suter et~al.(2019)Suter, Miladinovic, Sch{\"o}lkopf, and
  Bauer]{suter2019robustly}
Suter, R., Miladinovic, D., Sch{\"o}lkopf, B., and Bauer, S.
\newblock Robustly disentangled causal mechanisms: Validating deep
  representations for interventional robustness.
\newblock In \emph{ICML}, 2019.

\bibitem[Szegedy et~al.(2014)Szegedy, Zaremba, Sutskever, Bruna, Erhan,
  Goodfellow, and Fergus]{szegedy2013intriguing}
Szegedy, C., Zaremba, W., Sutskever, I., Bruna, J., Erhan, D., Goodfellow, I.,
  and Fergus, R.
\newblock Intriguing properties of neural networks.
\newblock In \emph{ICLR}, 2014.

\bibitem[Tang et~al.(2021)Tang, Gao, Zhu, Zhang, Li, and
  Metaxas]{tang2021crossnorm}
Tang, Z., Gao, Y., Zhu, Y., Zhang, Z., Li, M., and Metaxas, D.~N.
\newblock Crossnorm and selfnorm for generalization under distribution shifts.
\newblock In \emph{ICCV}, 2021.

\bibitem[Torralba \& Efros(2011)Torralba and Efros]{torralba2011unbiased}
Torralba, A. and Efros, A.~A.
\newblock Unbiased look at dataset bias.
\newblock In \emph{CVPR}, 2011.

\bibitem[Veitch et~al.(2021)Veitch, D'Amour, Yadlowsky, and
  Eisenstein]{veitch2021counterfactual}
Veitch, V., D'Amour, A., Yadlowsky, S., and Eisenstein, J.
\newblock Counterfactual invariance to spurious correlations in text
  classification.
\newblock In \emph{NeurIPS}, 2021.

\bibitem[Venkateswara et~al.(2017)Venkateswara, Eusebio, Chakraborty, and
  Panchanathan]{venkateswara2017deep}
Venkateswara, H., Eusebio, J., Chakraborty, S., and Panchanathan, S.
\newblock Deep hashing network for unsupervised domain adaptation.
\newblock In \emph{CVPR}, 2017.

\bibitem[Vershynin(2018)]{vershynin2018high}
Vershynin, R.
\newblock \emph{High-dimensional probability: An introduction with applications
  in data science}, volume~47.
\newblock Cambridge University Press, 2018.

\bibitem[Villani(2009)]{villani2009optimal}
Villani, C.
\newblock \emph{Optimal transport: old and new}, volume 338.
\newblock Springer, 2009.

\bibitem[Villani(2021)]{villani2021topics}
Villani, C.
\newblock \emph{Topics in optimal transportation}, volume~58.
\newblock American Mathematical Soc., 2021.

\bibitem[Wang et~al.(2022)Wang, Liu, Zhang, Zhang, Gong, Liu, and
  Han]{wang2022watermarking}
Wang, Q., Liu, F., Zhang, Y., Zhang, J., Gong, C., Liu, T., and Han, B.
\newblock Watermarking for out-of-distribution detection.
\newblock In \emph{NeurIPS}, 2022.

\bibitem[Wang et~al.(2021)Wang, Luo, Qiu, Huang, and
  Baktashmotlagh]{wang2021learning}
Wang, Z., Luo, Y., Qiu, R., Huang, Z., and Baktashmotlagh, M.
\newblock Learning to diversify for single domain generalization.
\newblock In \emph{ICCV}, 2021.

\bibitem[Xiao et~al.(2021)Xiao, Shen, Zhen, Shao, and Snoek]{xiao2021bit}
Xiao, Z., Shen, J., Zhen, X., Shao, L., and Snoek, C.
\newblock A bit more bayesian: Domain-invariant learning with uncertainty.
\newblock In \emph{ICML}, 2021.

\bibitem[Xu et~al.(2021)Xu, Zhang, Zhang, Wang, and Tian]{xu2021fourier}
Xu, Q., Zhang, R., Zhang, Y., Wang, Y., and Tian, Q.
\newblock A fourier-based framework for domain generalization.
\newblock In \emph{CVPR}, 2021.

\bibitem[Ye et~al.(2021)Ye, Xie, Cai, Li, Li, and Wang]{ye2021towards}
Ye, H., Xie, C., Cai, T., Li, R., Li, Z., and Wang, L.
\newblock Towards a theoretical framework of out-of-distribution
  generalization.
\newblock In \emph{NeurIPS}, 2021.

\bibitem[Zhang et~al.(2020{\natexlab{a}})Zhang, Zhang, and Li]{zhang2020causal}
Zhang, C., Zhang, K., and Li, Y.
\newblock A causal view on robustness of neural networks.
\newblock In \emph{NeurIPS}, 2020{\natexlab{a}}.

\bibitem[Zhang et~al.(2019)Zhang, Zhang, Lu, Zhu, and Dong]{zhang2019you}
Zhang, D., Zhang, T., Lu, Y., Zhu, Z., and Dong, B.
\newblock You only propagate once: Accelerating adversarial training via
  maximal principle.
\newblock In \emph{NeurIPS}, 2019.

\bibitem[Zhang et~al.(2021)Zhang, Ahuja, Xu, Wang, and Courville]{zhang2021can}
Zhang, D., Ahuja, K., Xu, Y., Wang, Y., and Courville, A.
\newblock Can subnetwork structure be the key to out-of-distribution
  generalization?
\newblock In \emph{ICML}, 2021.

\bibitem[Zhang et~al.(2022{\natexlab{a}})Zhang, Zhang, Liu, Weller,
  Sch{\"o}lkopf, and Xing]{zhang2022towards}
Zhang, H., Zhang, Y.-F., Liu, W., Weller, A., Sch{\"o}lkopf, B., and Xing,
  E.~P.
\newblock Towards principled disentanglement for domain generalization.
\newblock In \emph{CVPR}, 2022{\natexlab{a}}.

\bibitem[Zhang et~al.(2020{\natexlab{b}})Zhang, Tian, Li, Wang, and
  Tao]{zhang2020principal}
Zhang, Y., Tian, X., Li, Y., Wang, X., and Tao, D.
\newblock Principal component adversarial example.
\newblock In \emph{IEEE Transactions on Image Processing}, volume~29, pp.\
  4804--4815. IEEE, 2020{\natexlab{b}}.

\bibitem[Zhang et~al.(2022{\natexlab{b}})Zhang, Gong, Liu, Niu, Tian, Han,
  Sch{\"o}lkopf, and Zhang]{zhang2021causaladv}
Zhang, Y., Gong, M., Liu, T., Niu, G., Tian, X., Han, B., Sch{\"o}lkopf, B.,
  and Zhang, K.
\newblock Causaladv: Adversarial robustness through the lens of causality.
\newblock In \emph{ICLR}, 2022{\natexlab{b}}.

\bibitem[Zhou et~al.(2020{\natexlab{a}})Zhou, Yang, Hospedales, and
  Xiang]{zhou2020deep}
Zhou, K., Yang, Y., Hospedales, T., and Xiang, T.
\newblock Deep domain-adversarial image generation for domain generalisation.
\newblock In \emph{AAAI}, 2020{\natexlab{a}}.

\bibitem[Zhou et~al.(2020{\natexlab{b}})Zhou, Yang, Hospedales, and
  Xiang]{zhou2020learning}
Zhou, K., Yang, Y., Hospedales, T., and Xiang, T.
\newblock Learning to generate novel domains for domain generalization.
\newblock In \emph{ECCV}, 2020{\natexlab{b}}.

\bibitem[Zhou et~al.(2021{\natexlab{a}})Zhou, Yang, Qiao, and
  Xiang]{zhou2021domain}
Zhou, K., Yang, Y., Qiao, Y., and Xiang, T.
\newblock Domain generalization with mixstyle.
\newblock In \emph{ICLR}, 2021{\natexlab{a}}.

\bibitem[Zhou et~al.(2021{\natexlab{b}})Zhou, Yang, Qiao, and
  Xiang]{zhou2021mixstyle}
Zhou, K., Yang, Y., Qiao, Y., and Xiang, T.
\newblock Mixstyle neural networks for domain generalization and adaptation.
\newblock In \emph{arXiv preprint arXiv:2107.02053}, 2021{\natexlab{b}}.

\end{thebibliography}
\bibliographystyle{icml2023}

\clearpage
\appendix
\onecolumn
\renewcommand{\appendixname}{Appendix~\Alph{section}}
\section{Proofs}
\subsection{Covering Number}\label{cn}
We use the covering number for the model classes in our derivation. Here we give the formal definition.
\begin{defin}
($\epsilon$-covering \cite{vershynin2018high}). Let  $(V,\|\cdot\|)$  be a normed space,  $\Theta \in V$ , and  $B(\cdot, \epsilon)$  the ball of radius  $\epsilon$ . Then  $\left\{V_{1}, \ldots, V_{N}\right\}$  is an $\epsilon$-covering of $ \Theta$  if  $\Theta \subset \bigcup_{i=1}^{N} B\left(V_{i}, \epsilon\right)$ , or equivalently,  $\forall \theta \in \Theta, \exists i$ 
 such that $\left\|\theta-V_{i}\right\| \leq \epsilon$ .

Upon our definition of  $\epsilon$ -covering, covering number is the minimal number of  $\epsilon$-balls one needs to cover  $\Theta$ .
\end{defin}

\begin{defin}
(Covering Number \cite{vershynin2018high}). $ \mathcal{N}(\Theta,\|\cdot\|, \epsilon)=\min \{n: \exists \epsilon$-covering over  $\Theta$  of size  $n\}$ .
\end{defin}
\subsection{Proof of Theorem \ref{thm1}}
\begin{proof}
We first recall the notations as follows:
\begin{align}
R^{\beta}_{\boldsymbol{\lambda}}(\mathbf{w}) =  (1-\beta ) \sum_{l=1}^N \lambda_l R_{l}(\mathbf{w})+\beta R_{\mathcal{S}_{\Omega}}(\mathbf{w}),
\end{align}
\begin{align}
\widehat{R}^{\beta}_{\boldsymbol{\lambda}}(\mathbf{w}) = (1-\beta ) \sum_{l=1}^N \lambda_l \widehat{R}_{l}(\mathbf{w})+\beta \widehat{R}_{\mathcal{S}_{\Omega}}(\mathbf{w}),
\end{align}
\begin{align}
R_{l}(\mathbf{w})= \mathbb{E} \ \ell\left(\mathbf{f}_{\mathbf{w}}\circ \mathbf{G} ; S_l,C, Y\right),
\end{align}
\begin{align}
\widehat{R}_{l}(\mathbf{w})=\frac{1}{n_{l}} \sum_{i=1}^{n_{l}} \ell\left(\mathbf{f}_{\mathbf{w}} ; \mathbf{x}_{l}^{i}, y_{l}^{i}\right),
\end{align}
\begin{align}
R_{\mathcal{S}_{\Omega}}(\mathbf{w}) = \mathbb{E} \max_{\mathbf{s}\in \mathcal{S}_{\Omega}} \ell\left(\mathbf{f}_{\mathbf{w}}\circ \mathbf{G} ; \mathbf{s},C, Y\right),
\end{align}
\begin{align}
\widehat{R}_{\mathcal{S}_{\Omega}}(\mathbf{w}) & =  \frac{1}{\sum_{l=1}^{N} n_{l}} \sum_{l=1}^{N} \sum_{i=1}^{n_{l}} \max_{\mathbf{s}\in \mathcal{S}_{\Omega} } \ell\left(\mathbf{f}_{\mathbf{w}}\circ \mathbf{G} ;\mathbf{s},\mathbf{c}_{l}^{i}, y_{l}^{i}\right).
\end{align}

Let  $\mathbf{w}^{*}$  be the solution of  $\min _{\mathbf{w} \in \mathcal{W}} R^{\beta }_{\boldsymbol{\lambda}}(\mathbf{w} )$ . Then similar to \citet{fang2020open,fang2022out}, we have

\begin{align}\label{thm1.1}
\begin{split}
 \space R^{\beta }_{\boldsymbol{\lambda}}(\widehat{\mathbf{w}} )-R^{\beta }_{\boldsymbol{\lambda}}\left(\mathbf{w}^{*} \right) 
 & \leq R^{\beta }_{\boldsymbol{\lambda}}(\widehat{\mathbf{w}} )-\widehat{R}^{\beta}_{\boldsymbol{\lambda}}(\widehat{\mathbf{w}} )+\widehat{R}^{\beta}_{\boldsymbol{\lambda}}(\widehat{\mathbf{w}} ) - R^{\beta}_{\boldsymbol{\lambda}}\left(\mathbf{w}^{*} \right)+\widehat{R}^{\beta}_{\boldsymbol{\lambda}}\left(\mathbf{w}^{*} \right)-\widehat{R}^{\beta}_{\boldsymbol{\lambda}}\left(\mathbf{w}^{*} \right) \\
& \leq (1-\beta)\left[\sum_{l=1}^N \lambda_l R_{l}(\widehat{\mathbf{w}})-\sum_{l=1}^N \lambda_l R_{l}\left(\mathbf{w}^{*}\right)\right]+\beta\left[R_{\mathcal{S}_{\Omega}}(\widehat{\mathbf{w}} )-R_{\mathcal{S}_{\Omega}}\left(\mathbf{w}^{*} \right)\right] \\
& -(1-\beta)\left[\sum_{l=1}^N \lambda_l \widehat{R}_{l}(\widehat{\mathbf{w}})-\sum_{l=1}^N \lambda_l \widehat{R}_{l}\left(\mathbf{w}^{*}\right)\right]-\beta\left[\widehat{R}_{\mathcal{S}_{\Omega}}(\widehat{\mathbf{w}} )-\widehat{R}_{\mathcal{S}_{\Omega}}\left(\mathbf{w}^{*} \right)\right] \\
& =(1-\beta)\sum_{l=1}^N \lambda_l \left[ R_{l}(\widehat{\mathbf{w}})- \widehat{R}_{l}(\widehat{\mathbf{w}})\right]+\beta\left[R_{\mathcal{S}_{\Omega}}(\widehat{\mathbf{w}} )-\widehat{R}_{\mathcal{S}_{\Omega}}(\widehat{\mathbf{w}} )\right] \\
& -(1-\beta)\sum_{l=1}^N \lambda_l\left[ R_{l}\left(\mathbf{w}^{*}\right)- \widehat{R}_{l}\left(\mathbf{w}^{*}\right)\right]-\beta\left[R_{\mathcal{S}_{\Omega}}\left(\mathbf{w}^{*} \right)-\widehat{R}_{\mathcal{S}_{\Omega}}\left(\mathbf{w}^{*} \right)\right],
\end{split}
\end{align}
where $\widehat{R}^{\beta}_{\boldsymbol{\lambda}}(\widehat{\mathbf{w}}) - \widehat{R}^{\beta}_{\boldsymbol{\lambda}}\left(\mathbf{w}^{*}\right) \leq 0 $.

By Lemma \ref{lemma1} and Lemma \ref{lemma4}, we have that with the probability at least  $1-2 e^{-t}>0$ , 

\begin{align}\label{thm1.2}
\begin{split}
& (1-\beta)\sum_{l=1}^N \lambda_l \left[ R_{l}(\widehat{\mathbf{w}})- \widehat{R}_{l}(\widehat{\mathbf{w}})\right]+\beta\left[R_{\mathcal{S}_{\Omega}}(\widehat{\mathbf{w}} )-\widehat{R}_{\mathcal{S}_{\Omega}}(\widehat{\mathbf{w}} )\right] \\
\leq &(1-\beta) \sum_{l=1}^N\frac{b_{0} M_{\ell}\lambda_l }{\sqrt{n_l}} \int_{0}^{1} \sqrt{\log \mathcal{N}\left(\mathcal{F}, M_{\ell} \epsilon, L^{\infty}\right)} d \epsilon \\
+&\beta \frac{b_{1} M_{\ell}}{\sqrt{\sum_{l=1}^N n_l}} \int_{0}^{1} \sqrt{\log \mathcal{N}\left(\mathcal{F}, M_{\ell} \epsilon, L^{\infty}\right)} \mathrm{d} \epsilon \\
+&(1-\beta)\sum_{l=1}^N \lambda_{l} M_{\ell} \sqrt{\frac{2 t}{n_l}}+\beta M_{\ell} \sqrt{\frac{2 t}{\sum_{l=1}^N n_{l}}},
\end{split}
\end{align}
here  $b_{0}$ and $b_{1}$  are uniform constants.

By Lemma \ref{lemma2} and Lemma \ref{lemma5}, we have that with the probability at least  $1-2 e^{-t}>0$ , 

\begin{align}\label{thm1.3}
\begin{split}
& (1-\beta) \sum_{l=1}^N \lambda_l\left[ R_{l}\left(\mathbf{w}^{*}\right)- \widehat{R}_{l}\left(\mathbf{w}^{*}\right)\right]+\beta\left[R_{\mathcal{S}_{\Omega}}\left(\mathbf{w}^{*} \right)-\widehat{R}_{\mathcal{S}_{\Omega}}\left(\mathbf{w}^{*} \right)\right] \\  \leq & (1-\beta)\sum_{l=1}^N \lambda_{l} M_{\ell} \sqrt{\frac{2 t}{n_l}}+\beta M_{\ell} \sqrt{\frac{2 t}{\sum_{l=1}^N n_{l}}}.
\end{split}
\end{align}

Combining Eqs. \eqref{thm1.1}, \eqref{thm1.2} and \eqref{thm1.3}, we have that with the probability at least $1-4 e^{-t}>0$ , 
\begin{align}
\begin{split}
 \space  R^{\beta }_{\boldsymbol{\lambda}}(\widehat{\mathbf{w}})-\min _{\mathbf{w} \in \mathcal{W}} R^{\beta }_{\boldsymbol{\lambda}}(\mathbf{w} ) \leq   \epsilon_{\boldsymbol{\lambda}}^{\beta}(n_1,...,n_N; t)
\end{split}
\end{align}
where
\begin{align}
\begin{split}
\epsilon_{\boldsymbol{\lambda}}^{\beta}(n_1,...,n_N; t)=&(1-\beta) \sum_{l=1}^N\frac{b_{0} M_{\ell}\lambda_l }{\sqrt{n_l}} \int_{0}^{1} \sqrt{\log \mathcal{N}\left(\mathcal{F}, M_{\ell} \epsilon, L^{\infty}\right)} d \epsilon \\
+&\beta \frac{b_{1} M_{\ell}}{\sqrt{\sum_{l=1}^N n_l}} \int_{0}^{1} \sqrt{\log \mathcal{N}\left(\mathcal{F}, M_{\ell} \epsilon, L^{\infty}\right)} \mathrm{d} \epsilon \\
+&2(1-\beta)\sum_{l=1}^N \lambda_{l} M_{\ell} \sqrt{\frac{2 t}{n_l}}+2\beta M_{\ell} \sqrt{\frac{2 t}{\sum_{l=1}^N n_{l}}},
\end{split}
\end{align}
here  $b_{0}$ and $b_{1}$  are uniform constants.
\end{proof}
\subsection{Corollary \ref{thm2}}\label{coro1}

\begin{coro}\label{thm2}
Given the same conditions in Theorem \ref{thm1}, if

$\bullet$ $\ell(\cdot ; \mathbf{x}, y)$  is $L$-Lipschitz w.r.t. norm $\| \cdot \|$, i.e., for any  $(\mathbf{x}, y) \in \mathcal{X} \times\mathcal{Y}$ , and  $\mathbf{w}, \mathbf{w}^{\prime} \in \mathcal{W} $,
\begin{align}
\begin{split}
\left|\ell (\mathbf{f_{w}} ;\mathbf{x}, y)-\ell\left(\mathbf{f_{w^{\prime}}} ; \mathbf{x}, y\right)\right| \leq L\left\|\mathbf{w}-\mathbf{w}^{\prime}\right\|,
\end{split}
\end{align}

$\bullet$ the parameter space  $\mathcal{W} \subset \mathbb{R}^{d^{\prime}}$  satisfies that
\begin{align}
\operatorname{diam}(\mathcal{W})=\sup _{\mathbf{w}, \mathbf{w}^{\prime} \in \mathcal{W}}\left\|\mathbf{w}-\mathbf{w}^{\prime}\right\|<+\infty ,
\end{align}

With the probability at least  $1-4 e^{-t}>0$ ,
\begin{align}
\begin{split}
R^{\beta }_{\boldsymbol{\lambda}}(\widehat{\mathbf{w}} )-\min _{\mathbf{w} \in \mathcal{W}} R^{\beta }_{\boldsymbol{\lambda}}(\mathbf{w} ) \leq \tilde{\epsilon}_{\boldsymbol{\lambda}}^{\beta}(n_1,...,n_N; t)
\end{split}
\end{align}
where
\begin{align}
\begin{split}
\tilde{\epsilon}_{\boldsymbol{\lambda}}^{\beta}(n_1,...,n_N; t) & =(1-\beta) \sum_{l=1}^N b_{0}\lambda_l \sqrt{\frac{M_{\ell} \operatorname{diam}(\mathcal{W}) L d^{\prime}}{n_l}} \\
& +\beta b_{1} \sqrt{\frac{M_{\ell}\operatorname{diam}(\mathcal{W}) L d^{\prime}}{\sum_{l=1}^N n_l}} \\
& +2(1-\beta)\sum_{l=1}^N \lambda_{l} M_{\ell} \sqrt{\frac{2 t}{n_l}}+2\beta M_{\ell} \sqrt{\frac{2 t}{\sum_{l=1}^N n_{l}}},
\end{split}
\end{align}

here  $b_{0}$ and $b_{1}$  are uniform constants.
\end{coro}
\begin{proof}
Similar to the proof of Theorem \ref{thm1}.

By Lemma \ref{lemma3} and Lemma \ref{lemma6}, we have that with the probability at least  $1-2 e^{-t}>0$,

\begin{align}\label{thm2.2}
\begin{split}
& (1-\beta)\sum_{l=1}^N \lambda_l \left[ R_{l}(\widehat{\mathbf{w}})- \widehat{R}_{l}(\widehat{\mathbf{w}})\right]+\beta\left[R_{\mathcal{S}_{\Omega}}(\widehat{\mathbf{w}})-\widehat{R}_{\mathcal{S}_{\Omega}}(\widehat{\mathbf{w}})\right] \\
\leq & (1-\beta) \sum_{l=1}^N b_{0}\lambda_l \sqrt{\frac{M_{\ell} \operatorname{diam}(\mathcal{W}) L d^{\prime}}{n_l}} \\
& +\beta b_{1} \sqrt{\frac{M_{\ell}\operatorname{diam}(\mathcal{W}) L d^{\prime}}{\sum_{l=1}^N n_l}} \\
& +(1-\beta)\sum_{l=1}^N \lambda_{l} M_{\ell} \sqrt{\frac{2 t}{n_l}}+\beta M_{\ell} \sqrt{\frac{2 t}{\sum_{l=1}^N n_{l}}}.
\end{split}
\end{align}

Combining Eqs. \eqref{thm1.1}, \eqref{thm2.2} and \eqref{thm1.3}, we have that with the probability at least  $1-4 e^{-t}>0$ ,
\begin{align}
\begin{split}
R^{\beta }_{\boldsymbol{\lambda}}(\widehat{\mathbf{w}} )-\min _{\mathbf{w} \in \mathcal{W}} R^{\beta }_{\boldsymbol{\lambda}}(\mathbf{w} ) \leq \tilde{\epsilon}_{\boldsymbol{\lambda}}^{\beta}(n_1,...,n_N; t)
\end{split}
\end{align}
where
\begin{align}
\begin{split}
\tilde{\epsilon}_{\boldsymbol{\lambda}}^{\beta}(n_1,...,n_N; t) & =(1-\beta) \sum_{l=1}^N b_{0}\lambda_l \sqrt{\frac{M_{\ell} \operatorname{diam}(\mathcal{W}) L d^{\prime}}{n_l}} \\
& +\beta b_{1} \sqrt{\frac{M_{\ell}\operatorname{diam}(\mathcal{W}) L d^{\prime}}{\sum_{l=1}^N n_l}} \\
& +2(1-\beta)\sum_{l=1}^N \lambda_{l} M_{\ell} \sqrt{\frac{2 t}{n_l}}+2\beta M_{\ell} \sqrt{\frac{2 t}{\sum_{l=1}^N n_{l}}},
\end{split}
\end{align}

here  $b_{0}$ and $b_{1}$  are uniform constants.
\end{proof}

\subsection{Proof of Theorem \ref{thm4}}
\begin{proof}
We first recall the notations as follows:
\begin{align}
R_{t}(\mathbf{w})=\mathbb{E} \ell\left(\mathbf{f}_{\mathbf{w}} ; X_t, Y_t \right) = \mathbb{E} \ \ell\left(\mathbf{f}_{\mathbf{w}}\circ \mathbf{G} ; S_t,C, Y\right),
\end{align}
\begin{align}
R_{l}(\mathbf{w})= \mathbb{E} \ \ell\left(\mathbf{f}_{\mathbf{w}}\circ \mathbf{G} ; S_l,C, Y\right),
\end{align}
\begin{align}
R_{\mathcal{S}_{\Omega}}(\mathbf{w}) = \mathbb{E} \max_{\mathbf{s}\in \mathcal{S}_{\Omega}} \ell\left(\mathbf{f}_{\mathbf{w}}\circ \mathbf{G} ; \mathbf{s},C, Y\right),
\end{align}
\begin{align}
R^{\beta}_{\boldsymbol{\lambda}}(\mathbf{w}) =  (1-\beta ) \sum_{l=1}^N \lambda_l R_{l}(\mathbf{w})+\beta R_{\mathcal{S}_{\Omega}}(\mathbf{w}).
\end{align}

Consider 
\begin{align}\label{eq::re1}
R_{t}(\widehat{\mathbf{w}})-R^{\beta }_{\boldsymbol{\lambda}}(\widehat{\mathbf{w}} ) =(1-\beta)\sum_{l=1}^N \lambda_l (R_{t}(\widehat{\mathbf{w}})- R_{l}(\widehat{\mathbf{w}}))+ \beta(R_{t}(\widehat{\mathbf{w}})-R_{\mathcal{S}_{\Omega}}(\widehat{\mathbf{w}})).
\end{align}
We have
\begin{align}\label{eq::re2}
R_{t}(\widehat{\mathbf{w}})- R_{l}(\widehat{\mathbf{w}}) = \mathbb{E} \ \ell\left(\mathbf{f}_{\widehat{\mathbf{w}}}\circ \mathbf{G} ; S_t,C, Y\right) - \mathbb{E} \ \ell\left(\mathbf{f}_{\widehat{\mathbf{w}}}\circ \mathbf{G} ; S_l,C, Y\right) \leq L_{c}\mathrm{W}_{c}\left(D_{S_{\mathrm{t}}}, D_{S_{l}}\right).
\end{align}
We set the largest risk in $\Omega$ w.r.t. $\mathbf{f}_{\mathbf{w}}$ is
\begin{align}
R_{\Omega}(\mathbf{w})=\max_{S_{\boldsymbol{\alpha}} \in \Omega}\mathbb{E} \ell\left(\mathbf{f}_{\mathbf{w}}\circ \mathbf{G} ; S_{\boldsymbol{\alpha}},C, Y\right),
\end{align}
where $\Omega$ stands for the uncertainty set  introduced in Eq. \eqref{omega}.

We have
\begin{align}
R_{\Omega}(\mathbf{w}) \leq R_{\mathcal{S}_{\Omega}}(\mathbf{w}),
\end{align}
then
\begin{align}\label{eq::re3}
\begin{split}
R_{t}(\widehat{\mathbf{w}})-R_{\mathcal{S}_{\Omega}}(\widehat{\mathbf{w}}) \leq R_{t}(\widehat{\mathbf{w}}) - R_{\Omega}(\widehat{\mathbf{w}}) = \mathbb{E} \ \ell\left(\mathbf{f}_{\widehat{\mathbf{w}}}\circ \mathbf{G} ; S_t,C, Y\right) - \max_{S_{\boldsymbol{\alpha}} \in \Omega}\mathbb{E} \ell\left(\mathbf{f}_{\widehat{\mathbf{w}}}\circ \mathbf{G} ; S_{\boldsymbol{\alpha}},C, Y\right).
\end{split}
\end{align}
We set
\begin{align}
S_{\boldsymbol{\alpha_{M}}} = {\arg\min}_{S_{\boldsymbol{\alpha}} \in \Omega}\mathrm{W}_{c}\left(D_{S_{t}}, D_{S_{\boldsymbol{\alpha}}}\right),
\end{align}
then
\begin{align}\label{eq::re4}
\begin{split}
\mathbb{E} \ \ell\left(\mathbf{f}_{\widehat{\mathbf{w}}}\circ \mathbf{G} ; S_t,C, Y\right) - \max_{S_{\boldsymbol{\alpha}} \in \Omega}\mathbb{E} \ell\left(\mathbf{f}_{\widehat{\mathbf{w}}}\circ \mathbf{G} ; S_{\boldsymbol{\alpha}},C, Y\right) \leq & \mathbb{E} \ \ell\left(\mathbf{f}_{\widehat{\mathbf{w}}}\circ \mathbf{G} ; S_t,C, Y\right) - \mathbb{E} \ell\left(\mathbf{f}_{\widehat{\mathbf{w}}}\circ \mathbf{G} ; S_{\boldsymbol{\alpha}_{M}},C, Y\right) \\ \leq & L_{c}\mathrm{W}_{c}\left(D_{S_{\mathrm{t}}}, D_{S_{\boldsymbol{\alpha}_{M}}}\right) =  \min_{S_{\boldsymbol{\alpha}} \in \Omega} L_{c}\mathrm{W}_{c}\left(D_{S_{\mathrm{t}}}, D_{S_{\boldsymbol{\alpha}}}\right).
\end{split}
\end{align}

Combining Eqs. \eqref{eq::re1}, \eqref{eq::re2}, \eqref{eq::re3} and \eqref{eq::re4}, we have
\begin{align}
R_{t}(\widehat{\mathbf{w}})-R^{\beta }_{\boldsymbol{\lambda}}(\widehat{\mathbf{w}} ) \leq (1-\beta)L_{c} \sum_{l=1}^N \lambda_l \mathrm{W}_{c}\left(D_{S_{\mathrm{t}}}, D_{S_{l}}\right)+ \beta L_{c} \min_{S_{\boldsymbol{\alpha}} \in \Omega} \mathrm{W}_{c}\left(D_{S_{\mathrm{t}}}, D_{S_{\boldsymbol{\alpha}}}\right).
\end{align}

Then by Theorem \ref{thm1}, we complete this proof.
\end{proof}
\subsection{Necessary Lammas}

\begin{lemma}\label{lemma1}
If  $0 \leq \ell (\mathbf{f_{w}} ; \mathbf{x}, y) \leq M_{\ell}$ , then with the probability at least  $1-e^{-t}>0$ , we have that for any  $\mathbf{w} \in \mathcal{W}$
\begin{align}
\begin{split}
\mathbb{E}_{(\mathbf{x}, y) \sim D_{X_l Y_l}} \ell (\mathbf{f_{w}} ; \mathbf{x}, y)-\frac{1}{n} \sum_{i=1}^{n} \ell\left(\mathbf{f_{w}} ; \mathbf{x}_l^{i}, y_l^{i}\right) \\ \leq \frac{b_{0} M_{\ell}}{\sqrt{n}} \int_{0}^{1} \sqrt{\log \mathcal{N}\left(\mathcal{F}, M_{\ell} \epsilon, L^{\infty}\right)} \mathrm{d} \epsilon+M_{\ell} \sqrt{\frac{2 t}{n}}
\end{split}
\end{align}
,where $b_{0}$ is a uniform constant.
\end{lemma}
\begin{proof}
Let
\begin{align}
\begin{split}
X_{\ell (\mathbf{f_{w}} ; \cdot)}=\mathbb{E}_{(\mathbf{x}, y) \sim D_{X_{1} Y_l}} \ell (\mathbf{f_{w}} ; \mathbf{x}, y)-\frac{1}{n} \sum_{i=1}^{n} \ell\left(\mathbf{w} ; \mathbf{x}_l^{i}, y_l^{i}\right) .
\end{split}
\end{align}
Then it is clear that
\begin{align}
\begin{split}
\mathbb{E}_{S \sim D_{X_{1} Y_{1}}^{n}} X_{\ell (\mathbf{f_{w}} ; \cdot)}=0 .
\end{split}
\end{align}
By Proposition 2.6.1 and Lemma 2.6.8 in \citet{vershynin2018high},
\begin{align}
\begin{split}
\left\|X_{\ell (\mathbf{f_{w}} ; \cdot)}-X_{\ell\left(\mathbf{f_{w^{\prime}}} ; \cdot\right)}\right\|_{\Phi_{2}} \leq \frac{c_{0}}{\sqrt{n}}\left\|\ell (\mathbf{f_{w}} ; \cdot)-\ell\left(\mathbf{f_{w^{\prime}}} ; \cdot\right)\right\|_{L^{\infty}},
\end{split}
\end{align}
where  $\|\cdot\|_{\Phi_{2}}$  is the sub-gaussian norm and  $c_{0}$  is a uniform constant. Therefore, by Dudley's entropy integral \cite{vershynin2018high}, we have
\begin{align}
\begin{split}
&\mathbb{E}_{S \sim D_{X_{1} Y_{1}}^{n}} \sup _{\mathbf{w} \in \mathcal{W}} X_{\ell (\mathbf{f_{w}} ; \cdot)} \\ \leq &\frac{b_{0}}{\sqrt{n}} \int_{0}^{+\infty} \sqrt{\log \mathcal{N}\left(\mathcal{F}, \epsilon, L^{\infty}\right)} \mathrm{d} \epsilon,
\end{split}
\end{align}
where  $b_{0}$  is a uniform constant and
\begin{align}
\mathcal{F}=\{\ell (\mathbf{f_{w}} ; \cdot): \mathbf{w} \in \mathcal{W}\} .
\end{align}
Note that
\begin{align}
\begin{split}
\mathbb{E}_{S \sim D_{X_{1} Y_{1}}^{n}} \sup _{\mathbf{w} \in \mathcal{W}} X_{\ell (\mathbf{f_{w}} ; \cdot)} & \leq \frac{b_{0}}{\sqrt{n}} \int_{0}^{+\infty} \sqrt{\log\mathcal{N}\left(\mathcal{F}, \epsilon, L^{\infty}\right)} \mathrm{d} \epsilon \\
& =\frac{b_{0}}{\sqrt{n}} \int_{0}^{M_{\ell}} \sqrt{\log\mathcal{N}\left(\mathcal{F}, \epsilon, L^{\infty}\right)} \mathrm{d} \epsilon \\
& =\frac{b_{0}}{\sqrt{n}} M_{\ell} \int_{0}^{1} \sqrt{\log \mathcal{N}\left(\mathcal{F}, M_{\ell} \epsilon, L^{\infty}\right)} \mathrm{d} \epsilon
\end{split}
\end{align}
Then, similar to the proof of Lemma \ref{lemma2}, we use the McDiarmid's Inequality, then with the probability at least  $1-e^{-t}>0$ , for any  $\mathbf{w} \in \mathcal{W}$ ,
\begin{align}
X_{\ell (\mathbf{f_{w}} ;))} \leq \frac{b_{0}}{\sqrt{n}} M_{\ell} \int_{0}^{1} \sqrt{\log \mathcal{N}\left(\mathcal{F}, M_{\ell} \epsilon, L^{\infty}\right)} \mathrm{d} \epsilon+M_{\ell} \sqrt{\frac{2 t}{n}} .
\end{align}
\end{proof}

\begin{lemma}\label{lemma2}
If  $0 \leq \ell (\mathbf{f_{w}} ; \mathbf{x}, y) \leq M_{\ell}$ , then for a fixed $ \mathbf{w}_{0} \in \mathcal{W}$ , with the probability at least  $1-e^{-t}>0$ ,
\begin{align}
\begin{split}
\frac{1}{n} \sum_{i=1}^{n} \ell\left(\mathbf{w}_{0} ; \mathbf{x}_l^{i}, y_l^{i}\right)-\mathbb{E}_{(\mathbf{x}, y) \sim D_{X_l Y_l}} \ell\left(\mathbf{w}_{0} ; \mathbf{x}, y\right) \leq M_{\ell} \sqrt{\frac{2 t}{n}} .
\end{split}
\end{align}
\end{lemma}
\begin{proof}
By \cite{sinha2017certifying}, it is clear that
\begin{align}
\begin{split}
\sup _{\mathrm{w}_{c}\left(D_{X^{\prime}}, D_{X_{\mathrm{A}}}\right) \leq \rho} \mathbb{E}_{\mathbf{x} \sim D_{X^{\prime}}} \ell\left(\mathbf{f}_{\mathbf{w}_{0}} ; \mathbf{x}\right)=\inf _{\gamma \geq 0}\left[\gamma \rho+\mathbb{E}_{\mathbf{x} \sim D_{X_{\mathrm{A}}}} \phi_{\gamma}\left(\mathbf{w}_{0} ; \mathbf{x}\right)\right]
\end{split}
\end{align}
Therefore, for each $\epsilon>0$ , there exists a constant  $\gamma_{\epsilon} \geq 0 $ such that
\begin{align}
\begin{split}
\gamma_{\epsilon} \rho+\mathbb{E}_{\mathbf{x} \sim D_{X_{\mathrm{A}}}} \phi_{\gamma_{\epsilon}}\left(\mathbf{w}_{0} ; \mathbf{x}\right) \leq \sup _{\mathrm{w}_{c}\left(D_{X^{\prime}, D_{X_{\mathrm{A}}}}\right) \leq \rho} \mathbb{E}_{\mathbf{x} \sim D_{X^{\prime}}} \ell\left(\mathbf{f}_{\mathbf{w}_{0}} ; \mathbf{x}\right)+\epsilon .
\end{split}
\end{align}
Combining the above inequality and McDiarmid's Inequality, then with the probability at least
\begin{align}
\begin{split}
1-\exp \left(\frac{-\epsilon_{0}^{2} m}{2 M_{\ell \mathrm{OE}}^{2}}\right)>0,
\end{split}
\end{align}
we have
\begin{align}
\begin{split}
\mathbb{E}_{\mathbf{x} \sim \widehat{D}_{X_{\mathrm{A}}}} \phi_{\gamma_{\epsilon}}\left(\mathbf{w}_{0} ; \mathbf{x}\right) \leq \mathbb{E}_{\mathbf{x} \sim D_{X_{\mathrm{A}}}} \phi_{\gamma_{\epsilon}}\left(\mathbf{w}_{0} ; \mathbf{x}\right)+\epsilon_{0}
\end{split}
\end{align}
If we set  $t=\epsilon_{0}^{2} m / 2 M_{\ell}^{2}$ , then
\begin{align}
\begin{split}
\epsilon_{0}=M_{\ell} \sqrt{\frac{2 t}{m}}
\end{split}
\end{align}
Hence, with the probability at least  $1-e^{-t}>0$ , we have
\begin{align}
\begin{split}
\gamma_{\epsilon} \rho+\mathbb{E}_{\mathbf{x} \sim \widehat{D}_{X_{\mathrm{A}}}} \phi_{\gamma_{\epsilon}}\left(\mathbf{w}_{0} ; \mathbf{x}\right) \leq \sup _{\mathrm{w}_{c}\left(D_{X^{\prime}}, D_{X_{\mathrm{A}}}\right) \leq \rho} \mathbb{E}_{\mathbf{x} \sim D_{X^{\prime}}} \ell\left(\mathbf{f}_{\mathbf{w}_{0}} ; \mathbf{x}\right)+\epsilon+M_{\ell} \sqrt{\frac{2 t}{m}},
\end{split}
\end{align}
which implies that with the probability at least  $1-e^{-t}>0$ ,
\begin{align}
\begin{split}
\sup _{\mathrm{W}_{c}\left(D_{X^{\prime}}, \widehat{D}_{X_{\mathrm{A}}}\right) \leq \rho} \mathbb{E}_{\mathbf{x} \sim D_{X^{\prime}}} \ell\left(\mathbf{f}_{\mathbf{w}_{0}} ; \mathbf{x}\right) \leq \sup _{\mathrm{W}_{c}\left(D_{X^{\prime}}, D_{X_{\mathrm{A}}}\right) \leq \rho} \mathbb{E}_{\mathbf{x} \sim D_{X^{\prime}}} \ell\left(\mathbf{f}_{\mathbf{w}_{0}} ; \mathbf{x}\right)+\epsilon+M_{\ell} \sqrt{\frac{2 t}{m}},
\end{split}
\end{align}
because
\begin{align}
\begin{split}
\gamma_{\epsilon} \rho+\mathbb{E}_{\mathbf{x} \sim \widehat{D}_{X_{\mathrm{A}}}} \phi_{\gamma_{\epsilon}}\left(\mathbf{w}_{0} ; \mathbf{x}\right) \geq \sup _{\mathrm{W}_{c}\left(D_{X^{\prime}}, \widehat{D}_{X_{\mathrm{A}}}\right) \leq \rho} \mathbb{E}_{\mathbf{x} \sim D_{X^{\prime}}} \ell\left(\mathbf{f}_{\mathbf{w}_{0}} ; \mathbf{x}\right)
\end{split}
\end{align}
By setting  $\epsilon=M_{\ell} \sqrt{2 t / m}$ and $\rho = 0$, we complete this proof.

\end{proof}

\begin{lemma}\label{lemma3}
If

$\bullet$ $0 \leq \ell (\mathbf{f_{w}} ; \mathbf{x}, y) \leq M_{\ell}$ ;

$\bullet$ $\ell(\cdot ; \mathbf{x}, y)$  is $L$-Lipschitz w.r.t. norm $\| \cdot \|$, i.e., for any  $(\mathbf{x}, y) \in \mathcal{X} \times \mathcal{Y}$ , and  $\mathbf{w}, \mathbf{w}^{\prime} \in \mathcal{W} $,
\begin{align}
\begin{split}
\left|\ell (\mathbf{f_{w}} ; \mathbf{x}, y)-\ell\left(\mathbf{f_{w^{\prime}}} ; \mathbf{x}, y\right)\right| \leq L\left\|\mathbf{w}-\mathbf{w}^{\prime}\right\|,
\end{split}
\end{align}

$\bullet$ the parameter space  $\mathcal{W} \subset \mathbb{R}^{d^{\prime}}$  satisfies that
\begin{align}
\operatorname{diam}(\mathcal{W})=\sup _{\mathbf{w}, \mathbf{w}^{\prime} \in \mathcal{W}}\left\|\mathbf{w}-\mathbf{w}^{\prime}\right\|<+\infty ,
\end{align}
then with the probability at least  $1-e^{-t}>0$ , we have that for any  $\mathbf{w} \in \mathcal{W}$ ,

\begin{align}
\begin{split}
& \mathbb{E}_{(\mathbf{x}, y) \sim D_{X_l Y_l}} \ell (\mathbf{f_{w}} ; \mathbf{x}, y)-\frac{1}{n} \sum_{i=1}^{n} \ell\left(\mathbf{f_{w}} ; \mathbf{x}_l^{i}, y_l^{i}\right) \\
\leq & b_{0} \sqrt{\frac{M_{\ell} \operatorname{diam}(\mathcal{W}) L d^{\prime}}{n}}+M_{\ell} \sqrt{\frac{2 t}{n}}
\end{split}
\end{align}
where $b_{0}$  is a uniform constant.
\end{lemma}
\begin{proof}
The proof is similar to Corollary 1 in \citet{sinha2017certifying}. Note that
\begin{align}
\begin{split}
\mathcal{F}=\{\ell (\mathbf{f_{w}} ; \mathbf{x}, y): \mathbf{w} \in \mathcal{W}\},
\end{split}
\end{align}
and  $\ell(\cdot ; \mathbf{x}, y)$  is  $L$-Lipschitz w.r.t. norm  $\|\cdot\|$ , therefore,

\begin{align}
\begin{split}
\mathcal{N}\left(\mathcal{F}, M_{\ell} \epsilon, L^{\infty}\right) & \leq \mathcal{N}\left(\mathcal{W}, M_{\ell} \epsilon / L,\|\cdot\|\right) \\
& \leq\left(1+\frac{\operatorname{diam}(\mathcal{W}) L}{M_{\ell} \epsilon}\right)^{d^{\prime}},
\end{split}
\end{align}

which implies that

\begin{align}
\begin{split}
& \int_{0}^{1} \sqrt{\log \left(\mathcal{N}\left(\mathcal{F}, M_{\ell} \epsilon, L^{\infty}\right)\right.} \mathrm{d} \epsilon \\
\leq &\sqrt{d^{\prime}} \int_{0}^{1} \sqrt{\log \left(1+\frac{\operatorname{diam}(\mathcal{W}) L}{M_{\ell} \epsilon}\right)} \mathrm{d} \epsilon \\
\leq &\sqrt{d^{\prime}} \int_{0}^{1} \sqrt{\frac{\operatorname{diam}(\mathcal{W}) L}{M_{\ell} \epsilon}} \mathrm{d} \epsilon=2 \sqrt{\frac{\operatorname{diam}(\mathcal{W}) L d^{\prime}}{M_{\ell}}}. \\
\end{split}
\end{align}

By Lemma \ref{lemma1}, we obtain this result.
\end{proof}

\begin{lemma}\label{lemma4}
If  $0 \leq \ell (\mathbf{f}_{\mathbf{w}}\circ \mathbf{G} ; \mathbf{s},\mathbf{c}, y) \leq M_{\ell}$ , then with the probability at least  $1-e^{-t}>0$ , we have that for any  $\mathbf{w} \in \mathcal{W}$
\begin{align}
\begin{split}
\mathbb{E}_{(\mathbf{x}, y) \sim D_{X_l Y_l}} \max_{\mathbf{s}\in \mathcal{S}_{\Omega}}\ell (\mathbf{f}_{\mathbf{w}}\circ \mathbf{G} ; \mathbf{s},\mathbf{c}, y)-\frac{1}{n} \sum_{i=1}^{n} \max_{\mathbf{s}\in \mathcal{S}_{\Omega}}\ell\left(\mathbf{f}_{\mathbf{w}}\circ \mathbf{G} ; \mathbf{s},\mathbf{c}_l^{i}, y_l^{i}\right) \\ \leq \frac{b_{0} M_{\ell}}{\sqrt{n}} \int_{0}^{1} \sqrt{\log \mathcal{N}\left(\mathcal{F}, M_{\ell} \epsilon, L^{\infty}\right)} \mathrm{d} \epsilon+M_{\ell} \sqrt{\frac{2 t}{n}}
\end{split}
\end{align}
,where $b_{0}$ is a uniform constant.
\end{lemma}
\begin{proof}
    The proof is similar to the proof of Lemma \ref{lemma1}.
\end{proof}

\begin{lemma}\label{lemma5}
If  $0 \leq \ell (\mathbf{f}_{\mathbf{w}}\circ \mathbf{G} ; \mathbf{s},\mathbf{c}, y) \leq M_{\ell}$ , then for a fixed $ \mathbf{w}_{0} \in \mathcal{W}$ , with the probability at least  $1-e^{-t}>0$ ,
\begin{align}
\begin{split}
\frac{1}{n} \sum_{i=1}^{n}  \max_{\mathbf{s}\in \mathcal{S}_{\Omega}}\ell\left(\mathbf{f}_{\mathbf{w}}\circ \mathbf{G} ; \mathbf{s},\mathbf{c}_l^{i}, y_l^{i}\right) - \mathbb{E}_{(\mathbf{x}, y) \sim D_{X_l Y_l}} \max_{\mathbf{s}\in \mathcal{S}_{\Omega}}\ell (\mathbf{f}_{\mathbf{w}}\circ \mathbf{G} ; \mathbf{s},\mathbf{c}, y) \leq M_{\ell} \sqrt{\frac{2 t}{n}} .
\end{split}
\end{align}
\end{lemma}
\begin{proof}
    The proof is similar to the proof of Lemma \ref{lemma2}.
\end{proof}

\begin{lemma}\label{lemma6}
If

$\bullet$ $0 \leq \ell (\mathbf{f}_{\mathbf{w}}\circ \mathbf{G} ; \mathbf{s},\mathbf{c}, y) \leq M_{\ell}$ ;

$\bullet$ $\ell(\cdot ; \mathbf{x}, y)$  is $L$-Lipschitz w.r.t. norm $\| \cdot \|$, i.e., for any  $(\mathbf{x}, y) \in \mathcal{X} \times \mathcal{Y}$ , and  $\mathbf{w}, \mathbf{w}^{\prime} \in \mathcal{W} $,
\begin{align}\label{6.1}
\begin{split}
\left|\ell (\mathbf{f_{w}} ; \mathbf{x}, y)-\ell\left(\mathbf{f_{w^{\prime}}} ; \mathbf{x}, y\right)\right| \leq L\left\|\mathbf{w}-\mathbf{w}^{\prime}\right\|,
\end{split}
\end{align}

$\bullet$ the parameter space  $\mathcal{W} \subset \mathbb{R}^{d^{\prime}}$  satisfies that
\begin{align}
\operatorname{diam}(\mathcal{W})=\sup _{\mathbf{w}, \mathbf{w}^{\prime} \in \mathcal{W}}\left\|\mathbf{w}-\mathbf{w}^{\prime}\right\|<+\infty ,
\end{align}
then with the probability at least  $1-e^{-t}>0$ , we have that for any  $\mathbf{w} \in \mathcal{W}$ ,

\begin{align}
\begin{split}
& \mathbb{E}_{(\mathbf{x}, y) \sim D_{X_l Y_l}} \max_{\mathbf{s}\in \mathcal{S}_{\Omega}}\ell (\mathbf{f}_{\mathbf{w}}\circ \mathbf{G} ; \mathbf{s},\mathbf{c}, y)-\frac{1}{n} \sum_{i=1}^{n} \max_{\mathbf{s}\in \mathcal{S}_{\Omega}}\ell\left(\mathbf{f}_{\mathbf{w}}\circ \mathbf{G} ; \mathbf{s},\mathbf{c}_l^{i}, y_l^{i}\right) \\
\leq & b_{0} \sqrt{\frac{M_{\ell} \operatorname{diam}(\mathcal{W}) L d^{\prime}}{n}}+M_{\ell} \sqrt{\frac{2 t}{n}}
\end{split}
\end{align}
where $b_{0}$  is a uniform constant.
\end{lemma}
\begin{proof}
The condition \eqref{6.1} is equivalent to

$\bullet$ $\ell(\cdot ; G(\mathbf{s},\mathbf{c}), y)$  is $L_{\mathrm{G}}$-Lipschitz w.r.t. norm $\| \cdot \|$, i.e., for any  $(\mathbf{s},\mathbf{c}, y) \in \mathcal{S} \times \mathcal{C} \times\mathcal{Y}$ , and  $\mathbf{w}, \mathbf{w}^{\prime} \in \mathcal{W} $,
\begin{align}
\begin{split}
\left|\ell (\mathbf{f}_{\mathbf{w}}\circ \mathbf{G} ;\mathbf{s},\mathbf{c}, y)-\ell\left(\mathbf{f_{w^{\prime}}}\circ \mathbf{G} ;\mathbf{s},\mathbf{c}, y\right)\right| \leq L_{\mathrm{G}}\left\|\mathbf{w}-\mathbf{w}^{\prime}\right\|,
\end{split}
\end{align}

Then the proof is similar to the proof of Lemma \ref{lemma3}.
\end{proof}

\section{Details of Realization}
\subsection{MODE-F: Fourier-based MODE}\label{fb}

\textbf{Fourier-based Transfer.}
The Fourier-based transfer has been used in many domain generalization methods. This transfer method is considered able to separate the stylistic information from the semantic information by using the discrete Fourier transform to decompose the image into its amplitude and phase, and then create more samples of different styles using different mixing methods for amplitude.

For a image $x$, its discrete Fourier transformation $\mathcal{F}(\mathbf{x})$:
\begin{align}
\label{fourier}
    \mathcal{F}(\mathbf{x})(\mathbf{u},\mathbf{ v}) = \sum_{h = 0}^{H-1} \sum_{w = 0}^{W-1} x(h, w) e^{-j 2 \pi\left(\frac{h}{H} u+\frac{w}{W} v\right)}
\end{align}
$\mathcal{F}^{-1}(x)$ is defined as the discrete inverse Fourier transformation. And both of these transformations can be implemented using FFT and do not require additional neural networks. 

After discrete Fourier transformation, the amplitude and phase are defined as:
\begin{align}
\label{amplitude}
\mathcal{A}(\mathbf{x})(\mathbf{u},\mathbf{ v}) & = \left[R^{2}(\mathbf{x})(\mathbf{u},\mathbf{ v})+I^{2}(\mathbf{x})(\mathbf{u},\mathbf{ v})\right]^{1 / 2} \\
\mathcal{P}(\mathbf{x})(\mathbf{u},\mathbf{ v}) & = \arctan \left[\frac{I(\mathbf{x})(\mathbf{u},\mathbf{ v})}{R(\mathbf{x})(\mathbf{u},\mathbf{ v})}\right] \label{phase}
\end{align}
where $R(x)$ and $I(x)$ represent the real and imaginary part of $\mathcal{F}(\mathbf{x})$, respectively.

In \cite{xu2021fourier}, Fourier-based data augmentation is implemented by:
\begin{align}
\hat{\mathcal{A}}_{\gamma}\left(\mathbf{x}\right) & = (1-\lambda) \mathcal{A}\left(\mathbf{x}\right)+\lambda \mathcal{A}\left(\mathbf{x}^{\prime}\right) \label{mix} \\
\hat{x} & = \mathcal{F}^{-1}\left[\hat{\mathcal{A}}_{\gamma}\left(\mathbf{x}\right)(\mathbf{u},\mathbf{ v}) * e^{-j * \mathcal{P}\left(\mathbf{x}\right)(\mathbf{u},\mathbf{ v})}\right] \label{ifourier}
\end{align}

where $x^{\prime}$ represents randomly selected image, $\lambda \sim U(0, \eta)$, and the hyperparameter $\eta$ controls the strength of the augmentation.$\hat{x}$ represents the augmented images.

\textbf{Adapt Fourier-based Transfer.}
Although the previous Fourier-based transfer methods have achieved good results in many domain generalization methods \cite{xu2021fourier,lv2022causality}, our approach requires that the transformation method used should be controllable. More specifically, It means that the direction of style change in the transformation process can be controlled by some parameters, and these parameters can be updated regularly in the process.

To meet the requirements, since a single image could not provide a sufficient amount of style, we randomly select M images $\mathbf{x}_{1},\mathbf{x}_{2},\cdots,\mathbf{x}_{M}$ as style providers and calculate their amplitudes $\mathcal{A}\left(\mathbf{x}_{1}\right),\mathcal{A}\left(\mathbf{x}_{2}\right),\cdots,\mathcal{A}\left(\mathbf{x}_{M}\right)$, and then acquire their linear combination of amplitudes:
\begin{align}
\label{mix2}
\hat{\mathcal{A}}_{\gamma}\left(\mathbf{x}\right) & = \gamma[\alpha_{0} \mathcal{A}\left(\mathbf{x}\right)+\sum_{i=1}^{M} \alpha_{i} \mathcal{A}\left(\mathbf{x}_{i}\right)]+(1-\gamma)\mathcal{A}\left(\mathbf{x}\right) \\
\hat{x} & = \mathcal{F}^{-1}\left[\hat{\mathcal{A}}_{\gamma}\left(\mathbf{x}\right)(\mathbf{u},\mathbf{v}) * e^{-j * \mathcal{P}\left(\mathbf{x}\right)(\mathbf{u},\mathbf{ v})}\right] \label{xhat}
\end{align}
where $\alpha_{0},\alpha_{1},\cdots,\alpha_{M}$ are the parameter of linear weighted which could be updated, $\gamma$ is hyperparameter to determine what percentage of the amplitude is involved in the searching.

By controlling the parameter $\alpha_{0},\alpha_{1},\cdots,\alpha_{M}$, we can control the changes in the augmented output's style. 

\textbf{Adversarial-attacks-inspired Update Strategy.}
 Motivated by these theoretical insights, in each iteration, the parameters $\alpha_{0},\alpha_{1},\cdots,\alpha_{M}$ will be updated by maximizing the empirical risk of the augmented sample, so that we could guide the augmented sample to be closer to the distribution with highest empirical risk.
 
 Inspired by Adversarial attacks \cite{szegedy2013intriguing,goodfellow2014explaining,madry2017towards,zhang2020principal}, to speed up the searching process, after the gradient backpropagation, we update the parameters $\alpha_{0},\alpha_{1},\cdots,\alpha_{M}$ using the gradient's direction and fixed step size:
\begin{align}
\label{updatea}
\tilde{\boldsymbol{\alpha}}^{k} & = \boldsymbol{\alpha}^{k-1}+\mu \operatorname{sign}\left(\nabla_{\boldsymbol{\alpha}} \ell(\mathbf{f_{w}};\hat{\mathbf{x}}^{k-1},y)\right) \\
\alpha^{k}_{l} & = \tilde{\alpha}^{k}_{l} / \sum_{i = 0}^{M} \tilde{\alpha}^{k}_{i} \label{norm}
\end{align}  
where $\boldsymbol{\alpha}^{k} = [\alpha_{0}^{k},\alpha_{1}^{k},\cdots,\alpha_{M}^{k}]^\top$, $\mu$ is the step size, $\hat{x}^{k}$ comes from Equ Equ $\ref{mix2}$, Equ $\ref{xhat}$ with $\alpha^{k}$, $\ell$ is standard Cross-entropy loss, $\mathbf{f_{w}}$ is the network.
Equ $\ref{norm}$ means that since the sum is limited to 1, the parameters $\alpha_{0},\alpha_{1},\cdots,\alpha_{M}$ are normalized after each update.

After $K$ iterations in inner optimization, we get $\boldsymbol{\alpha}^{K}$ , and then we calculate $\hat{x}^{final}$ using Equ $\ref{mix2}$, Equ $\ref{xhat}$ and $\boldsymbol{\alpha}^{K}$.
 
Finally, to maintain the category label and thus enforce semantic consistency, we require that the generated sample $\hat{\mathbf{x}}^{final}$ is classified into the same category together with the original sample $\mathbf{x}$, and calculate the total loss to update network $\mathbf{f_{w}}$:
\begin{align}
\label{totalloss}
\mathcal{L}_{MODE} & = (1-\beta)\ell(\mathbf{f_{w}};\mathbf{x},y) + \beta  \ell(\mathbf{f_{w}};\hat{\mathbf{x}}^{final},y)
\end{align}
where $\ell$ is standard Cross-entropy loss, $\beta$ is hyperparameter to control the influence of augmented images.

The full Fourier-based training algorithm is shown in Algorithm $\ref{alg:fourier}$ and Figure $\ref{fig:fourier}$.

\begin{figure}[]
	\centering
	\includegraphics[width=0.7\textwidth]{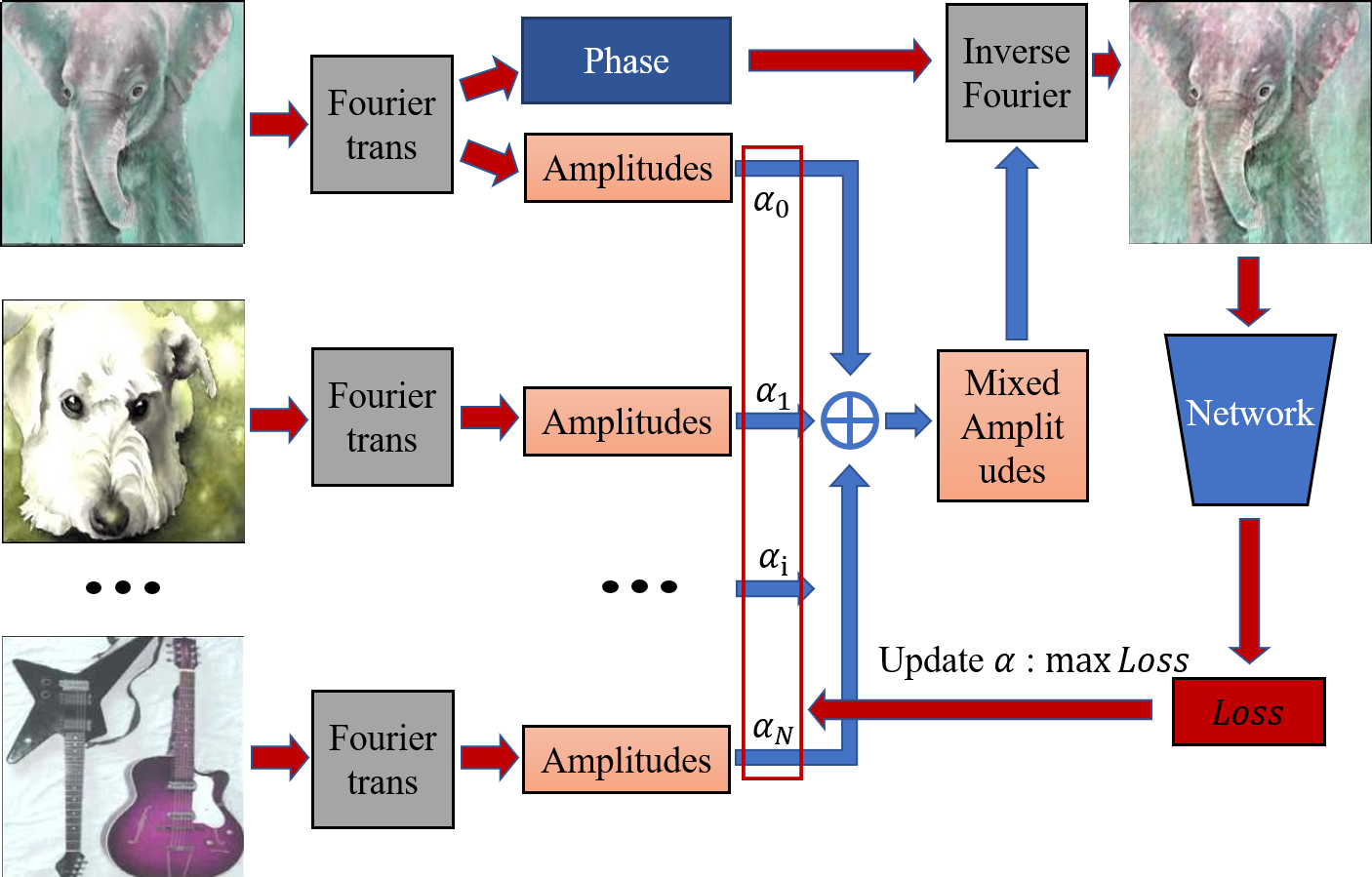}
	\caption{Fourier-based algorithm approach. The top left picture is the input picture, and the rest are the amplitude provider. In an iteration, the controlling parameters $\alpha$ are used to mix the amplitudes to obtain a new amplitude, together with the phase of the original input image to generate an augmented image. The augmented is input to the network to calculate the loss, and the controlling parameters $\alpha$ are updated by maximizing the loss.}
	\label{fig:fourier}
\end{figure}
\begin{figure}[]
	\centering
	\includegraphics[width=0.7\textwidth]{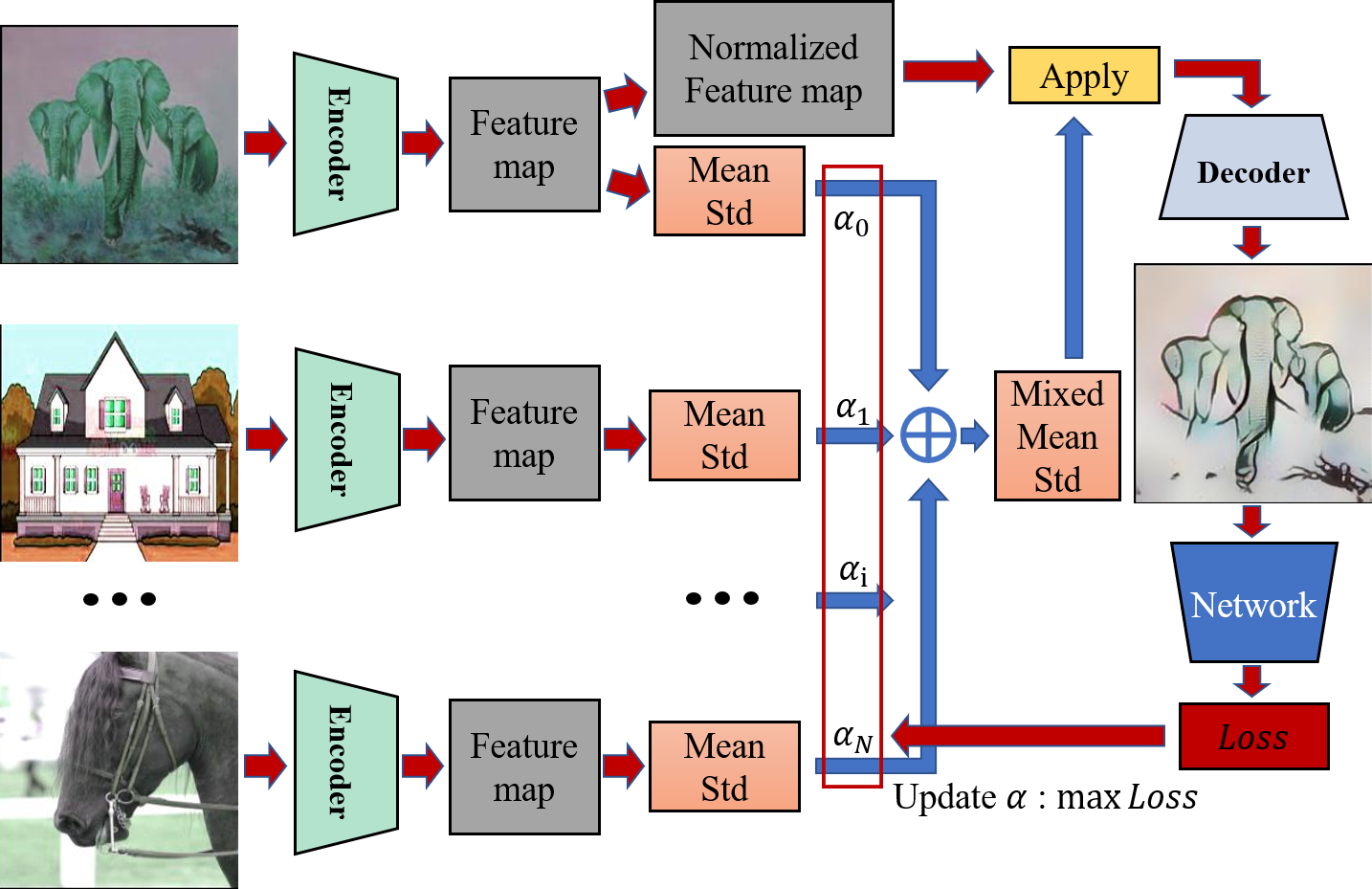}
	\caption{AdaIN-based approach. The top left image is the input content image, and the rest are the style provider. In an iteration, the controlling parameters $\alpha$ are used to mix the mean and std to obtain new mean and std, which is applied to the normalized feature map of the original input content image to generate an augmented image. The augmented is input to the network to calculate the loss, and the controlling parameters $\alpha$ is updated by maximizing the loss.}
	\label{fig:adain}
\end{figure}
\begin{algorithm}[tb]
   \caption{Fourier-based MODE}
   \label{alg:fourier}
\begin{algorithmic}
   \STATE {\bfseries Input:} data $x_i$, batch size $n$, number of iterations $K$ in inner optimization, step size $\mu$, number of amplitude providers $M$, network architecture parametrized by $\mathbf{w}$, hyperparameter $\beta$ and $\gamma$
   \STATE {\bfseries Output:} Robust network $\mathbf{f_{w}}$
   \STATE Randomly initialize network $\mathbf{f_{w}}$, or initialize network with pre-trained configuration
   \REPEAT
   \STATE Read mini-batch $\boldsymbol{x} = [\mathbf{x}_{1}, ..., \mathbf{x}_{n}]$ from training set
   \STATE Calculate their phase and amplitude as $[P_{1}, ..., P_{n}]$,  $[A_{1}, ..., A_{n}]$, respectively
   \STATE \textit{\textcolor{blue}{\# Exploration}}
   \FOR{$i=1$ {\bfseries to} $n$ (in parallel)} 
   \STATE Randomly select $M$ other image's amplitudes in $B$ as amplitude providers $[\hat{A}_{1}, ..., \hat{A}_{M}]$
   \STATE Initialize $\alpha_{0},\alpha_{1},\cdots,\alpha_{M}$ as $\boldsymbol{\alpha}^{0}$
   \FOR{$k=1$ {\bfseries to} $K$}
   \STATE Calculate $\hat{\mathbf{x}}_{i}^{k}$ using Equ $\ref{mix2},\ref{xhat}$ and $\boldsymbol{\alpha}^{k-1}$
   \STATE Update $\boldsymbol{\alpha}^{k}$ using Equ $\ref{updatea}, \ref{norm}$
   \ENDFOR
   \STATE Calculate $\hat{\mathbf{x}}_{i}^{final}$ using Equ $\ref{mix2}, \ref{xhat}$ and $\boldsymbol{\alpha}^{K}$
   \ENDFOR
   \STATE \textit{\textcolor{blue}{\# Update Model}}
   \STATE Calculate $\mathcal{L}_{MODE}$ using Equ $\ref{totalloss}$
   \STATE Update $\mathbf{w}$ by performing one step gradient update using $\nabla_{\mathbf{w}}\mathcal{L}_{MODE}$
   \UNTIL{training converged}
\end{algorithmic}
\end{algorithm}

\begin{figure*}[!t]
	\centering
	\includegraphics[width=0.9\textwidth]{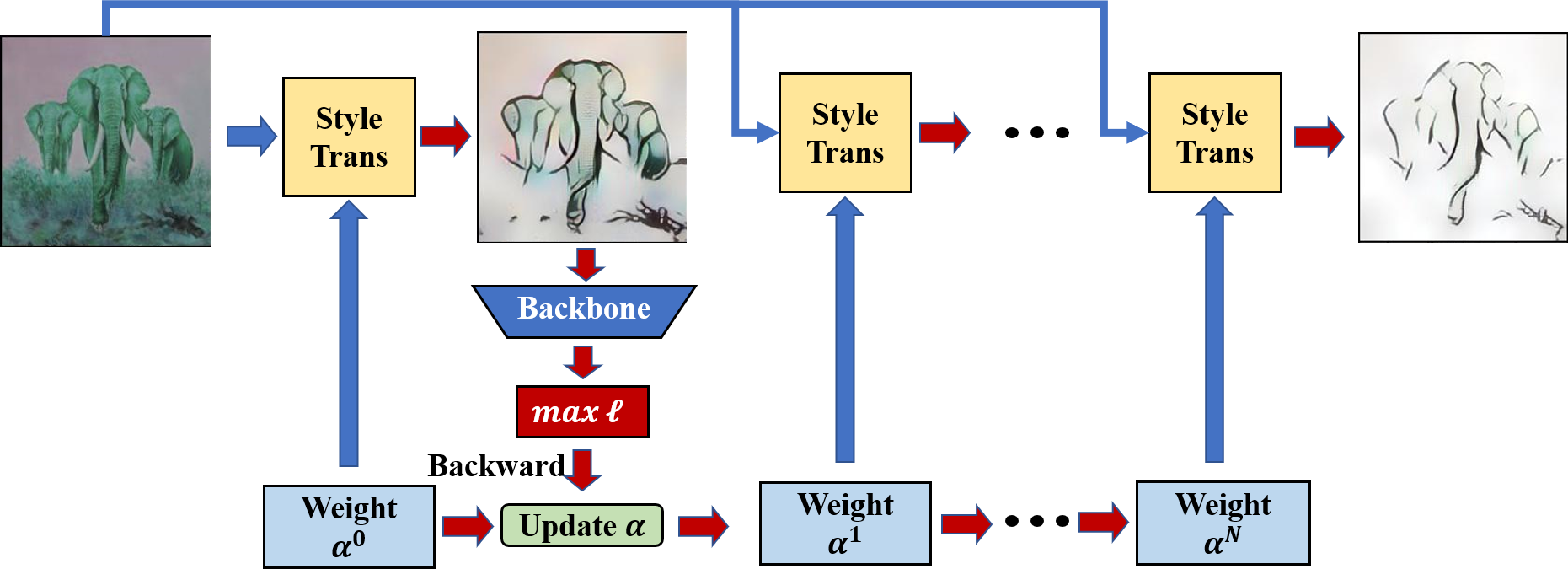}
	\caption{The overall approach. Using existing generation methods like Style trans, \textbf{MODE} could create more aggressive samples by updating the controlling parameters $\alpha$ through multiple steps, inspired by Adversarial attack.}
	\label{fig:framework2}
\end{figure*}

\subsection{MODE-A: AdaIN-based MODE}\label{ab}

\textbf{Neural Style Transfer.}
Although this Fourier-based transfer method has been widely used in many domain generalization methods, it still has obvious disadvantages: most of the changes in the generated augmented image are reflected in adding irregular color blocks and textures to the original image, which is rarely seen in reality. At the same time, compared with the good results applied to some low-resolution datasets such as handwritten digit datasets, when applied to some real-world high-resolution datasets, the generated results lacking authenticity are also difficult to be satisfactory. It prompts us to consider other style transfer methods with better results. 

Neural Style Transfer \cite{huang2017arbitrary} has developed rapidly in recent years. Compared with the Fourier-based method, using the pre-trained neural network model to process the image, the results generated by neural style transfer are usually more authentic. AdaIN \cite{huang2017arbitrary} is one of the representative methods of neural style transfer. It uses the mean and std of feature map output by the fixed encoder to represent style information and trains a decoder to restore stylized images from the feature map whose mean and std had been changed. 

To apply AdaIN in our approach, we use the mean and std introduced above to represent the non-semantic factor $S$ and use the normalized feature map to represent the semantic factor $C$.

AdaIN consists of an encoder $E$, a decoder $D$, and a mean-std processing module for the feature map. Encoder $E$ will process input content image $x$ and style image $\hat{\mathbf{x}}$ into feature map $\mathbf{z}=E(\mathbf{x})$, $\hat{\mathbf{z}}=E(\hat{\mathbf{x}})$, respectively. And then mean-std processing module will calculate the mean and std of the $\mathbf{z}$ and $\hat{\mathbf{z}}$ separately:
\begin{align}
\boldsymbol{\mu}(\mathbf{z}) & = \frac{1}{H W} \sum_{h = 1}^{H} \sum_{w = 1}^{W} \mathbf{z}_{c, h, w} \\
\boldsymbol{\sigma}(\mathbf{z}) & = \sqrt{\frac{1}{H W-1} \sum_{h = 1}^{H} \sum_{w = 1}^{W}\left(\mathbf{z}_{c, h, w}-\boldsymbol{\mu}(\mathbf{z})\right)^{2}}
\end{align}
where $\mathbf{z}$ should be a feature map of shape $C \times H \times W$ with $C, H, W$ being the number of channels, height, and width.

After obtaining the $\boldsymbol{\mu}(\mathbf{z}), \boldsymbol{\sigma}(\mathbf{z}), \boldsymbol{\mu}(\hat{\mathbf{z}}), \boldsymbol{\sigma}(\hat{\mathbf{z}})$, module will control their mixing by hyperparameter $\lambda$ :
\begin{align}
\tilde{\boldsymbol{\mu}}(\mathbf{z},\hat{\mathbf{z}},\lambda ) & = \lambda \boldsymbol{\mu}(\mathbf{z}) +(1-\lambda )\boldsymbol{\mu}(\hat{\mathbf{z}}) \label{mean} \\
\tilde{\boldsymbol{\sigma}}(\mathbf{z},\hat{\mathbf{z}},\lambda ) & = \lambda \boldsymbol{\sigma}(\mathbf{z}) +(1-\lambda )\boldsymbol{\sigma}(\hat{\mathbf{z}}) \label{std}
\end{align}
where $0 \le\lambda\le 1$ is the hyperparameter to control the level of stylization

And then mean-std processing module will apply the mixed mean and std to the normalized feature map of the content image :
\begin{align}
\tilde{\mathbf{z}} & = \tilde{\boldsymbol{\mu}}(\mathbf{z},\hat{\mathbf{z}},\lambda )+\tilde{\boldsymbol{\sigma}}(\mathbf{z},\hat{\mathbf{z}},\lambda ) \frac{\mathbf{z}-\boldsymbol{\mu}(\mathbf{z})}{\boldsymbol{\sigma}(\mathbf{z})}
\end{align}
where $\tilde{\mathbf{z}}$ is the output feature map.

Finally, the decoder $D$ will process the output feature map $\tilde{\mathbf{z}}$ to restore stylized images $\tilde{x} = D(\tilde{\mathbf{z}})$.

\textbf{Adapt AdaIN Transfer.}
Original AdaIN style transfer can only control the degree of stylization but not the direction of stylization.
We should make AdaIN style transfer controllable, i.e. the direction of style change in the transformation process can be controlled by some parameters, and these parameters can be updated regularly in the process.

To satisfy this requirement, since AdaIN use mean and std of feature map to represent style information,  we randomly select M images $\mathbf{x}_{1},\mathbf{x}_{2},\cdots,\mathbf{x}_{M}$ as style providers and calculate their feature map output using encoder $\mathbf{z}_{1} = E(\mathbf{x}_{1}),\mathbf{z}_{2} = E(\mathbf{x}_{2}),\cdots,\mathbf{z}_{M} = E(\mathbf{x}_{M})$, and finally calculate their mean and std of the feature map  $\boldsymbol{\mu}\left(\mathbf{z}_{1}\right),\boldsymbol{\mu}\left(\mathbf{z}_{2}\right),\cdots,\boldsymbol{\mu}\left(\mathbf{z}_{M}\right)$, $\boldsymbol{\sigma}\left(\mathbf{z}_{1}\right),\boldsymbol{\sigma}\left(\mathbf{z}_{2}\right),\cdots,\boldsymbol{\sigma}\left(\mathbf{z}_{M}\right)$, and then acquire their linear combination of mean and std:
\begin{align}
\tilde{\boldsymbol{\mu}}(\mathbf{z},\mathbf{z}_{1,\cdots,M},\alpha_{0,\cdots,M}) & = \alpha_{0} \boldsymbol{\mu}(\mathbf{z}) +\sum_{i=1}^{M}\alpha_{i} \boldsymbol{\mu}(\mathbf{z}_{i}) \label{meanmix}\\
\tilde{\boldsymbol{\sigma}}(\mathbf{z},\mathbf{z}_{1,\cdots,M},\alpha_{0,\cdots,M}) & = \alpha_{0} \boldsymbol{\sigma}(\mathbf{z}) +\sum_{i=1}^{M}\alpha_{i} \boldsymbol{\sigma}(\mathbf{z}_{i}) \label{stdmix}
\end{align}

And then apply the linearly mixed mean and std to the normalized feature map of content image $\mathbf{z}$, and use the decoder $D$ to restore stylized images $\tilde{x}$:
\begin{align}
\begin{split}
\tilde{\mathbf{z}}  &= \tilde{\boldsymbol{\mu}}(\mathbf{z},\mathbf{z}_{1,\cdots,M},\alpha_{0,\cdots,M}) \\ & +\tilde{\boldsymbol{\sigma}}(\mathbf{z},\mathbf{z}_{1,\cdots,M},\alpha_{0,\cdots,M}) \frac{\mathbf{z}-\boldsymbol{\mu}(\mathbf{z})}{\boldsymbol{\sigma}(\mathbf{z})} \label{apply}\end{split}\\
\tilde{x} & = D(\gamma\tilde{\mathbf{z}}+(1-\gamma)\mathbf{z}) \label{decoder}
\end{align}

where $\gamma$ is hyperparameter to determine what percentage of the amplitude is involved in the Exploration.

\textbf{Adversarial-attacks-inspired Update Strategy.}
Motivated by these theoretical insights, in each iteration, the parameters $\alpha_{0},\alpha_{1},\cdots,\alpha_{M}$ will be updated by maximizing the empirical risk of the augmented data, so that we could guide the augmented data to be closer to the distribution with highest empirical risk.
 
Inspired by adversarial attacks \cite{madry2017towards,szegedy2013intriguing,goodfellow2014explaining,zhang2020principal}, to speed up the process, after the gradient backpropagation, we update the parameters $\alpha_{0},\alpha_{1},\cdots,\alpha_{M}$ using the gradient's direction and fixed step size:
\begin{align}
\label{updatea2}
\tilde{\boldsymbol{\alpha}}^{k} & = \boldsymbol{\alpha}^{k-1}+\mu\operatorname{sign}\left(\nabla_{\boldsymbol{\alpha}} \ell(\mathbf{f_{w}};\hat{\mathbf{x}}^{k-1},y)\right) \\
\alpha^{k}_{l} & = \tilde{\alpha}^{k}_{l} / \sum_{i = 0}^{M} \tilde{\alpha}^{k}_{i} \label{norm2}
\end{align}  
where $\boldsymbol{\alpha}^{k} = [\alpha_{0}^{k},\alpha_{1}^{k},\cdots,\alpha_{M}^{k}]^\top$, $\mu$ is the step size, $\hat{x}^{k}$ comes from Equ $\ref{meanmix},\ref{stdmix},\ref{apply},\ref{decoder}$ with $\alpha^{k}$, $\ell$ is standard Cross-entropy loss, $\mathbf{f_{w}}$ is the network.
Equ $\ref{norm2}$ means that since the sum is limited to 1, the parameters $\alpha_{0},\alpha_{1},\cdots,\alpha_{M}$ are normalized after each update.

After $K$ iterations in inner optimization, we get $\boldsymbol{\alpha}^{K}$ , and then we calculate $\hat{x}^{final}$ using Equ $\ref{meanmix},\ref{stdmix},\ref{apply},\ref{decoder}$ and $\boldsymbol{\alpha}^{K}$.

Finally, we use the original image and augmented image to compute the total loss :
\begin{align}
\label{totalloss2}
\mathcal{L}_{MODE} & = (1-\beta)\ell(\mathbf{f_{w}};\mathbf{x},y) + \beta  \ell(\mathbf{f_{w}};\hat{\mathbf{x}}^{final},y)
\end{align}
where $\ell$ is standard Cross-entropy loss, $\beta$ is hyperparameter to control the influence of augmented images.

The full AdaIN-based training algorithm is shown in Algorithm \ref{alg:AdaIN} and Figure \ref{fig:adain}.
The overall approach is shown in Figure \ref{fig:framework2}.
\begin{algorithm}[h]
   \caption{AdaIN-based MODE}
   \label{alg:AdaIN}
\begin{algorithmic}
   \STATE {\bfseries Input:} data $x_i$, batch size $n$, number of iterations $K$ in inner optimization, step size $\mu$, number of style providers $M$, network architecture parametrized by $\mathbf{w}$, hyperparameter $\beta$ and $\gamma$
   \STATE {\bfseries Output:} Robust network $\mathbf{f_{w}}$
   \STATE Randomly initialize network $\mathbf{f_{w}}$, or initialize network with pre-trained configuration
   \REPEAT
   \STATE Read mini-batch $\boldsymbol{x} = [\mathbf{x}_{1}, ..., \mathbf{x}_{n}]$ from training set
   \STATE Calculate their feature map $[\mathbf{z}_{1}=E(\mathbf{x}_{1}),\cdots, \mathbf{z}_{n}=E(\mathbf{x}_{n})]$
   \STATE Calculate their mean and std of feature map $[\boldsymbol{\mu}(\mathbf{z}_{1}),\cdots, \boldsymbol{\mu}(\mathbf{z}_{n})]$, $[\boldsymbol{\sigma}(\mathbf{z}_{1}),\cdots, \boldsymbol{\sigma}(\mathbf{z}_{n})]$ using Equ $\ref{mean},\ref{std}$
   \STATE \textit{\textcolor{blue}{\# Exploration}}
   \FOR{$i=1$ {\bfseries to} $n$ (in parallel)} 
   \STATE Randomly select $M$ other image's mean and std as style providers $[\mu_{1}, ..., \mu_{M}]$, $[\sigma_{1}, ..., \sigma_{M}]$
   \STATE Initialize $\alpha_{0},\alpha_{1},\cdots,\alpha_{M}$ as $\boldsymbol{\alpha}^{0}$
   \FOR{$k=1$ {\bfseries to} $K$}
   \STATE Calculate $\hat{\mathbf{x}}_{i}^{k}$ using Equ $\ref{meanmix},\ref{stdmix},\ref{apply},\ref{decoder}$ and $\boldsymbol{\alpha}^{k-1}$
   \STATE Update $\boldsymbol{\alpha}^{k}$ using Equ $\ref{updatea2}, \ref{norm2}$
   \ENDFOR
   \STATE Calculate $\hat{\mathbf{x}}_{i}^{final}$ using Equ $\ref{meanmix},\ref{stdmix},\ref{apply},\ref{decoder}$ and $\boldsymbol{\alpha}^{K}$
   \ENDFOR
   \STATE \textit{\textcolor{blue}{\# Update Model}}
   \STATE Calculate $\mathcal{L}_{MODE}$ using Equ $\ref{totalloss2}$
   \STATE Update $\mathbf{w}$ by performing one step gradient update using $\nabla{\mathbf{w}}\mathcal{L}_{MODE}$
   \UNTIL{training converged}
\end{algorithmic}
\end{algorithm}

\section{Discussion}\label{Discussion}

\subsection{Why exploring the worst-case for each sample rather than data distribution is more suitable}

MODE explores the worst case for each sample, namely, $\min \mathbb{E} [\textbf{max} \ell (x,y)]$. In contrast, DRO performs exploration for the data distribution, namely, $\min \textbf{max} \mathbb{E} [\ell (x,y)]$.

Restricted capacity for searching distribution may make it hard to find the worst-case distribution. In contrast, the restricted capacity for searching worst-case samples is relatively easy for deep models \cite{su2019one}.

Moreover, it is challenging to exactly estimate a distribution under DG scenarios. Specifically, if we use a batch of samples to estimate and update the distribution, the estimated distribution could be imprecise, leading to bad explorations. But if we use all samples for the estimation in each iteration, the computational complexity could be excessive, particularly for high-resolution and large-scale datasets (e.g., DomainNet with 0.6M 224x224 images). A possible solution to address the challenge is to use the dual method (dual theorems for optimization need to be developed) for simplification. However, it is beyond the scope, as we mainly propose to perform moderately distributional exploration for DG. Thus, we will leave it as our future work. In contrast, performing sample-level exploration makes the distribution estimation unnecessary, bypassing the above issues.

\subsection{MODE does not require domain ID}

MODE does not require domain ID, as demonstrated in our theoretical proof and Algorithms. However, Domain ID, in certain situations, may provide a marginal performance gain (see Table below). 

\begin{table}[htbp]
  \centering
  \caption{The Impact of the Domain ID on MODE-A OfficeHome}
    \begin{tabular}{|c|c|c|c|c|c|}
    \hline
    MODE-A OfficeHome & A & C & P & R & Avg. \\
    \hline
    with domain ID & 60.0 & \textbf{57.3} & 74.0 & 75.7 & 66.7 \\
    without domain ID & \textbf{60.1} & 57.0 & \textbf{74.2} & \textbf{76.0} & \textbf{66.8} \\
    \hline
    \end{tabular}%
  \label{tab:addlabel5}%
\end{table}%

\subsection{Whether learnable $\lambda$ can bring additional improvements}

In the initial design, we merely thought of a simple approach, where $\lambda$ stands for the ratio of each domain. Namely, $\lambda=\frac{n_i}{N}$, where $n_i$ is the number of samples in the $i^{th}$ domain and $N$ denotes the total number.

We conduct experiments to investigate whether learnable $\lambda$ (such as the way in Group DRO) can bring additional improvements. The results are shown below. We can see that making $\lambda$ learnable can bring a certain performance gain, compared with a fixed $\lambda$.

\begin{table}[htbp]
  \centering
  \caption{Leave-one-domain-out classification accuracies (in \%) on PACS with Learnable $\lambda$}
    \begin{tabular}{|c|c|c|c|c|c|}
    \hline
    \textbf{MODE-F} & A & C & P & S & Avg. \\
    \hline
    Fixed $\lambda$ & 84.5$\pm$0.6 & 80.4$\pm$0.8 & 95.5$\pm$0.2 & 82.2$\pm$0.7 & 85.7 \\
    Learnable $\lambda$ & 84.1$\pm$0.7 & \textbf{81.2$\pm$0.9} & 95.1$\pm$0.3 & \textbf{83.6$\pm$0.4} & \textbf{86.0} \\
    \hline
    \textbf{MODE-A} & A & C & P & S & Avg. \\
    \hline
    Fixed $\lambda$ & 84.4$\pm$0.9 & 81.9$\pm$0.9 & 95.2$\pm$0.3 & 85.8$\pm$0.3 & 86.9 \\
    Learnable $\lambda$ & \textbf{85.5$\pm$1.2} & 81.7$\pm$0.6 & 95.2$\pm$0.4 & 85.5$\pm$0.9 & \textbf{87.1} \\
    \hline
    \end{tabular}%
  \label{tab:addlabel55}%
\end{table}%

\subsection{Why randomly select the style providers}

We choose to randomly select other images as style providers instead of using fixed images as providers. The motivation is straightforward. Randomly selecting samples can diversify the styles used for exploration. Theoretical results indicate a trade-off between the size of the search space in exploration, while exploring more styles can help improve performance. Therefore, we provide different search spaces for each exploration.

To further verify the perspective, we conducted experiments under various experimental settings. In our experiments, the only difference between the two settings is whether the style provider is fixed or not. The results are presented in the following tables. Built upon the results, we find that randomly selecting style providers can indeed enhance the model's performance, demonstrating the rationality of the random selection mechanism.

\begin{table}[htbp]
  \centering
  \caption{Leave-one-domain-out classification accuracies (in \%) on PACS with Fixed Style Provider}
    \begin{tabular}{|c|c|c|c|c|c|}
    \hline
    \textbf{MODE-F} & A & C & P & S & Avg. \\
    \hline
    Randomly & 84.5$\pm$0.6 & 80.4$\pm$0.8 & 95.5$\pm$0.2 & 82.2$\pm$0.7 & \textbf{85.7} \\
    Fixed & 83.1$\pm$0.5 & 79.2$\pm$0.7 & 95.4$\pm$0.4 & 81.1$\pm$0.3 & 84.4 \\
    \hline
    \textbf{MODE-A} & A & C & P & S & Avg. \\
    \hline
    Randomly & 84.4$\pm$0.9 & 81.9$\pm$0.9 & 95.2$\pm$0.3 & 85.8$\pm$0.3 & \textbf{86.9} \\
    Fixed & 82.4$\pm$0.7 & 80.6$\pm$0.2 & 95.2$\pm$0.2 & 83.2$\pm$1.0 & 85.3 \\
    \hline
    \end{tabular}%
  \label{tab:addlabel2}%
\end{table}%

Inspired by the mentioned approach of fixed-style providers, it is interesting to explore whether there exist optimal style providers for a given sample. This is an interesting and challenging problem that we would like to explore in our future work.

\subsection{Why MODE-A outperforms MODE-F}

The mentioned two realizations are related to the semantic and non-semantic partition. Thus, a possible explanation for the difference in model performance is that these two approaches have different abilities in partitioning semantic and non-semantic factors.

Regarding MODE-A, we use AdaIN \cite{huang2017arbitrary,li2022uncertainty} as the partition mechanism. AdaIN separates the semantic and non-semantic factors by processing the feature maps output by the model. The statistical features of the feature maps, such as mean and variance, are used as a good representation of the non-semantic factors, while the normalized feature maps are used as a good representation of the semantic factors. AdaIN creates new samples with different styles by applying different mean and std to the normalized feature map. Regarding MODE-F, we use the Fourier-based method \cite{xu2021fourier} as the partition mechanism. The Fourier-based method assumes that the amplitude spectrum contains more style information, and the phase spectrum contains more semantic information. The Fourier-based method creates new samples with different styles by adjusting the amplitude spectrum of the samples.

In the Fourier-based method, it is difficult to produce a reasonable stylized image by directly adjusting the amplitude spectrum. This often adds some chaotic and disordered color blocks to the generated image, which are unlikely to occur in the real world and may even affect semantic factors. On the other hand, AdaIN with the help of a pre-trained network can combine the low-level style features with the original semantics of the image in a reasonable way, in line with human intuition. These more reasonable images with different styles will enable the model to better learn the distinctions and connections between semantic and non-semantic factors.

Furthermore, we believe that the difference between domains is not limited to the style difference that Fourier and AdaIN target. For example, viewing angle and distance of objects in images are not something that Fourier and AdaIN can change, but these kinds of domain shifts often exist in reality. However, we also believe that selecting appropriate mechanisms to address various domain shifts will consistently yield favorable outcomes with our framework.

\subsection{More comparisons of work related to low confidence issues }

\citet{liu2021stable,liu2022distributionally} focuses on the uncertainty set in the DRO problem, and a Wasserstein distance is employed to determine the uncertainty. In contrast, MODE addresses the challenge when applying DRO to DG problems, where the overly large uncertainty set is shrunk to a subset through a semantic and non-semantic partition. Our semantic and non-semantic strategy is a unique contribution that distinguishes MODE from existing DRO methods \cite{liu2021stable,liu2022distributionally}.

Previous work \cite{liu2021stable,liu2022distributionally} uses adversarial attacks to generate new samples for exploitation. In contrast, MODE employs style transformation methods commonly used in DG to generate new samples.

GroupDRO \cite{sagawa2019distributionally} explores the worst-case by leveraging group information to re-weight groups. In contrast, MODE explores the worst-case by constructing new samples. 

Geometric Wasserstein DRO \cite{liu2022distributionally2} uses data geometry to construct more reasonable and effective uncertainty sets. In contrast, MODE shrinks the uncertainty set by introducing a semantic and non-semantic partition.

Topology-aware robust optimization (TRO) \cite{qiao2023topology} constructs the uncertainty using the data topology. In contrast, MODE constructs the uncertainty subset by constraining the search space with the same semantic factors. Thus, the main difference lies in how to constrain the uncertainty set.

Besides the above difference, previous methods perform exploration for the data distribution, namely, $\min \textbf{max} \mathbb{E} [\ell (x,y)]$. In contrast, MODE explores the worst case for each sample, namely, $\min \mathbb{E} [\textbf{max} \ell (x,y)]$.

\subsection{More comparisons of work related to data augmentation for non-semantic information}

DSU \cite{li2022uncertainty} focuses on addressing the uncertain nature of domain shifts by modeling feature statistics as uncertain distributions. This is achieved through the use of AdaIN, where non-semantic factors (i.e., feature map's mean and std) are replaced with randomly chosen values from the modeled distributions. By effectively modeling domain shifts with uncertainty, DSU significantly enhances the network's generalization ability.

CrossNorm and SelfNorm \cite{tang2021crossnorm} address the problem of domain shift by developing two simple and efficient normalization methods that can reduce the non-semantic domain shift between different distributions. It has been discovered that processing the mean and variance of channels for samples or feature maps can help improve generalization ability. These methods are complementary and can be applied to various fields.

In contrast, we take a different approach to solving the problem of domain shift. We aim to improve the model's overall generalization ability by exposing it to more difficult domains during training. This min-max game is common in DRO, but simply applying DRO does not always lead to good results. Instead, we focus on constraining the exploration of semantic and non-semantic factors and propose a theoretical framework to demonstrate the feasibility of our approach.

We theoretically prove that actively improving the model's performance on a range of data distributions can help enhance its overall generalization ability, even if the final test domain is not included in the range of distributions explored. To achieve this, we use Fourier and AdaIN and actively search for the most challenging domains before each update step.

By generating new samples, we enable the model to explore more difficult domains. In contrast, CrossNorm and SelfNorm focus on designing a new normalization method that can be embedded into the model, processing the mean and variance of channels of feature maps.

Although MODE and DSU both use AdaIN to generate samples, DSU models non-semantic factors as a multivariate Gaussian distribution and randomly samples the factors within this distribution. In contrast, following our theoretical results, we actively explore more challenging non-semantic factors in the space, resulting in more challenging samples each time.

Moreover, we highlight the difference between our method and previous works \cite{tang2021crossnorm,li2022uncertainty} through an empirical perspective. Specifically, we compare our baselines in experiments. We have since reproduced these two methods:

\begin{table}[htbp]
  \centering
  \caption{Additional experiments about DSU \cite{li2022uncertainty} and CNSN \cite{tang2021crossnorm}.}
    \begin{tabular}{|c|c|c|c|c|c|}
    \hline
    ResNet18 PACS & A     & C     & P     & S     & Avg. \\
    \hline
    CNSN \cite{tang2021crossnorm} & 83.6$\pm$0.3 & 79.1$\pm$0.3 & 96.5$\pm$0.1 & 80.2$\pm$0.3 & 84.8 \\
    DSU \cite{li2022uncertainty} & 83.1$\pm$0.3 & 79.8$\pm$0.4 & 96.3$\pm$0.1 & 77.3$\pm$0.1 & 84.1 \\
    MODE-F & 84.5$\pm$0.6 & 80.4$\pm$0.8 & 95.5$\pm$0.2 & 82.2$\pm$0.7 & 85.7 \\
    MODE-A & 84.4$\pm$0.9 & 81.9$\pm$0.9 & 95.2$\pm$0.3 & 85.8$\pm$0.3 & 86.9 \\
    \hline
    \end{tabular}%
  \label{tab:addlabe3}%
\end{table}%

We can see that our method can outperform the baselines. Besides the performance gain, we realize that it is necessary to consider the running time of each method. Accordingly, we also compare our method with our baselines, taking running time and FLOPs into consideration.

\begin{table}[htbp]
  \centering
  \caption{Running Time and FLOPs of ResNet18 PACS}
    \begin{tabular}{|c|c|c|}
    \hline
    ResNet18 PACS & Running Time & FLOPs \\
    \hline
    CNSN \cite{tang2021crossnorm} & 25min & 1x \\
    DSU \cite{li2022uncertainty} & 35min & 1.2x \\
    MODE-F & 4h & $\sim$8x \\
    MODE-A & 5h & $\sim$9x \\
    \hline
    \end{tabular}%
  \label{tab:addlabe4}%
\end{table}%

\begin{table}[htbp]
  \centering
  \caption{Performance Comparison on CNSN \cite{tang2021crossnorm} with different number of epochs. All results are conducted on the PACS dataset with Sketch(S) as the unknown target domain.}
    \begin{tabular}{|c|c|c|}
    \hline
    CNSN \cite{tang2021crossnorm} num of epoch & 50 epoch & 100 epoch  \\
    \hline
    Running Time (min) & 30 & 59  \\
    \hline
    Acc (\%) & 80.2 & 80.1  \\
    \hline
    \end{tabular}%
\end{table}%

\begin{table}[htbp]
  \centering
  \caption{Performance Comparison on different number of Inner Steps $K$ in MODE-A. All results are conducted on the PACS dataset with Sketch(S) as the unknown target domain.}
    \scalebox{0.9}{\begin{tabular}{|c|c|c|c|c|c|c|c|c|c|c|c|}
    \hline
     $K$ & 0 (Random) & 1 & 2 & 3 & 4 & 5 & 6 & 7 & 8 & 9 & 10 \\
    \hline
    Running Time (min) &  41 & 68 & 95 & 122 & 157 & 189 & 211 & 242 & 278 & 301 & 331 \\
    \hline
    Acc (\%) &  80.49 & 81.29 & 82.11 & 83.45 & 84.56 & 84.74 & 84.77 & 86.76 & 86.05 & 85.75 & 85.82 \\
    \hline
    \end{tabular}}%
  \label{tab:addlabel555}%
\end{table}%

It can be observed that due to the presence of the inner step, our method takes much longer running time than our baselines as the cost of promoting model performance. Thus, it is interesting to explore a more efficient approach to reduce the time cost while improving model performance, like \citet{shafahi2019adversarial,zhang2019you}. We thank the reviewer for the insightful comments and we will explore the exciting direction in our future work. 

Our method is unable to adequately augment out-of-distribution (OOD) samples which have different semantic factors with training domains \cite{wang2022watermarking}. 

\section{More Result and Implementation Details}\label{moreresult}
\subsection{Datasets}
\textbf{VLCS} \cite{torralba2011unbiased}
consists of 10,729 images from four domains, namely Caltech (C), Labelme (L), Pascal(V), Sun (S). There are five classes in each domain.

\textbf{DomainNet} \cite{peng2019moment} is a large-scale dataset designed for domain generalization, which contains 6.3 million images from 345 categories covering a wide range of visual domains. It has 6 domains:  ClipArt(C), Infograph (I), Painting (P), Quickdraw (Q), Real (R) and Sketch (S).

\textbf{Mini-Domainnet} \cite{zhou2021domain} is a highly challenging subset of DomainNet with a lower resolution (96x96) and 0.1M images. It has about 140K images with 126 classes and 4 domains: ClipArt(C), Painting (P), Real (R) and Sketch (S).

\subsection{Implementation Details}
Following the commonly used leave-one-domain-out strategy \cite{li2017deeper,xu2021fourier}, the model will be tested on one domain after training on all other domains.

\textbf{Basic Details} For VLCS, we use a pre-trained AlexNet backbone. We train the network using SGD optimizer with learning rate 5e-4, momentum 0.9, and weight decay 5e-4. We train the model for 50 epochs with batch size 32. The learning rate is decayed by 0.1 every 40 epochs.

For DomainNet, we use a pre-trained ResNet50 backbone. We train the network using adam optimizer with learning rate 2e-4 and weight decay 1e-4. We train the model for 50 epochs with batch size 256. The learning rate is decayed by 0.1 every 30 epochs.

For Mini-DomainNet, we use a pre-trained ResNet18 backbone. We train the network using sgd optimizer with learning rate 5e-3 and weight decay 5e-4. We train the model for 60 epochs with batch size 256. Cosine learning rate
scheduler is used.

\textbf{Method-specific Details}
The method-specific details are shown in \ref{md}.

\subsection{Experimental Results}
\textbf{Results on VLCS}
We show the Leave-one-domain-out classification accuracies (in $\%$) on VLCS on Tab $\ref{tab:3}$. It can be observed that our approach achieves the highest average accuracy, but our result is only a little better than other methods. We think that it is because VLCS is different from other datasets in domain shift. All the data in VLCS are real-world images having complex compositions and background, which can't be handled well by Fourier-based transfer and AdaIN transfer. There will be better results by choosing a more suitable generation method to apply to our framework.
\begin{table}\label{tab:3}
\centering
\caption{Leave-one-domain-out classification accuracies (in $\%$) on VLCS in AlexNet. The best and second-best results are highlighted in bold and underlined, respectively.}
\begin{tabular}{c|ccccc}
\hline
\textbf{Dataset} & \multicolumn{5}{c}{VLCS}                                                                              \\ \hline
Methods          & C              & L             & V             & \multicolumn{1}{c|}{S}              & Avg.           \\ \hline
DeepAll \cite{zhou2020deep}       & 96.3           & 59.7          & 70.6          & \multicolumn{1}{c|}{64.5}           & 72.8           \\
MLDG \cite{li2018learning}             & 97.9           & 59.5          & 66.4          & \multicolumn{1}{c|}{64.8}           & 72.2           \\
Epi-FCR \cite{li2019episodic}         & 94.1           & {\ul 64.3}    & 67.1          & \multicolumn{1}{c|}{65.9}           & 72.9           \\
MAML \cite{finn2017model}            & 97.8           & 58            & 67.1          & \multicolumn{1}{c|}{64.1}           & 71.8           \\
Jigen \cite{carlucci2019domain}           & 96.9           & 60.9          & {\ul 70.6}    & \multicolumn{1}{c|}{64.3}           & 73.2           \\
MMLD \cite{matsuura2020domain}            & 96.6           & 58.7          & \textbf{72.1} & \multicolumn{1}{c|}{66.8}           & 73.5           \\
CICF \cite{li2021confounder}            & {\ul 97.8}     & 60.1          & 69.7          & \multicolumn{1}{c|}{67.3}           & 73.7           \\
MASF \cite{dou2019domain}            & 94.8           & \textbf{64.9} & 69.1          & \multicolumn{1}{c|}{67.6}           & 74.1           \\ \hline
MODE-F (ours)    & \textbf{97.87} & 61.17         & 69.54         & \multicolumn{1}{c|}{\textbf{68.73}} & {\ul 74.33}    \\
MODE-A (ours)      & 96.92          & 63.05         & 70.28         & \multicolumn{1}{c|}{{\ul 67.97}}    & \textbf{74.55} \\ \hline
\end{tabular}
\end{table}

\textbf{Results on DoaminNet}
We show the Leave-one-domain-out classification accuracies (in $\%$) on VLCS on Tab $\ref{tab:4}$. It can be observed that our approach achieves the higher average accuracy.

\begin{table}[ht]
\centering
\caption{Leave-one-domain-out classification accuracies (in $\%$) on DomainNet in ResNet50.}
\label{tab:4}
\begin{tabular}{l|cccccc|c}
\hline
Methods & Clipart & Infograph & Painting & Quickdraw & Real & Sketch & Avg. \\ \hline
Baseline & 66.35 & 23.01 & 50.48 & 13.82 & 63.57 & 50.79 & 44.67 \\ 
MODE-F (ours) & \textbf{68.50} & 23.14 & \textbf{53.04} & 15.92 & \textbf{63.72} & \textbf{54.99} & \textbf{46.55} \\ 
MODE-A (ours) & 68.26 & \textbf{23.39} & 52.45 & \textbf{16.78} & 63.05 & 53.96 & 46.31 \\ \hline
\end{tabular}
\end{table}

\textbf{Results on Mini-DoaminNet}
We show the Leave-one-domain-out classification accuracies (in $\%$) on VLCS on Tab $\ref{tab:5}$. It can be observed that our approach achieves the higher average accuracy.

\begin{table}[htbp]
  \centering
  \caption{Leave-one-domain-out classification accuracies (in $\%$) on Mini-DomainNet in ResNet18. }
    \begin{tabular}{c|cccc|c}
    \hline
    Methods & Clipart & Painting & Real & Sketch & Avg. \\
    \hline
    Baseline & 59.04 & 47.20 & \textbf{56.18} & 51.74 & 53.54 \\
    MODE-F (ours) & 60.63 & 48.09 & 54.92 & \textbf{55.39} & 54.75 \\
    MODE-A (ours) & \textbf{63.56} & \textbf{48.25} & 55.87 & 52.69 & \textbf{55.09} \\
    \hline
    \end{tabular}%
  \label{tab:5}%
\end{table}%


\clearpage
\clearpage

\subsection{The Overall Method-specific Details}\label{md}

The method-specific details of our approach are shown in Table \ref{mdf} and Table \ref{mda}.

\begin{table}[htbp]\label{mdf}
\caption{The method-specific details of the Fourier-based approach.}
\begin{tabular}{|c|c|c|c|c|c|c|}
\hline
Dataset & Domain & Number of inner steps $K$ & Inner step size $\mu$ & $\beta$ & $\gamma$ & The number of style providers $M$ \\ \hline
Digit-DG     & All    & 10                        & 0.05                  & 0.3     & 1        & 3                                 \\ \hline
DomainNet     & All    & 7                        & 0.05                  & 0.3     & 1        & 12                                 \\ \hline
Others     & All    & 10                        & 0.05                  & 0.3     & 1        & 8                                 \\ \hline
\end{tabular}
\end{table}

\begin{table}[htbp]\label{mda}
\caption{The method-specific details of the AdaIN-based approach.}
\scalebox{0.92}{\begin{tabular}{|c|c|c|c|c|c|c|}
\hline
Dataset                     & Domain & Number of inner steps $K$ & Inner step size $\mu$ & $\beta$              & $\gamma$             & The number of style providers $M$           \\ \hline
PACS                        & All    & \multirow{5}{*}{10}       & \multirow{5}{*}{0.05} & \multirow{5}{*}{0.4} & \multirow{2}{*}{1}   & \multirow{2}{*}{3 (Each domain provide one)} \\ \cline{1-2}
\multirow{2}{*}{OfficeHome} & C      &                           &                       &                      &                      &                                             \\ \cline{2-2} \cline{6-7} 
                            & other  &                           &                       &                      & \multirow{2}{*}{0.3} & \multirow{3}{*}{8}                          \\ \cline{1-2}
VLCS                        & All    &                           &                       &                      &                      &                                             \\ \cline{1-2} \cline{6-6} 
Mini-DomainNet                        & All    &                           &                       &                      &          \multirow{2}{*}{1}            &                                            \\ \cline{1-3} \cline{7-7} 
DomainNet                        & All    &            5               &                       &                      &                      &                   5 (Each domain provide one)                         \\ \hline
\end{tabular}}
\end{table}

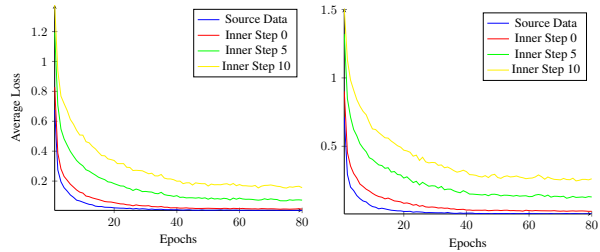
\begin{figure}[htbp] 
\centering 
\begin{minipage}{.24\textwidth}
\begin{tikzpicture}[scale=0.48] 
\begin{axis}[
    xlabel= Epochs, 
    ylabel= Average Loss, 
    tick align=outside,
    axis y line=left,
    axis x line=bottom
    ]
\addplot [mark = none, blue] table {b00.txt};
\addlegendentry{Source Data}
\addplot [mark = none, red] table {b0.txt};
\addlegendentry{Inner Step 0}
\addplot [mark = none, green] table {b5.txt};
\addlegendentry{Inner Step 5}
\addplot [mark = none, yellow] table {b10.txt};
\addlegendentry{Inner Step 10}
\end{axis}
\end{tikzpicture}
\end{minipage}%
\begin{minipage}{.24\textwidth}
\begin{tikzpicture}[scale=0.48] 
\begin{axis}[
    xlabel= Epochs, 
    tick align=outside,
    axis y line=left,
    axis x line=bottom
    ]
\addplot [mark = none, blue] table {00.txt};
\addlegendentry{Source Data}
\addplot [mark = none, red] table {0.txt};
\addlegendentry{Inner Step 0}
\addplot [mark = none, green] table {5.txt};
\addlegendentry{Inner Step 5}
\addplot [mark = none, yellow] table {10.txt};
\addlegendentry{Inner Step 10}
\end{axis}
\end{tikzpicture}
\end{minipage}
\caption{The average loss of augmented samples in different inner steps to the current model changes with the number of epochs in a training process. The left results are conducted on the PACS dataset with Art-painting (A) as the unknown target domain; the right results are conducted with Sketch (S) as the unknown target domain.}
\end{figure}

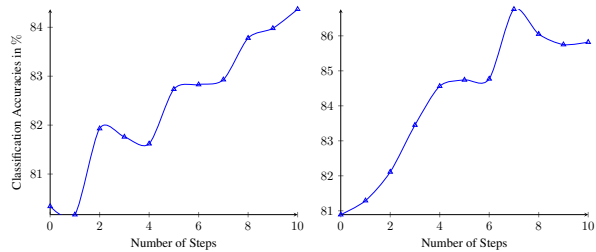
\begin{figure}[!h] 
\centering 
\begin{minipage}{.24\textwidth}
\begin{tikzpicture}[scale=0.48] 
\begin{axis}[
    xlabel= Number of Steps, 
    ylabel= Classification Accuracies in $\%$, 
    tick align=outside,
    axis y line=left,
    axis x line=bottom
    ]
\addplot[smooth,mark=triangle,blue] plot coordinates { 
    (0,80.34)
    (1,80.17)
    (2,81.93)
    (3,81.76)
    (4,81.62)
    (5,82.735)
    (6,82.83)
    (7,82.93)
    (8,83.78)
    (9,83.98)
    (10,84.37)
};
\end{axis}
\end{tikzpicture}
\end{minipage}%
\begin{minipage}{.24\textwidth}
\begin{tikzpicture}[scale=0.48] 
\begin{axis}[
    xlabel= Number of Steps, 
    tick align=outside,
    axis y line=left,
    axis x line=bottom
    ]
\addplot[smooth,mark=triangle,blue] plot coordinates { 
    (0,80.89)
    (1,81.29)
    (2,82.11)
    (3,83.45)
    (4,84.56)
    (5,84.74)
    (6,84.77)
    (7,86.76)
    (8,86.05)
    (9,85.75)
    (10,85.82)
};
\end{axis}
\end{tikzpicture}
\end{minipage}
\caption{The effect of the number of inner steps. The left results are conducted on the PACS dataset with Art-painting (A) as the unknown target domain; the right results are conducted with Sketch (S) as the unknown target domain.}
\end{figure}

\clearpage
\clearpage

\subsection{Visualization Results}\label{vr}

\begin{figure}[htbp]
	\centering
	\includegraphics[width=1\textwidth]{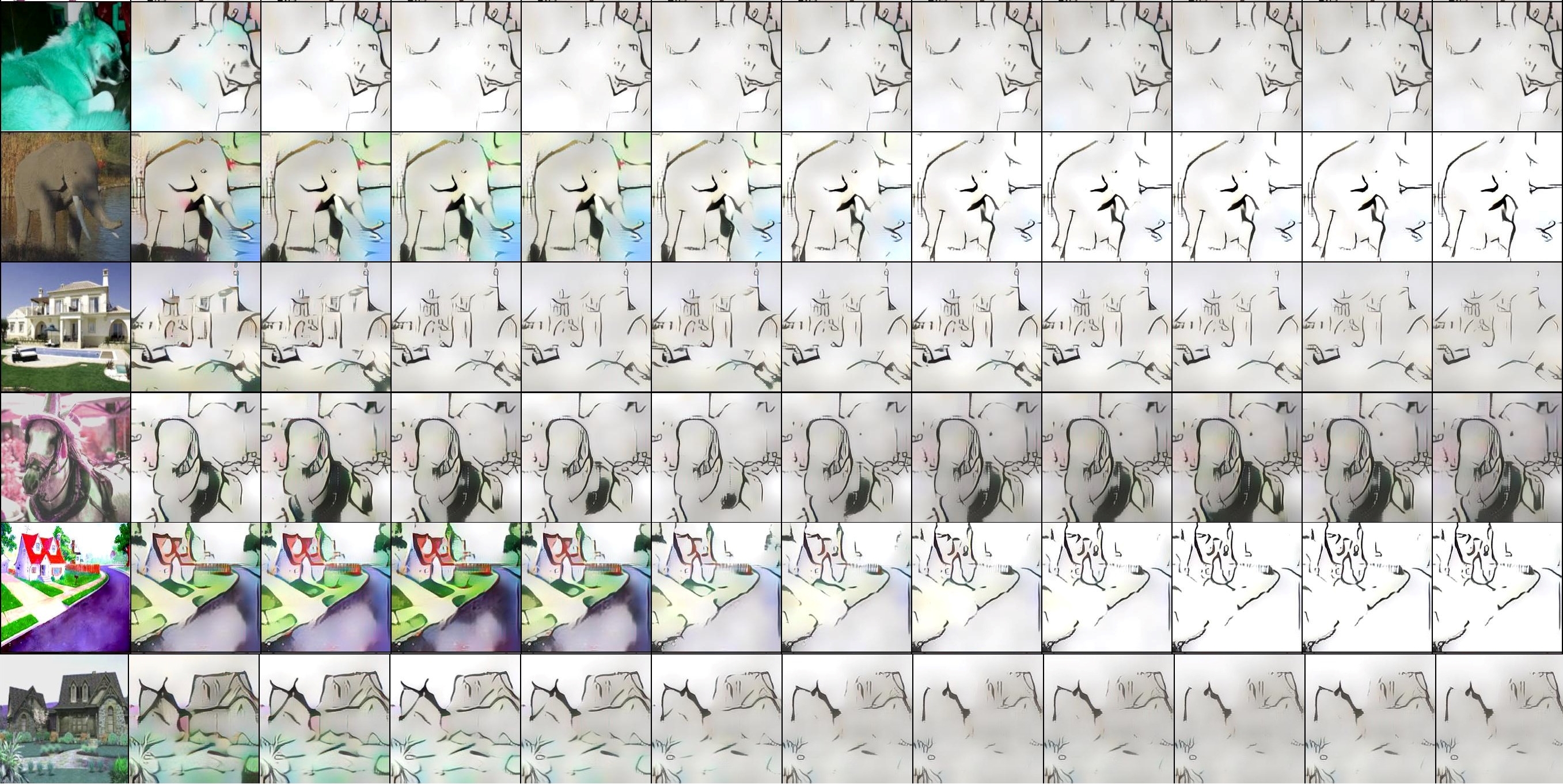}
	\caption{The changes of images during exploration when $\Omega$ is too large in our approach. In each row, the leftmost image is the original image, and from left to right is the result of each inner explore step of this image. The results are conducted on the PACS dataset with Art-painting (A) as the target with $\gamma=1$ and $M=10$.}
	\label{fig:bad}
\end{figure}

\begin{figure}[htbp]
	\centering
	\includegraphics[width=1\textwidth]{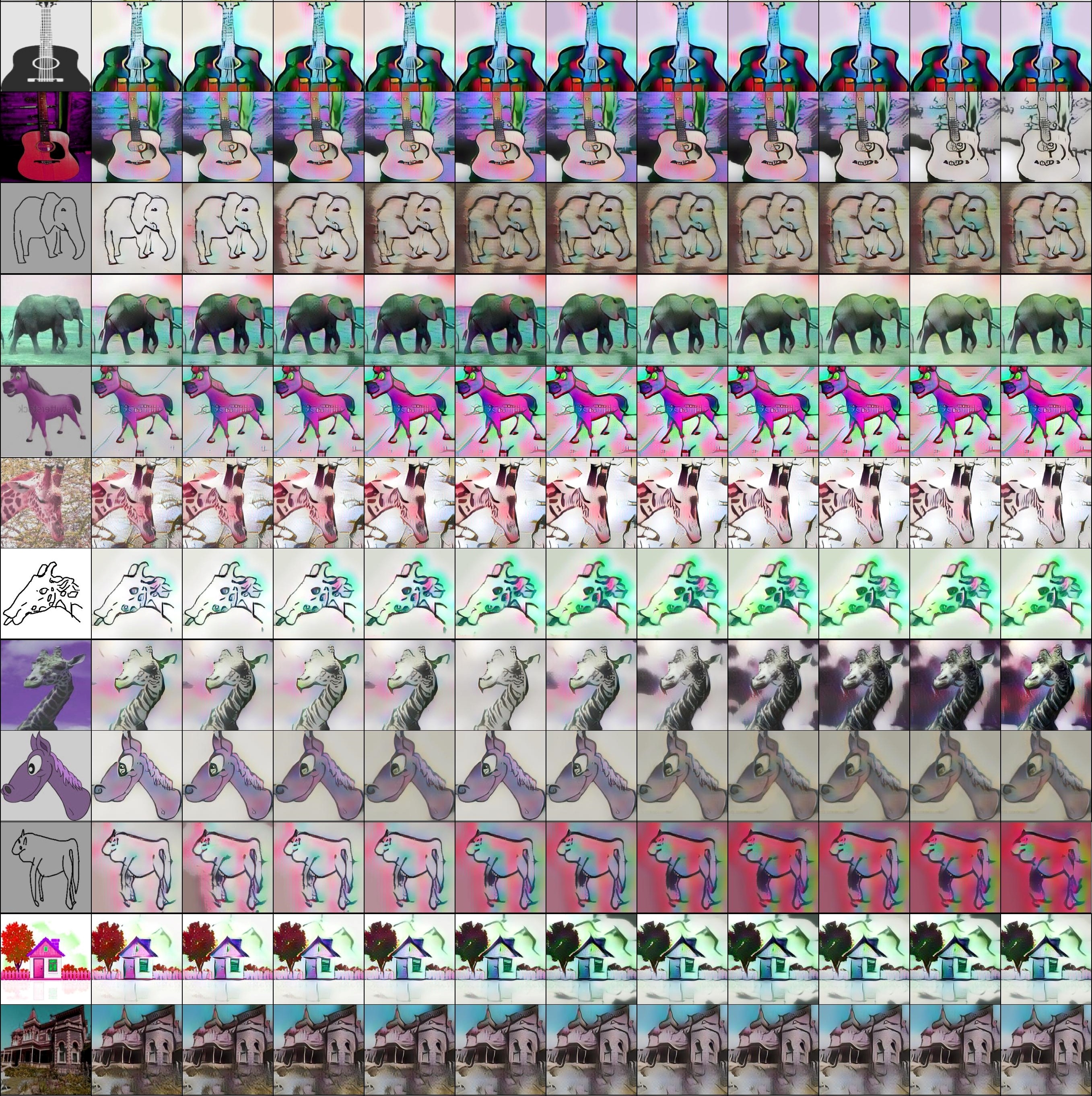}
	\caption{The changes of images during exploration in our approach. In each row, the leftmost image is the original image, and from left to right is the result of each inner explore step of this image.}
	\label{fig:good}
\end{figure}

\end{document}